\documentclass[twoside]{article}

\usepackage[accepted]{aistats2026}
\usepackage[round]{natbib}

\usepackage{url}
\usepackage{array}
\usepackage{microtype}
\usepackage{graphicx}
\usepackage{xcolor}
\usepackage{subfigure}
\usepackage{amsmath}
\usepackage{bm}
\usepackage{dsfont}
\usepackage{mathtools}
\usepackage{nccmath}
\usepackage{amssymb}
\usepackage{multirow}
\usepackage{booktabs}
\usepackage{varwidth}
\usepackage{pifont}
\usepackage{makecell}
\usepackage{wrapfig}
\usepackage[flushleft]{threeparttable}
\usepackage{varwidth}
\usepackage{pifont}
\usepackage{makecell}
\usepackage{enumitem}
\usepackage{adjustbox}
\usepackage{graphicx,calc}
\usepackage{environ}
\usepackage{multicol}
\usepackage{multirow}
\usepackage{graphicx}
\usepackage{wrapfig}
\usepackage{tikz}
\usepackage{bbm}
\usepackage[flushleft]{threeparttable}
\usepackage{pifont}
\usepackage{enumitem}

\usepackage{pifont}

\usepackage{color, colortbl}
\definecolor{Gray}{gray}{0.93}
\definecolor{Orange}{rgb}{1,0.5,0}
\definecolor{DGray}{gray}{0.83}
\definecolor{LightCyan}{rgb}{0.88,1,1}

\definecolor{natural}{HTML}{648FFF}
\definecolor{specialized}{HTML}{DC267F}
\definecolor{fgvc}{HTML}{362682}
\definecolor{others}{HTML}{e6c595}
\definecolor{all}{HTML}{FE6100}

\newcommand{\stress}[1]{#1}

\definecolor{lightorange}{HTML}{fc8e62}
\definecolor{lightgray}{gray}{0.6}

\newcommand{\orangecolor}[1]{\textcolor{lightorange}{#1}}
\newcommand{\partcirc}{\orangecolor{$\mathbf{\circ}$\,}}
\newcommand{\fullcirc}{\orangecolor{$\bullet$\,}}

\usepackage{wrapfig}

\usepackage{pifont}
\usepackage{color, colortbl}

\usepackage{blindtext}
\usepackage{lipsum}

\usepackage{multirow}
\usepackage{graphicx}
\usepackage{listings}

\usepackage{bbm}

\usepackage[most]{tcolorbox}

\usepackage{tikz}
\usepackage{tabularx}

\usepackage{amsmath,amsfonts,bm}

\newtheorem{thm}{Theorem}
\newtheorem{lemma}{Lemma}

\newtheorem{prpst}{Proposition}

\def\bfa{{\boldsymbol a}}

\def\bfe{{\boldsymbol e}}

\def\bfh{{\boldsymbol h}}

\def\bfk{{\boldsymbol k}}

\def\bfq{{\boldsymbol q}}

\def\bfv{{\boldsymbol v}}
\def\bfw{{\boldsymbol w}}
\def\bfx{{\boldsymbol x}}

\def\bfz{{\boldsymbol z}}
\def\bfmu{{\boldsymbol \mu}}

\def\bfH{{\boldsymbol H}}
\def\bfI{{\boldsymbol I}}

\def\bfP{{\boldsymbol P}}

\def\bfW{{\boldsymbol W}}

\def\bfdt{{\boldsymbol \delta}}

\def\eqref#1{(\ref{#1})}

\def\1{\bm{1}}

\DeclareMathAlphabet{\mathsfit}{\encodingdefault}{\sfdefault}{m}{sl}
\SetMathAlphabet{\mathsfit}{bold}{\encodingdefault}{\sfdefault}{bx}{n}

\newcommand{\btheta}{\boldsymbol{\theta}}

\newcommand{\bdelta}{\boldsymbol\delta}

\newcommand{\vpt}{\textsc{VPT}}
\newcommand{\vp}{\textsc{VP}}
\newcommand{\lora}{\textsc{LoRA}}
\newcommand{\bias}{\textsc{Bias}}

\newcommand{\nt}{\textsc{Norm-Tune}}

\newcommand{\adapter}{\textsc{Adapter}}
\newcommand{\adapterformer}{\textsc{AdapterFormer}}
\newcommand{\ours}{\textsc{AP}}

\newcommand{\ff}{\textsc{Full-Tune}}
\newcommand{\lp}{\textsc{Linear-Probe}}
\newcommand{\scale}{\textsc{AttnScale}}

\newcommand{\cifarten}{{CIFAR-10}}
\newcommand{\cifarhun}{{CIFAR-100}}
\newcommand{\waterbirds}{{Waterbirds}}
\newcommand{\oxfordflowers}{{OxfordFlowers}}
\newcommand{\sun}{{SUN397}}
\newcommand{\oxfordpets}{{OxfordPets}}
\newcommand{\flowers}{{Flowers102}}
\newcommand{\dtd}{{DTD}}
\newcommand{\stanfordcars}{{StanfordCars}}
\newcommand{\food}{{Food101}}
\newcommand{\eurosat}{{EuroSAT}}
\newcommand{\ucf}{{UCF101}}
\newcommand{\gtsrb}{{GTSRB}}
\newcommand{\svhn}{{SVHN}}

\newcommand{\camelyon}{{Camelyon}}

\newcommand{\cub}{{CUB200}}
\newcommand{\nabirds}{{NA-Birds}}
\newcommand{\dog}{{StanfordDog}}
\newcommand{\caltech}{{Caltech-101}}

\newcommand{\vtab}{{VTAB}}
\newcommand{\vtabk}{{VTAB-1k}}
\newcommand{\fgvc}{{FGVC}}

\usepackage[utf8]{inputenc} 
\usepackage[T1]{fontenc}    
\usepackage{hyperref}       
\usepackage{amsfonts}       
\usepackage{nicefrac}       
\usepackage{microtype}      

\definecolor{ceruleanblue}{rgb}{0.16, 0.32, 0.75}

\hypersetup{
colorlinks=true,
citecolor=ceruleanblue,
linkcolor=ceruleanblue,
urlcolor=black}

\NewEnviron{elaboration}{
\par
\begin{tikzpicture}
\node[rectangle,minimum width=0.44\textwidth] (m) {\begin{minipage}{0.44\textwidth}\BODY\end{minipage}};
\draw[thick] (m.south west) rectangle (m.north east);
\end{tikzpicture}
}

\begin{document}

%

%
\runningauthor{Y. Zhang, H. Li, Y. Yao, A. Chen, S. Zhang, P.-Y. Chen, M. Wang, S. Liu}

\twocolumn[

\aistatstitle{Visual Prompting Reimagined: The Power of Activation Prompts}

\aistatsauthor{Yihua Zhang$^{1,*}$, Hongkang Li$^{2,*}$, Yuguang Yao$^{1,*}$, Aochuan Chen$^{1}$, Shuai Zhang$^{3}$,\\ \textbf{Pin-Yu Chen$^{4}$, Meng Wang$^{5}$, Sijia Liu$^{1,4}$}}

\aistatsaddress{$^{1}$Michigan State University, $^{2}$University of Pennsylvania,\\ $^{3}$New Jersey Institute of Technology, $^{4}$IBM Research, $^{5}$Rensselaer Polytechnic Institute\\ $^{*}$Equal Contribution} ]

\begin{abstract}
Visual prompting (VP) has emerged as a popular method to repurpose pretrained vision models for adaptation to downstream  tasks. Unlike conventional model fine-tuning techniques, VP introduces a universal perturbation directly into the input data to facilitate task-specific fine-tuning rather than modifying model parameters. However, there exists a noticeable performance gap between VP and conventional fine-tuning methods, highlighting an unexplored realm in theory and practice to understand and advance (input-level) VP to reduce its current performance gap. Towards this end,  we introduce a generalized concept, termed activation prompt (AP), which extends the scope of (input-level) VP by enabling universal perturbations to be applied to activation maps within the intermediate layers of the model. By using AP to revisit the problem of VP and employing it as an analytical tool, we demonstrate the intrinsic limitations of VP in both performance and efficiency, revealing why input-level prompting may lack effectiveness compared to AP, which exhibits a model-dependent layer preference. We show that AP is closely related to normalization tuning in convolutional neural networks and vision transformers, although each model type has distinct layer preferences for prompting. We also theoretically elucidate the rationale behind such preference by analyzing global features across layers. Through extensive experiments across 29 datasets and various model architectures, we provide a comprehensive performance analysis of AP, comparing it with VP and parameter-efficient fine-tuning baselines. Our results demonstrate AP's superiority in both accuracy and efficiency, considering factors such as time, parameters, memory usage, and throughput.
\end{abstract}

\section{Introduction}
Large pretrained models have emerged as fundamental components in deep learning\,\citep{touvron2023llama,chiang2023vicuna,bai2023qwen} in recent years. Despite their exceptional performance, the substantial increase in computational demands, as highlighted in recent studies \citep{frantar2023massive}, has underlined the need for more economical and lightweight fine-tuning approaches. Thus, the pretraining-finetuning paradigm rises, allowing for quickly adapting a pretrained model to downstream tasks\,\citep{jia2022visual,hu2021lora, cai2020tinytl, sung2022lst,chen2023understanding}. 
Among the various parameter-efficient finetuning (PEFT) methods\,\citep{hu2021lora,chen2022adaptformer,pfeiffer2020adapterhub,he2021towards,xu2023exploring}, \textit{prompting technique} has been gaining popularity in {the vision domain \citep{liu2023pre}.}

Different from the model-centric PEFT techniques in computer vision (CV), the conventional visual prompting (\textbf{VP}) crafts specific input perturbations (known as `prompts') to reprogram the pretrained model for a targeted task, without altering the model parameters. This offers a new data-centric viewpoint to analyze, understand, and harness the pretrained model\,\citep{chen2023understanding}. However, 
despite the recent advancement,
the performance of state-of-the-art (SOTA) VP methods still lags behind model-based fine-tuning methods\,\citep{chen2023understanding}.
It appears that the potential of VP has not been fully realized for CV models, particularly when considering its relative progress compared to its counterpart in natural language processing (NLP)\,\citep{liu2023pre, li2021prefix}.
In this work, \textbf{we aim to} rigorously and comprehensively examine VP and explore its enhancement tailored for CV models, including convolutional neural networks (CNNs) and vision Transformers (ViTs). We ask:
\begin{tcolorbox}[before skip=2mm, after skip=0.0cm, boxsep=0.0cm, middle=0.0cm, top=0.1cm, bottom=0.1cm]
\textit{\textbf{(Q)} Is VP (visual prompting)  truly beneficial for improving vision  tasks, and under what conditions does it prove effective or ineffective?}
\end{tcolorbox}
\vspace*{3mm}

To tackle question \textbf{(Q)}, we present a generalized variant of VP termed activation prompt (\textbf{AP}), which involves the incorporation of learnable perturbations into the activation maps of intermediate layers, rather than focusing solely on the input layer. See \textbf{Fig.\,\ref{fig: teaser}} for an illustration. The introduction of AP allows us to study the ineffectiveness of VP, as VP can be treated as a specific realization of AP. By using AP as both a bridge and an analytical tool, we show the conventional input-based VP might not be the most effective or efficient design. In fact, appropriately implemented AP can outperform traditional VP significantly. To shed light on the underlying mechanism of AP, we present both empirical evidence and theoretical insights. 

\begin{figure}[t]
    \centering
    \includegraphics[width=0.8\linewidth]{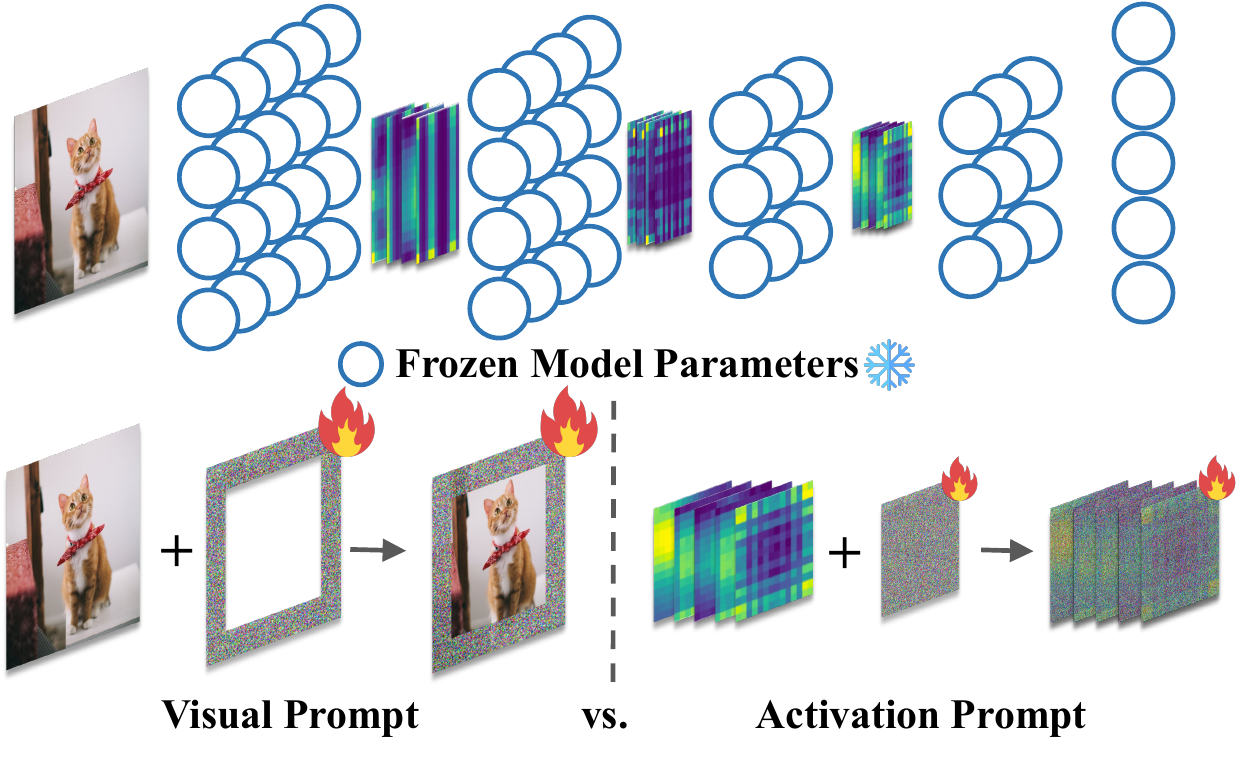}
    \vspace*{-0.5em}
    \caption{An illustration of the proposed activation prompt vs. the conventional input-based prompt.}
    \label{fig: teaser}
    \vspace*{-1em}
\end{figure}

The work \textbf{most relevant} to ours is \citep{jia2022visual}, which also integrates prompts with intermediate layers of ViTs, resulting in the method known as visual prompt tuning (VPT). However,  our work has the following distinctions from VPT.
\textit{First}, AP and VPT diverge in their designs. AP concentrates on the targeted application of prompts to a single model layer. In contrast, VPT and its deep variant (termed VPT-deep)  apply prompts across multiple layers. Specifically, VPT-deep initiates prompts at one layer and extends them across all subsequent layers. The distinctive layer-prompting approach makes VPT not covering VP as a special case. In contrast, AP serves as a generalized framework for VP, making it easier to analyze its effectiveness.
\textit{Second}, this work identifies the layer preference of CV models regarding prompts. Through AP, we can gain insights into these layer preferences on both CNNs and ViTs. In contrast, VPT does not conduct a systematic analysis of layer and architectural type effects.
\textit{Third}, another notable difference between our work and the VPT study is our theoretical analysis. We establish a connection between AP and normalization tuning and theoretically validate the concept of layer preference and its influence on various architectural designs. Our theoretical analysis also shows that the traditional implementation of input-level VP could be suboptimal. 
In summary, our \textbf{contributions} include:

$\bullet$ 
We propose AP (activation prompt) as a valuable tool for gaining insights into VP (visual prompting). And {\ours} establishes itself as a versatile and  effective prompting technique in its own right, revealing a provable relationship with normalization tuning (Sec.\,\ref{sec: method}). 

$\bullet$ We offer in-depth analyses of AP's layer preference and its architecture effects. Empirical studies unveiled the connection between the layer preference and the capacity for capturing global features (Sec.\,\ref{sec: ap_layer}). In addition, we theoretically validate those findings (Sec.\,\ref{sec: theory}).

$\bullet$ Through extensive experimentation involving $29$ datasets across various benchmarks, we affirm that {\ours} enhances the input-level VP in diverse learning scenarios. Furthermore, {\ours} narrows the performance gap even when compared to $6$ other SOTA PEFT methods.

\section{Related Work}
\label{sec: related_work}

\textbf{Visual prompting.} 
VP was first proposed in \citep{bahng2022visual,jia2022visual} to extend the prompting technique in NLP. A similar idea with a different name, known as adversarial reprogramming, was also proposed earlier in CV \citep{elsayed2018adversarial, chen2022model, neekhara2022cross,zhang2022fairness,chen2022visual}. It aims at re-purposing a fixed pretrained model to adapt to a new task. Recent advancement focuses on improved label mapping\,\citep{chen2021adversarial, yang2023visual} and normalization\,\citep{wu2022unleashing} to enhance VP. Others extend VP to areas like adversarial defense\,\citep{chen2023visual, mao2022understanding} and distribution shift\,\citep{huang2023diversity,tsai2023self}, and vision-language models\,\citep{zhou2022learning}.

\textbf{Theoretical study on prompt engineering.} Existing theoretical works on prompt engineering include the expressive power of the introduced parameter \citep{wei2021pretrained, bai2023transformers, akyurek2022learning}, the optimization process \citep{ding2022delta, von2023transformers}, and the generalization analysis \citep{xie2021explanation, oymak2023role, zhang2023trained, li2023transformers_v2, HCL23, li2024training, litraining, li2024nonlinear, li2025understanding}. Most studies concentrate on in-context learning, a tuning-free hard prompt method. In contrast, for soft prompt tuning, \citet{wei2021pretrained} show that prompting is powerful enough to remove nonessential information for the downstream task. \citet{ding2022delta} interpret prompt tuning as a subspace optimization method for the solution or functional space. Notably, there is solely one study \citep{oymak2023role} on the generalization dynamics of gradient-based prompt tuning but relying on a single-layer Transformer architecture without the MLP layer, making it incapable of examining the impact of multiple layers.

\textbf{Parameter-efficient fine-tuning.} PEFT demonstrates that only finetuning a small part of a large pretrained model can achieve outstanding performance. In the domain of CV, besides prompting-based methods, PEFT methods can be roughly classified into two categories. The former focuses on identifying a small ratio of parameters to update from the pretrained model, such as normalization tuning \citep{basu2023strong}. The latter designs additional modules to the original network backbone to adapt to downstream tasks\,\citep{xu2023exploring,karimi2021compacter,lian2022scaling, luo2023towards}. 
Examples include LoRA\,\citep{hu2021lora,li2024learning}, adapter-based methods\,\citep{chen2022adaptformer, pfeiffer2020adapterhub, karimi2021compacter,luo2023towards}, and FACT\,\citep{jie2023fact} that tensorizes the ViT weights to a 3D tensor and reduces the tunable parameter ratio to less than $0.01\%$. We note that {\ours} differentiates from the methods above by avoiding additional inference overheads or any requirements on the model architectures.

\section{Activation Prompt: Design and Rationale}
\label{sec: method}

\textbf{Preliminaries on classical VP.} 
VP harnesses universal pixel-level perturbations applied to input images as a means of model adaptation\,\citep{bahng2022exploring}. For example, VP enables the transfer learning of an ImageNet-trained source model to various downstream tasks {without the need for fine-tuning the model weights}.
It has sparked significant interest in the recent research \citep{chen2023understanding,wu2022unleashing,bahng2022exploring}. 
Concretely, let $f_{\btheta}$  denote the pre-trained source model parameterized by $\btheta$, and $\mathcal{D} = \{(\bfx_1, y_1), (\bfx_2, y_2), \dots, (\bfx_N, y_N) \}$ denote the fine-tuning dataset for a downstream task, with $\bfx$ and $y$ being the data feature and label, respectively.
\textbf{The objective of VP} is to obtain a perturbation vector, denoted as ${\bdelta}_{\textrm{VP}}$, which is tailored to a specific task but remains agnostic to the input data. This vector is then used to transform the input data  $\bfx$ through the function $g(\bfx, {\bdelta}_{\textrm{VP}})$. 
Here $g$ symbolizes the transformation template function that molds the input image to fit the desired prompt pattern. Two prevalent templates include the addition $g(\bfx, {\bdelta}_{\textrm{VP}}) = \bfx + {\bdelta}_{\textrm{VP}}$\,\citep{zhang2022fairness,bahng2022exploring}, and the  {resize-and-concatenation} \(g(\bfx, {\bdelta}_{\textrm{VP}}) = [{\bdelta}_{\textrm{VP}}, M(\bfx)]\)\,\citep{chen2023understanding,zhang2022fairness}, where \(M\) is the resizing function. Unless specified otherwise, we consider the additive VP formulation.

\textbf{{\ours}: Generalizing VP in feature space.} The conventional VP approach primarily focuses on making direct modifications to the input data. However, this direct manipulation may have two limitations.
\textit{First}, raw input data typically contains an abundance of details, which can introduce complications for tasks like prompt generation due to issues such as background clutter and semantic ambiguity  \citep{yu2017exploiting}. In contrast, intermediate features tend to encompass a broader range of local and global attributes, preserving more class-discriminative information for decision-making \citep{bau2017network}.
\textit{Second}, parameter updates in VP demand gradient propagation throughout the entire network. Consequently, even with a lower number of tunable parameters, the training cost may increase. 
Motivated by the above, we broaden the scope of VP into the feature domain and introduce the concept of \textbf{activation prompting (AP)}, see Fig.\,\ref{fig: teaser} for an illustration. Given a neural network model with $L$ layers, represented as $\btheta = [{\btheta^{(1)}}, {\btheta^{(2)}}, \dots, {\btheta^{(L)}}]$, the output from the $l$-th layer is denoted as $\bfz^{(l)} = f_{\btheta^{(l)}}(\bfz^{(l - 1)})$, where $\bfz^{(0)} = \bfx$ (\textit{i.e.}, the input date). Similar to VP, AP at the $l$-th layer is defined by a perturbation vector $\bdelta^{(l)}$ to the intermediate feature $\bfz^{(l)}$, leading to the `prompted' feature map $g(\bfz^{(l)}, \bdelta^{(l)})=\bfz^{(l)}+\bdelta^{(l)}$. We denote the output with the $l$-th-layer AP given $\btheta$ as $f_{\btheta}(\bfx, \bdelta^{(l)})$.  
\textbf{The objective of AP} is then to optimize $\bdelta^{(l)}$ so as to facilitate the adaptation of the fixed source model $f_{\btheta}$ for performing the downstream task on $\mathcal D$.
It is evident that AP can be conceptualized as an extension of VP when we set the layer number $l$ to $0$. Moreover, the optimization process for both VP and AP can be carried out similarly through empirical risk minimization on $\mathcal D$, \textit{i.e.}, $ \min_{\bdelta^{(l)} } \frac{1}{|\mathcal{D}|}\sum_{(\bfx,y)\in \mathcal{D}}\ell(f_{\boldsymbol{\theta}}(\bfx, \bdelta^{(l)}); y)$, where $\ell$ is the sample-wise cross-entropy loss. 

AP also exhibits several notable attributes different from VP. \textit{First}, the number of parameters in AP directly relates to the size of the feature map $\bfz^{(l)}$. Hence, a properly designed AP can substantially reduce the parameter count. \textit{Second}, while the optimization of AP mirrors that of VP, its parameter update does not necessitate back-propagation throughout the entire network. For example, deeper AP within the architecture reduces computational demands during training.

\begin{figure}[htb]
    \centering
    \includegraphics[width=0.7\linewidth]{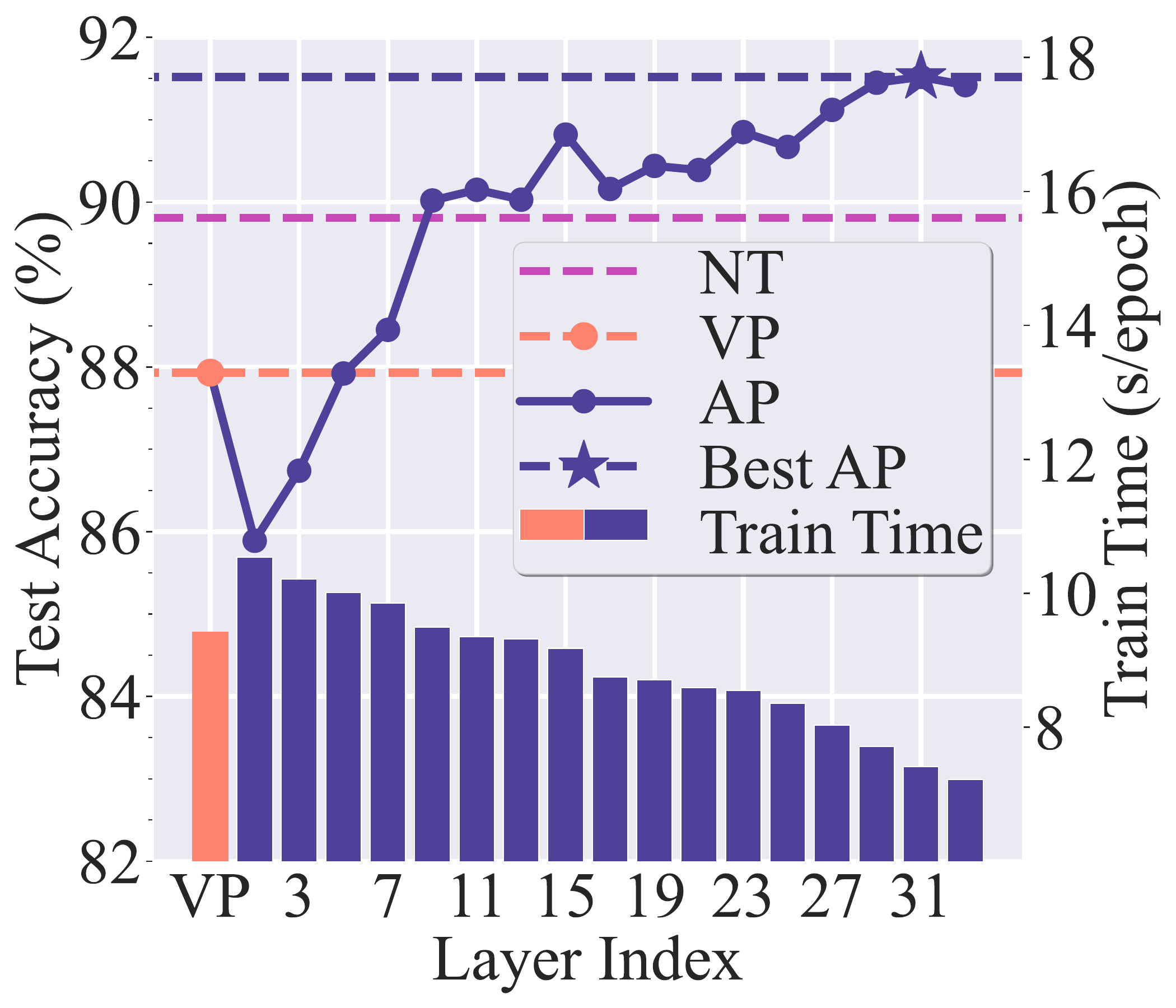}
    \vspace*{-1em}
    \caption{ Performance and efficiency comparison of {\vp}, {\nt} and {\ours} over different layers of ResNet-101 on OxfordPets.}
    \vspace*{-1em}
    \label{fig: preliminary}
\end{figure}
\textbf{AP could be a better design than VP.} Next, we present a preliminary experiment that serves as a \textit{warm-up}, demonstrating how AP exhibits the potential to improve accuracy performance, as well as enhance computation and parameter efficiency when compared to VP.
We examine the commonly used transfer learning scenario for applying VP, in which the source model ResNet-101 \citep{he2016deep} is initially trained on ImageNet \citep{deng2009imagenet} and is subsequently transferred to the CIFAR-10 dataset \citep{krizhevsky2009learning}.
\textbf{Fig.}\,\ref{fig: preliminary} presents a performance comparison between  AP and  VP against the layer index on ResNet-101, at which   AP is introduced. 
The preliminary results provide several key insights, which will be substantiated in more detail later.
\textit{First}, AP holds the potential to substantially enhance the accuracy of transfer learning when compared to VP. For instance, when AP is applied at layer 31, it achieves the highest accuracy in transfer learning, surpassing VP by approximately 5\%. In fact,   more comprehensive experiments presented in Sec.\,\ref{sec: experiments} demonstrate that applying AP to a \textit{deeper} layer consistently produces the most significant accuracy improvements across a wide range of CNNs.
\textit{Second}, due to the preference for {deeper} layers when utilizing AP in CNNs, there exists a computational advantage since back-propagation from the output to the input layer is \textit{not} required. 
\textit{Third}, AP maintains the parameter efficiency merit compared to VP. For instance, at the layer that exhibits the best performance, AP utilizes only $100k$ parameters, whereas VP employs $150k$ parameters.
The experiments above indicate that \textit{AP has the potential to outperform VP, offering not only improved accuracy but also greater efficiency}. 

\textbf{Understanding AP through its connection to normalization tuning.}
Normalization tuning (\nt), as a PEFT technique, finetunes parameters within model's normalization layers, \textit{i.e.}, BatchNorm for CNNs \citep{ioffe2015batch} and  LayerNorm for ViTs \citep{ba2016layer}. For clarity, we denote the tunable parameters of a normalization layer by $\boldsymbol{\gamma}=(\gamma_1,\cdots,\gamma_{D'})^\top$ for linear coefficients and $\boldsymbol{\beta}=(\beta_1,\cdots,\beta_{D'})^\top$ for biases, with $D'$ representing the number of channels or the token dimension.
Further, define $\boldsymbol{\mu}$ and $\boldsymbol{\sigma}$ as the
channel-wise mean and standard deviation constants of $\bfz^{(l)}$ for BatchNorm over the entire batch. For LayerNorm, they represent the data-wise mean and standard deviation of $\bfz^{(l)}$ across the embedding dimension.
Given that both AP and {\nt} utilize a linear model for feature representations, \textit{i.e.}, $g(\bfz^{(l)}, \bdelta^{(l)})=\bfz^{(l)}+\bdelta^{(l)}$ for AP and $g(\bfz^{(l)}, \boldsymbol{\gamma}, \boldsymbol{\beta}) = \boldsymbol{\gamma} \cdot (\bfz^{(l)} - \boldsymbol{\mu})/\sqrt{\boldsymbol{\sigma}} + \boldsymbol\beta$ for {\nt}, AP can be interpreted as a variant of {\nt}. 
\textbf{Fig.\,\ref{fig: APvsNT}} illustrates the connection.

$\bullet$ \textit{CNNs}: When AP's perturbations are consistent across all feature maps, the unit-scaling BatchNorm-based {\nt} closely mirrors the formulation of {\ours}, differentiated merely by a linear mapping plus a bias. This equivalence becomes apparent when relating $\bfW^{(l)}\bdelta^{(l)}$ to $\boldsymbol{\beta}-\boldsymbol{\gamma}\cdot\bfmu/\sqrt{\boldsymbol{\sigma}}$, especially when $\boldsymbol{\gamma}/\sqrt{\boldsymbol{\sigma}}=1$, supposing $\bfW^{(l)}$ as the weight for the $l$-th layer.

$\bullet$ \textit{ViTs}: Assuming uniform perturbations across tokens and consistent mean value    across data dimensions within a batch, {\ours} reduces to the unit-scaling LayerNorm-based {\nt}. 
This can be represented as $\bdelta^{(l)}=\boldsymbol{\beta}-\bfmu$, given $\boldsymbol{\gamma}/\sqrt{\boldsymbol{\sigma}}=1$.  

\begin{figure}[h]
    \vspace*{-1em}
    \centering
    \includegraphics[width=0.8\linewidth]{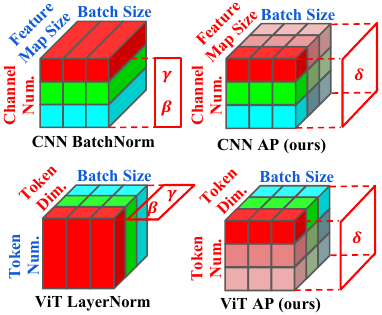}
    \vspace*{-0.5em}
    \caption{
    \textcolor{red}{Tunable parameter} shape of {\nt} and {\ours}. The same color indicates shared parameters across \textcolor{blue}{dimensions}.}
    \vspace*{-1em}
    \label{fig: APvsNT}
\end{figure}

Due to more flexible perturbations of {\ours}, such a connection exhibits increased power of {\ours} than {\nt}. 
We formally prove and summarize the proposed connection in Proposition\,\ref{prpst: AP-NT full} in Appx.\,\ref{subsec: proof-AP-NT}.
Meanwhile, we remark that another key difference of {\ours} compared to {\nt} is that no parameters of the model backbone need to be altered during training. This differentiates ``prompting'' from other PEFT methods, where the former keeps the pretrained model backbone intact. 
In the realm of PEFT, recent research has also shown that LayerNorm-based {\nt} serves as a robust baseline of model adaptation for ViTs \citep{basu2023strong}. Beyond that, we will show that {\ours} can surpass {\nt} and remain effective for CNNs.

\section{A Deep Dive into AP: Layer and Architecture Effects}
\label{sec: ap_layer}

\begin{figure*}[htb]
\centering
\centerline{
\begin{tabular}{cccc}
\hspace*{0mm}\includegraphics[width=.24\textwidth,height=!]{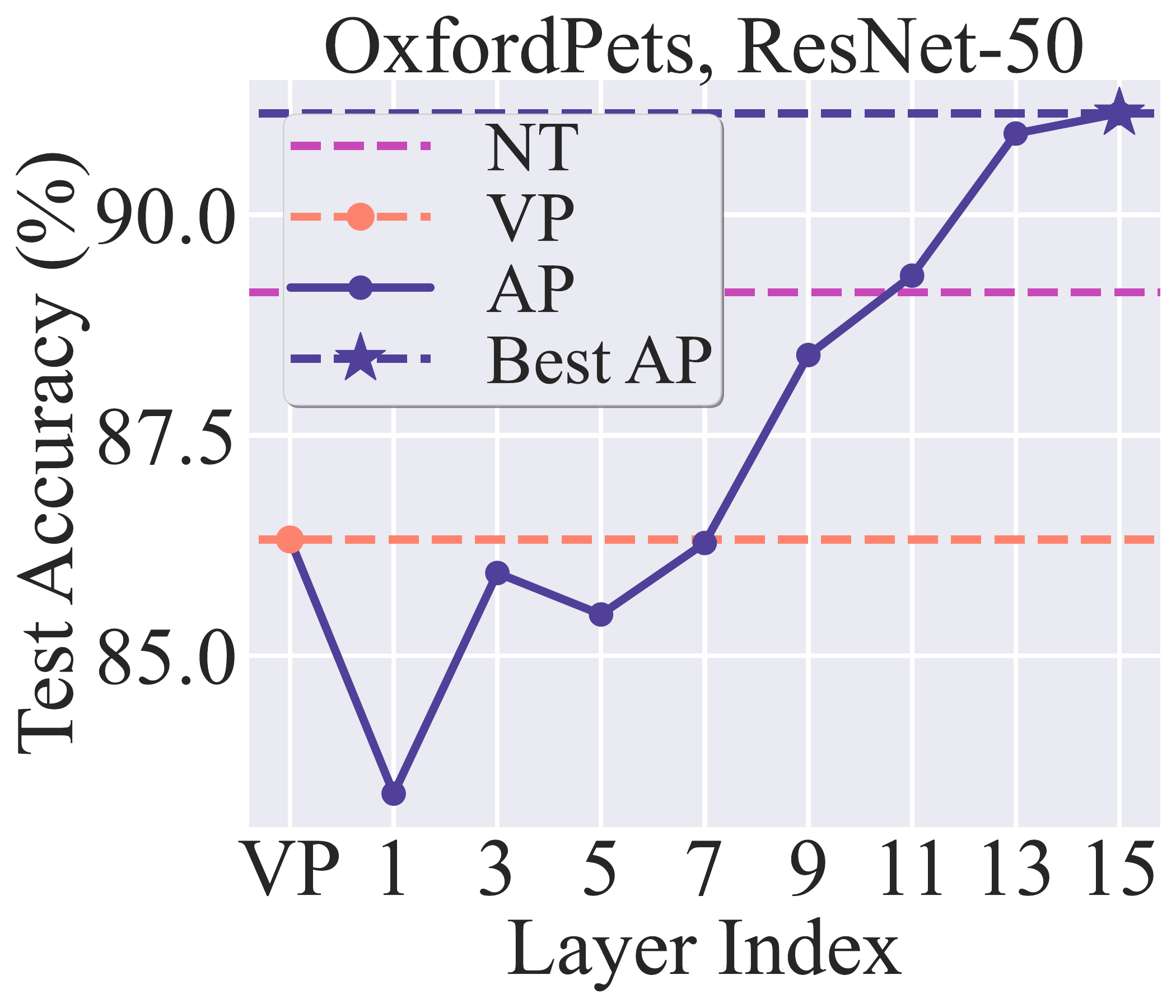}  
&\hspace*{-4mm}\includegraphics[width=.24\textwidth,height=!]{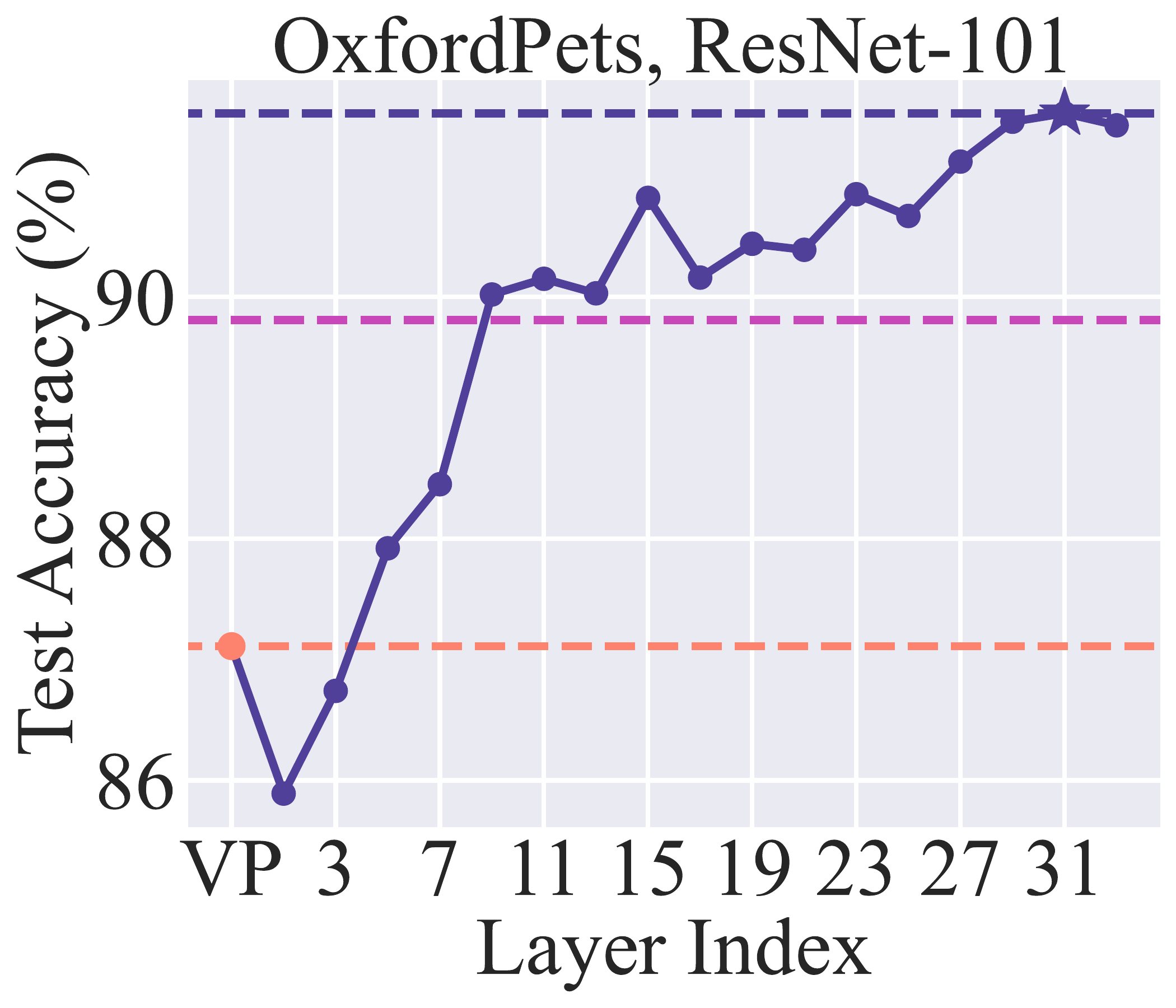}
&\hspace*{-4mm}\includegraphics[width=.24\textwidth,height=!]{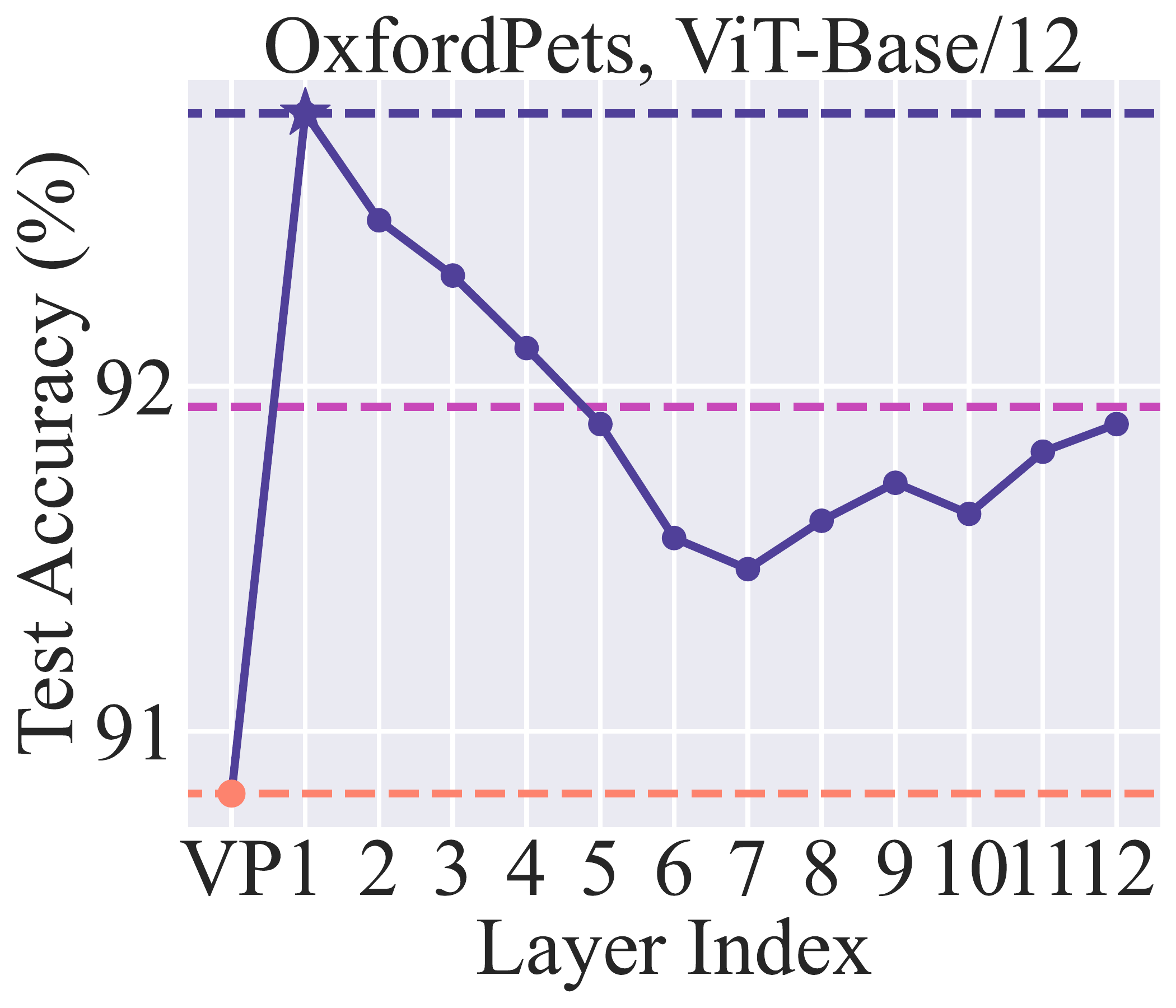}
&\hspace*{-4mm}\includegraphics[width=.24\textwidth,height=!]{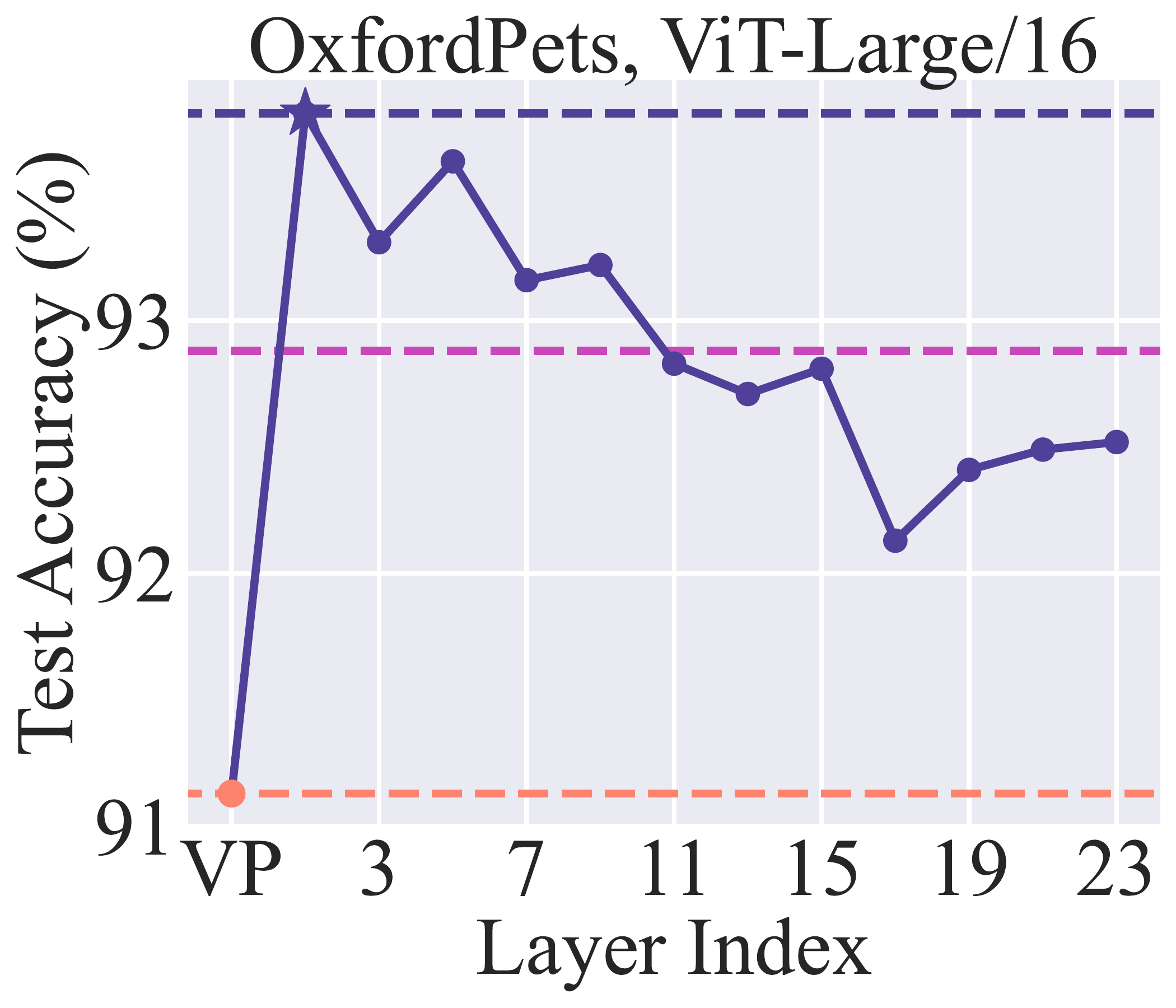}
\end{tabular}}
\vspace*{-0.5em}
\caption{{Layer preference of {\ours} with different model architectures on OxfordPets\,\citep{parkhi2012cats}. CNNs and ViTs exhibit opposite layer preferences. Results on more datasets are provided in Fig.\,\ref{fig: layer_effect_complete}.
}}
\vspace*{-1em}
\label{fig: layer_effect}
\end{figure*}

Our findings in Fig.\,\ref{fig: preliminary} suggest that the effectiveness of {\ours} may be contingent on the \textit{specific layer} where it is installed. For a deeper understanding of this characteristic and its association with \textit{model architecture}, we examine both ResNet and ViT model types.

\textbf{Fig.\,\ref{fig: layer_effect}} follows and expands Fig.\,\ref{fig: preliminary} by covering the additional models, \textit{i.e.}, ResNet-50, ViT-Base/12, and ViT-L/16, and showcasing the transfer learning accuracy enabled by AP on the downstream dataset OxfordPets as a function of the layer index to which AP is applied. 
As we can see, a key observation is that \textit{ResNets and ViTs exhibit contrasting layer preferences for {\ours}}, where \ding{72} indicates the best performance of AP in Fig.\,\ref{fig: layer_effect} under each architecture. 
Specifically, CNNs exhibit a preference for {\ours} in their \textit{deeper} layers, while ViTs tend to favor {\ours} in their \textit{shallower} layers. Moreover, within the comfort layer zone, the performance of {\ours} consistently outperforms {\nt}.

\begin{figure}[h]
\centering
\centerline{
\begin{tabular}{cc}
\includegraphics[width=.5\linewidth,height=!]{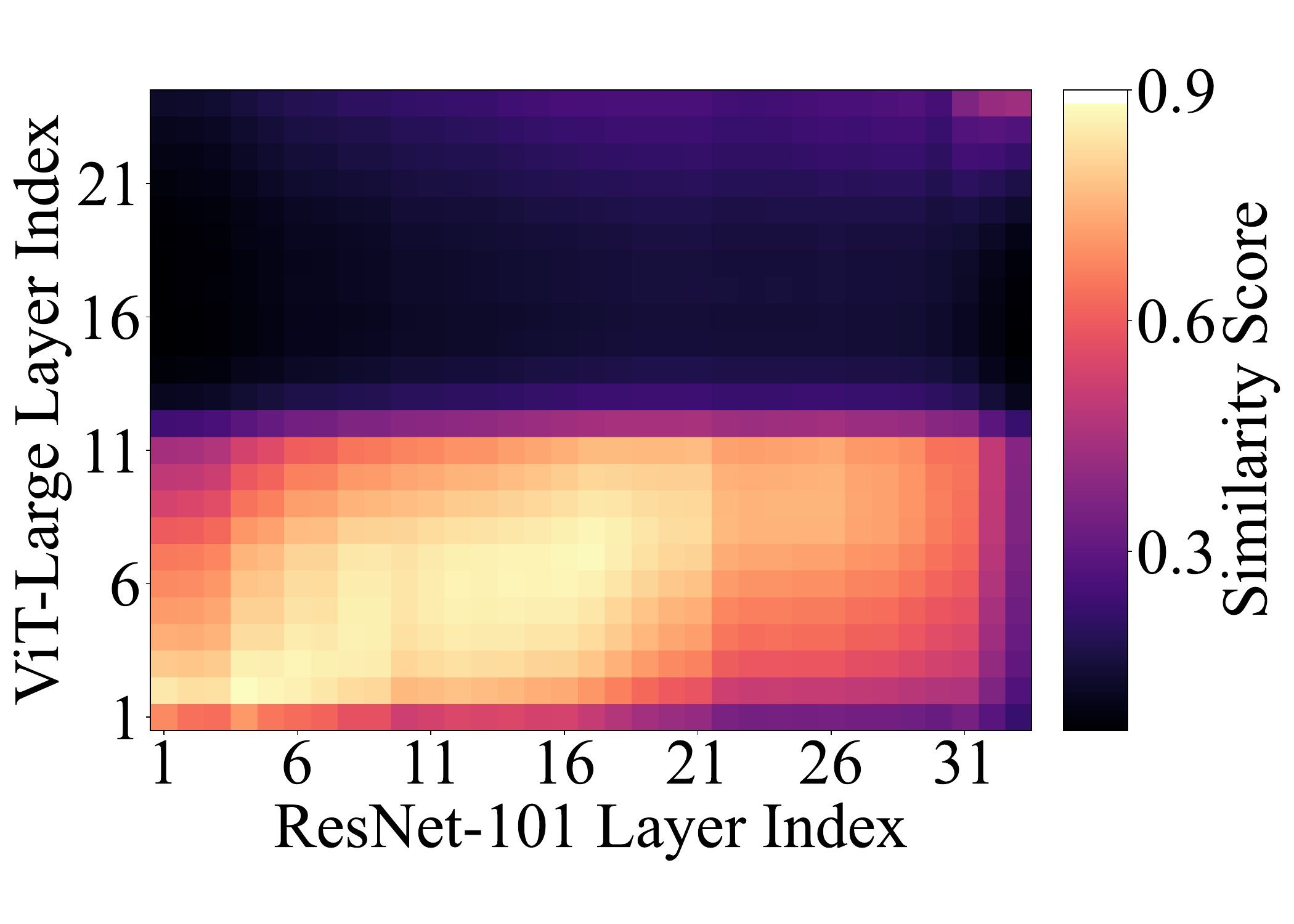}  
&
\hspace*{-4mm}
\includegraphics[width=.5\linewidth,height=!]{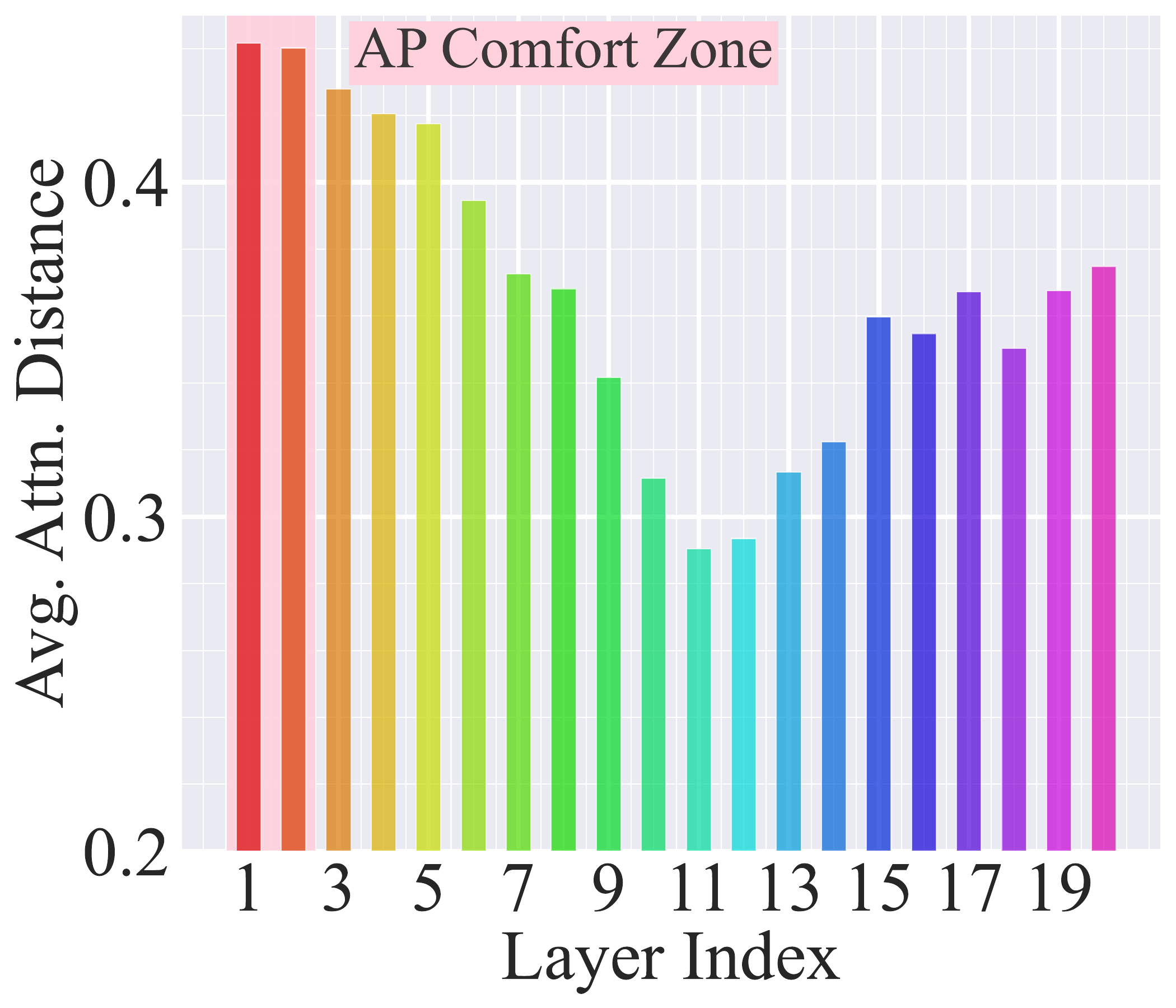}
\end{tabular}}
\vspace*{-0.5em}
\caption{{
Features dissection to understand the layer effect. (Left) CKA-based feature similarity comparison between ViT-L/16 and ResNet-101. 
(Right) Average attention distances across all the heads and layers of ViT-L/16. Larger distance indicates global features.}}
\vspace*{-0.5em}
\label{fig: feature_dissection}
\end{figure}

\textbf{Dissecting CNNs and ViTs: AP prioritizes `global' features over `local' features.} To unpack the intriguing AP's layer preference behavior above, we next examine the features captured by different layers of CNNs and ViTs. 
To this end,  we first employ the Centered Kernel Alignment (CKA)-based feature similarity analysis\,\citep{cortes2012algorithms} to measure the layer-wise representation similarity between CNNs and ViTs, \textit{e.g.},
ResNet-101 and ViT-L/16 in \textbf{Fig.\,\ref{fig: feature_dissection}}.
As we can see, the deep features of ResNet-101 predominantly align with the middle layers of ViT-L/16. This concurs with the observations made in \citep{raghu2021vision}, which suggest that ViTs have the capability to capture features reminiscent of the deeper layers of CNNs even within their relatively early layers.
In addition, as indicated by network dissection analysis for CNNs \citep{bau2017network}, CNNs tend to prioritize low-level visual concepts, \textit{i.e.}, \textit{local features}, like color and texture, in their shallower layers. In contrast, they transition to high-level, class-discriminative concepts, encompassing  \textit{global features} like scenes and objects in deeper layers.

Drawing upon the analyses presented above and insights in \textbf{Fig.\,\ref{fig: layer_effect}}, \textit{we hypothesize} that {\ours} exhibits a preference for deep layers in CNNs and shallow layers in ViTs, which can be attributed to the models' inclinations toward global features over local features.
To bolster our hypothesis, we investigate how global information is distributed across the layers of ViTs. We employ a methodology used in \citep{raghu2021vision} and \citep{walmer2023teaching} to compute the average attention distance between the position of query tokens and the locations they attend to with the query within each self-attention head in ViTs. This analysis unveils how each self-attention layer contributes to the balance between local and global information in the overall representation.
In \textbf{Fig.\,\ref{fig: feature_dissection}} (B), we present the average attention distance across 16 attention heads for with different layer indices of a pretrained ViT-L/16. A general trend can be observed: the distribution of the sorted attention distance moves firstly downwards (layer index from $1$ to layer $12$). This implies that the ratio of the global features captured by attention in general decreases. When the layer index is larger than $15$, the global feature ratio slightly increases. This trend roughly aligns well with the patterns observed in {Fig.\,\ref{fig: layer_effect}}. These observations underscore our claim that {\ours}'s layer preference is influenced by the presence of global features.
We provide theoretical support in the next section to support the layer and architecture effect. Particularly, we focus on the more challenging part of ViTs, since the study on CNNs is abundant. Furthermore, we provide theoretical support in the following section to support the layer and architecture effect.

\textbf{Remark on the comparison of AP vs. VPT}. While VPT \citep{jia2022visual} also suggests adding extra tokens (prompts) to all intermediate layers of a ViT, our approach differs fundamentally. 
AP rigorously explores the optimal layer selection for effective prompting, where (input-level) VP is covered as a special case. Unlike VPT, AP uncovers new insights into layer-specific effects, architectural dependencies, and their explanations, supported by both empirical and theoretical analyses. Our findings show that strategic layer selection in AP can match or surpass the effectiveness of VPT's multi-layer prompting (see Tab.\,\ref{tab: more_peft}). 

\section{Theoretical Analyses for Layer and Architecture Effects}
\label{sec: theory}

From a generalization perspective, we focus on the layer and architecture effect for \textbf{ViTs}:
\textit{To achieve the desired generalization performance (or test accuracy), will shallow-layer {\ours} tuning require less sample complexity than deep ones for ViTs? }
If so, with the same sample complexity, shallow-layer AP could achieve better performance. To show this, we present the theoretical setups that satisfy the conditions of global features for ViTs, followed by the generalization analysis with sample complexity bound in Theorem \ref{thm: ViT}. 

\textbf{Problem setup.}
Building on the theoretical frameworks for analyzing the training and generalization of Transformers \citep{li2023theoretical, oymak2023role}, we derive theoretical insights by considering a binary classification problem. We use a single-head, two-layer ViT \citep{huang2023context, nichani2024transformers} as the pretrained model, applied to the dataset $\{\bfx_n, y_n\}_{n=1}^N$.
Here $y_n\in\{+1,-1\}$, and each data  $\bfx_n\in\mathbb{R}^{d\times P}$ consists of $P$ tokens. The training is implemented by a stochastic gradient descent (SGD) with the loss $\ell(f_{\btheta}(\bfx_n, \bdelta); y_n)$, where $f_{\btheta}$ and $\bdelta$ are the pretrained model and the trainable {\ours}. The generalization is evaluated by the population risk $\mathbb{E}[\ell(f_{\btheta}(\bfx, \bdelta); y)]$.

\textbf{Data assumption.} Each token of $\bfx_{n}$ is formulated as a pattern added with a Gaussian noise following $\mathcal{N}(0,\sigma^2)$, $\sigma\leq O(1/P)$. We consider four patterns $\{\bfv_1,\bfv_2,\bfv_3,\bfv_4\}$. In each $\bfx_n$, only one token corresponds to either $\bfv_1$ or $\bfv_2$, named discriminative patterns that decide the label. Other $P-1$ tokens correspond to either $\bfv_3$ or $\bfv_4$, named irrelevant non-discriminative patterns ones for the downstream task. For instance, if one token within $\bfx_n$ is the noisy version of $\bfv_1$ ($\bfv_2$), then its downstream task label $y^n=1$ ($y^n=-1$). 

\textbf{Pretrained model assumption.} 
We have mild assumptions on the MLP neuron weights and self-attention matrices of the pretrained model, which have been used and numerically verified in existing works. Recent SOTA theoretical findings \citep{zhenmei2022theoretical, li2023theoretical, WL21} reveal, during pretraining, the weights of each neuron in the MLP tend to converge towards one of the patterns present in the raw data, \textit{e.g}, $\bfv_1,\bfv_3$. Following the observation above, we assume neuron weights in the $\ell$-th MLP after pretraining to be one of the patterns in $\{\bfv_1, \bfv_2, \bfv_3, \bfv_4 \}$. Typically, $\bfv_1$ and $\bfv_2$ are patterns in the downstream task that have relevance to the labels, while $\bfv_3$ and $\bfv_4$ are ones also in the downstream task but do not bear a relation to the labels. Plus, as suggested by the global features introduced in Sec.\,\ref{sec: ap_layer} that make tokens attend to other tokens, we assume the key and value matrices to be scalings of permutation matrices.
Detailed data and model assumptions can be found in Appx.\,{\ref{subsec: prop3}}. 

Given a set of queries $\bfq_1,\cdots,\bfq_P$ and keys $\bfk_1,\cdots,\bfk_P$ for an attention head, we formally define the \textit{average attention distance} mentioned in \textbf{Fig.\,\ref{fig: feature_dissection}} as $\sum_{i=1}^P |i-\arg\max_{j\in[P]}\left\langle \bfk_j, \bfq_i\right\rangle|/P$, 
i.e., the average distance between the query $\bfq_i$ and the key $\bfk_j$ that has the largest inner product with $\bfq_i$, $i,j\in[P]$. 
Assuming the discriminative key and value are away from the discriminative query with a distance of $d_A\geq 1$, we have the following Lemma on decreasing the average attention distance.

\vspace*{-2mm}
\begin{lemma}\label{lemma: delta-0}
    The average attention distance defined above decreases from $(1+d_A)/P$ to $1/P$ after the 1st layer of the simplified two-layer ViT.
\end{lemma}
\vspace*{-2mm}

Lemma \ref{lemma: delta-0}  supports our empirical observation in \textbf{Fig.\,\ref{fig: feature_dissection} (B)} of decreasing attention distance values within deep layers in ViT.
In addition, the reduction in the attention distance leads to an increased sample complexity, as summarized in the following theorem.

\vspace*{-2mm}
\begin{thm}\label{thm: ViT}
  Training  a 2-layer ViT with SGD returns a model with zero generalization error, as long as the batch size $B\geq \Omega(1)$, and the required number of samples $N$ satisfy either (i) $N\geq N_1=\Theta(P)$ if adding {\ours} to the 1st layer; (ii) $N\geq N_2=\Theta(P^2\log P)$ if adding {\ours} to the 2nd layer. $N_2$ is order-wise larger than $N_1$. 
\end{thm}

\begin{figure}[t]
    \centering
    \includegraphics[width=0.6\linewidth,height=!]{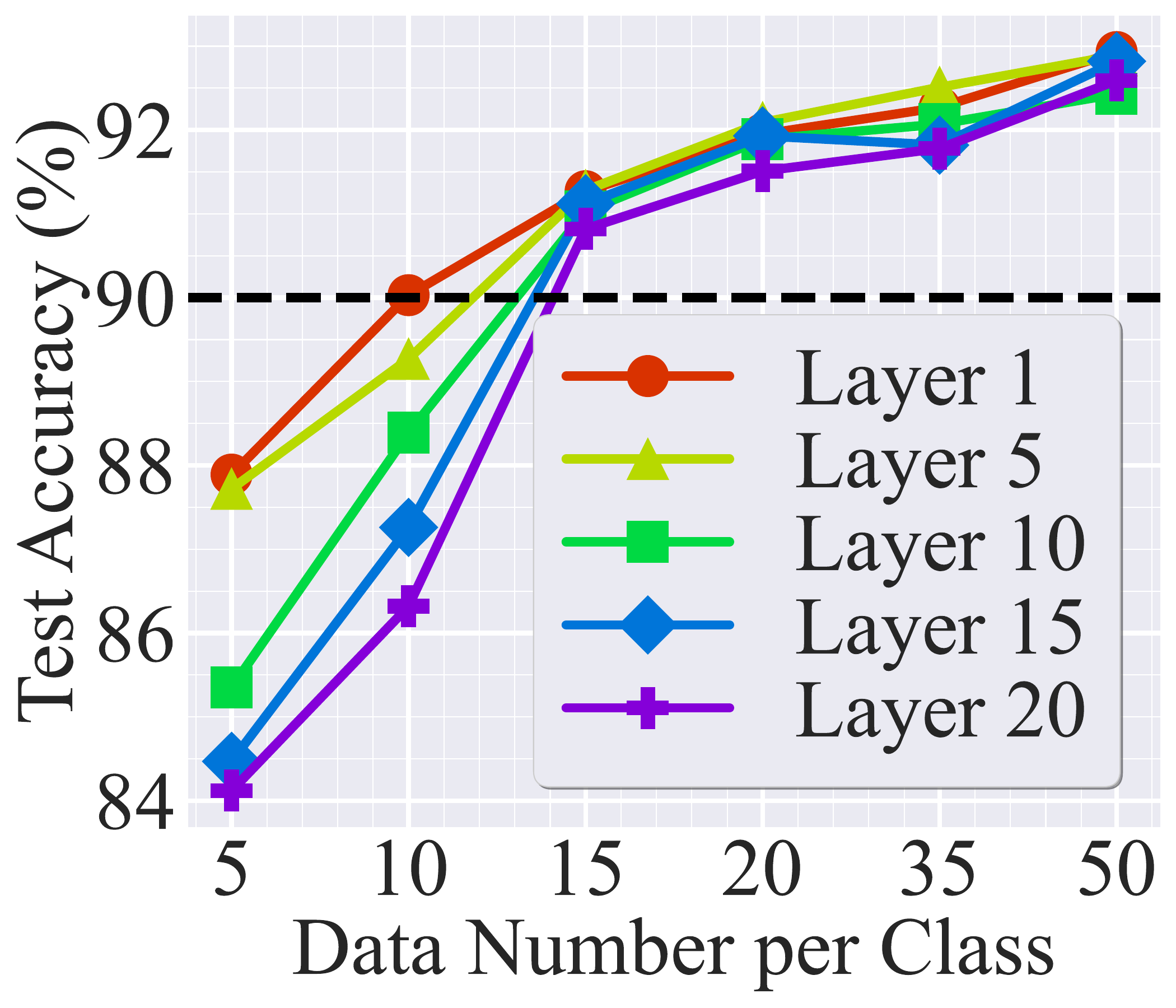}
    \vspace*{-0.5em}
    \caption{{Sample complexity study of {\ours}.}}
    \label{fig: sample_complexity}
    \vspace*{-1em}
\end{figure}

Theorem \ref{thm: ViT} shows deep-layer {\ours} requires more training samples than the shallow one to achieve the same generalization, as shown by the dashed line in \textbf{Fig.\,\ref{fig: sample_complexity}}. Accordingly, with the same number of training samples and setup, shallow-layer {\ours} generalizes better. The proof of Theorem \ref{thm: ViT} can be found in Sec.\,\ref{subsec: proof thm ViT}. The basic proof idea is that for {\ours} in the shallow layer, a trained prompt with a norm of $\Theta(P)$ that removes non-discriminative patterns is enough to make all tokens attend to discriminative tokens. Thus, the amount of global features does not decrease. This can ensure zero generalization by abundant global features. For {\ours} in deep layers, however, given Lemma.\,\ref{lemma: delta-0}, a lack of global features leads to an evident mismatch between discriminative tokens in the 2nd-layer self-attention. Hence, a trained prompt with a norm of $\Theta(P^2\log P)$ is necessary to direct the attention to focus on discriminative tokens. The proof concludes with the demonstration that the sample complexity bound is proportional to the the trained prompts magnitude.

\section{Experiments}
\label{sec: experiments}
\subsection{Experiment Setup}

\begin{table*}[htb]
\caption{{Performance comparison on 19 datasets from 3 benchmarks. Three parameter-efficient baselines (denoted by {\partcirc}) are compared to {\ours}, where the best performance is highlighted in \textbf{bold}. The most expensive {\ff} (denoted by {\fullcirc}) serves as the performance reference. Each accuracy value is averaged over 5 trials, with variance omitted due to its negligible values ($\leq 0.3\%$).
``Average'' column averaged accuracies over all the datasets.
}}
\fontsize{6pt}{6pt}\selectfont
\newcolumntype{C}{>{\centering\arraybackslash}X}
\setlength{\tabcolsep}{2.2pt}
\setlength{\extrarowheight}{7pt}
\renewcommand{\arraystretch}{0.6}
\begin{tabularx}{\linewidth}{p{5pt}p{1.8cm}!{\color{lightgray}\vline}CCCCC!{\color{lightgray}\vline}CCCCCCCCC!{\color{lightgray}\vline}CCCcC!{\color{lightgray}\vline}C}
\toprule[1pt]
& \multicolumn{1}{c}{\underline{\textbf{Benchmark}}}
& \multicolumn{5}{c}{\underline{\textbf{FGVC}}} 
& \multicolumn{9}{c}{\underline{\textbf{VTAB}}} 
& \multicolumn{5}{c}{\underline{\textbf{Others}}} \\
\rotatebox{90}{\raisebox{1.0pt}{\textbf{Architecture}}}
& 
& \rotatebox{90}{\raisebox{1.0pt} {\cub}} 
& \rotatebox{90}{\raisebox{1.0pt} {\dog}}
& \rotatebox{90}{\raisebox{1.0pt} {\stanfordcars}}
& \rotatebox{90}{\raisebox{1.0pt} {\nabirds}}
& \rotatebox{90}{\raisebox{1.0pt} {\oxfordflowers}}
& \rotatebox{90}{\raisebox{1.0pt} {\cifarhun}}
& \rotatebox{90}{\raisebox{1.0pt} {\caltech}}
& \rotatebox{90}{\raisebox{1.0pt} {\dtd}}
& \rotatebox{90}{\raisebox{1.0pt} {\flowers}}
& \rotatebox{90}{\raisebox{1.0pt} {\oxfordpets}}
& \rotatebox{90}{\raisebox{1.0pt} {\svhn}}
& \rotatebox{90}{\raisebox{1.0pt} {\sun}}
& \rotatebox{90}{\raisebox{1.0pt} {\camelyon}}
& \rotatebox{90}{\raisebox{1.0pt} {\eurosat}}
& \rotatebox{90}{\raisebox{1.0pt} {\cifarten}} 
& \rotatebox{90}{\raisebox{1.0pt} {\gtsrb}}
& \rotatebox{90}{\raisebox{1.0pt} {\ucf}}
& \rotatebox{90}{\raisebox{1.0pt} {\food}} 
& \rotatebox{90}{\raisebox{1.0pt} {\waterbirds}}
& \rotatebox{90}{\raisebox{1.0pt} Average} \\
\midrule
\multirow{6}{*}{\rotatebox{90}{\textbf{ResNet-101}}}
& {\fullcirc} {\ff} 
& 88.91 &  90.13 &  87.76 &  84.45 &  99.98 &  92.24 &  99.13 &  79.97 &  99.81 &  90.49 &  97.14 &  79.19 &  91.13 &  99.13  & 97.24 &  97.68 &  88.32 &  82.72 & 96.69 & 91.69 \\
\cmidrule{2-22}
& {\partcirc} {\lp}  
& 63.76 &  86.63 &  49.62 &  52.09 &  82.01 &  73.87 &  90.58 &  61.35 &  93.14 &  91.17 &  66.30 &  54.51 &  83.36 &  95.84 &  92.25 &  79.64 &  71.03 &  64.31 &  88.11 & 75.76
\\
& {\partcirc} {\nt} 
& \stress{66.39} &  \stress{87.59} &  \textbf{67.64} &  \stress{56.72} &  66.50 &  \textbf{82.58} &  91.32 &  63.53 &  \stress{92.85} &  89.81 &  \textbf{95.26} &  54.56 &  \stress{84.42} &  \stress{96.14} &  \stress{93.90} &  \textbf{96.43} &  69.44 &  \textbf{72.54} &  \textbf{88.95} & 79.81
\\
& {\partcirc} {\vp} 
& 65.72 &  86.91 &  51.04 &  54.23 &  \stress{78.50} &  72.01 &  \stress{93.51} &  63.12 &  90.17 &  \stress{87.93} &  80.68 &  \stress{54.97} &  83.71 &  95.44 &  92.55 &  83.18 &  \stress{66.30} &  57.89 &  86.71 & 76.03
\\
& {\partcirc} {\ours} (ours)
& \textbf{69.42} &  \textbf{87.79} &  \stress{59.06} &  \textbf{58.31} &  \textbf{85.14} &  \stress{76.94} &  \textbf{94.85} &  \textbf{69.80} &  \textbf{95.13} &  \textbf{91.31} &  \stress{87.30} &  \textbf{56.83} &  \textbf{84.91} &  \textbf{97.21} &  \textbf{94.08} &  \stress{90.43} &  \textbf{73.96} &  \stress{68.12} &  \stress{88.13} & \textbf{80.45}
\\

\midrule
\multirow{6}{*}{\rotatebox{90}{\textbf{ViT-L/16}}}
& {\fullcirc} {\ff} 
& 89.79 &  93.31 &  89.42 &  84.75  &  99.91 &  93.19 &  99.25 &  75.30 &  99.39 &  93.35 &  98.13 &  79.31 &  91.93 &  97.92 &  98.30 &  97.90 &  89.25 &  86.16 & 97.93 & 92.34 \\
\cmidrule{2-22}
& {\partcirc} {\lp}   
& 84.69 &  86.11 &  65.24 &  75.71  &  \stress{99.40} &  88.55 &  97.01 &  73.31 &  99.24 &  91.15 &  65.79 &  72.37 &  84.05 &  97.26 &  98.13 &  80.72 &  83.02 &  83.02 & 94.16 & 85.20 \\
& {\partcirc} {\nt} 
& \stress{85.90} &  \stress{89.76} &  \textbf{75.61} &  \stress{78.78}  &  99.35 &  \stress{90.69} &  \stress{98.01} &  \stress{78.90} &  \stress{99.76} &  \stress{92.88} &  \stress{88.30} &  73.57 &  79.82 &  97.17 &  98.44 &  \stress{90.86} &  \stress{85.15} &  \stress{83.21} & 94.36 & \stress{88.45} \\
& {\partcirc} {\vp} 
& 85.24 &  87.02 &  67.64 &  76.20  &  99.32 &  89.44 &  97.81 &  77.72 &  99.72 &  91.31 &  85.70 &  \stress{74.33} &  \stress{84.27} &  \stress{97.85} &  \textbf{98.80} &  89.09 &  84.67 &  82.23 & \textbf{95.03} & 87.54 \\
& {\partcirc} {\ours} (ours)
& \textbf{86.74} &  \textbf{90.83} &  \stress{69.41} &  \textbf{79.83} &  \textbf{99.70} &  \textbf{90.96} &  \textbf{98.99} &  \textbf{78.96} &  \textbf{99.84} &  \textbf{93.89} &  \textbf{88.87} &  \textbf{75.44} &  \textbf{86.99} &  \textbf{98.33} &  \stress{98.54} &  \textbf{91.49} &  \textbf{86.80} &  \textbf{84.04} & \stress{94.60} & \textbf{89.17} \\
\bottomrule[1pt]
\end{tabularx}
\vspace*{-1em}
\label{tab: main}
\end{table*}

\textbf{Datasets and models.}
We utilize two commonly used architectures for the source datasets: ResNet-101 from the ResNet family \citep{he2016deep} and ViT-L/16 from the ViT family \citep{dosovitskiy2020image}, both pretrained on ImageNet-1K \citep{russakovsky2015imagenet}. In the target domain, we consider over 20 datasets from transfer learning benchmarks {\fgvc}\,\citep{maji13fine-grained}  and {\vtab} \citep{zhai2019large}. In {\vtab}, we consider both \textit{full-data} and \textit{few-shot} ({\vtabk}) regimes. We also consider other commonly used datasets \citep{chen2023understanding} for transfer learning like {\cifarten}\,\citep{krizhevsky2009learning}, {\ucf}\,\citep{soomro2012ucf101}, {\gtsrb}\,\citep{Houben-IJCNN-2013}, {\food}\,\citep{bossard2014food}, and {\waterbirds}\,\citep{sagawa2019distributionally}. More details can be found in Appx.\,\ref{app: exp_detail}. We cover three types of baselines in transfer learning. \textit{First}, we primarily compare {\ours} to finetuning methods designed for both CNNs and ViTs in transfer learning. These include  {\lp} that only finetunes the classification head, the conventional (input-level) {\vp}\,\citep{bahng2022exploring} and {\nt}\,\citep{basu2023strong}, that tunes \textit{all} the normalization layers. \textit{Second}, we select {\ff} as reference method due to its superior accuracy, which fine-tunes the entire model, albeit being the most computationally expensive. \textit{Third}, we use other 9 SOTA PEFT baselines used in ViTs: {\vpt}\,\citep{jia2022visual}, \textsc{GateVPT}\,\citep{yoo2023improving}, \textsc{E2VPT}\,\citep{han20232vpt}, {\lora}\,\citep{hu2021lora},  {\adapter}\,\citep{chen2022adaptformer}, {\bias}\,\citep{zaken2021bitfit}, {\nt}\,\citep{basu2023strong}, {\scale}\,\citep{basu2023strong}, {\adapterformer}\,\citep{chen2022adaptformer}, and SSF\,\citep{lian2022scaling}.

\textbf{Implementation, training, and evaluations.} 
{\ours} is installed at the input of the third-to-last ResNet block in ResNet-101 and the third Transformer block in ViT-L/16, based on the layer effect. All the methods are trained for 100 epochs using the Cross-Entropy loss with an Adam optimizer\,\citep{KingmaB2015adam}. Hyperparameters, including learning rates, are determined by a grid-search for each method; see in Appx.\,\ref{app: exp_detail}. We compare different methods in terms of performance (test acc.) and efficiency. We evaluate efficiency from $4$ perspectives: (1) tunable parameter number, (2) memory cost, (3) train time per epoch, and (4) inference throughput, as shown in Tab.\,\ref{tab: baseline_overview}.

\subsection{Experiment Results}

\textbf{{\ours} is not only effective but also efficient.} The performance of {\ours} in the full-data regime is shown below.  \textit{Two key observations} can be drawn from results: (1) {\ours} consistently outperforms baselines in most datasets, in particular with a significant improvement over {\vp} (Tab.\,\ref{tab: main}); (2) {\ours} demonstrates remarkable efficiency across various efficiency metrics, proving itself as a cost-effective method (Tab.\,\ref{tab: baseline_overview}).

\underline{\textbf{Tab.\,\ref{tab: main}}} shows the performance of {\ours} vs. baselines: {\vp}, {\nt}, {\lp}, and {\ff}. As we can see, {\ours} consistently outperforms {\vp} in \textit{all} $19$ datasets. {\ours} yields an increase in the average accuracy of $4\%$ and $1.5\%$ compared to {\vp} for both model types. 
In some datasets, such as StanfordCars, SVHN and GTSRB using ResNet-101, this advantage can increase to $7\% $$\sim$$ 9\%$. 
{\ours} also remains effective compared to {\nt}, which has proven to be a strong PEFT method for ViTs in (\citet{basu2023strong}). {\ours} performs the best in 13/15 out of 19 datasets for ResNet-101/ViT-L/16, respectively. {\ff} remains the best-performing in most datasets, but {\ours} still manages to approach or surpass it; see OxfordPets for ResNet-101 and DTD for ViT-L/16. Notably, {\ours} is much more efficient than {\ff}, as shown below.

\begin{table}[htb]
\caption{{A performance overview. The efficiency analysis is on (ViT-L, {\cifarten}) with a batch size of 128, and time consumption is based on using a single RTX-A6000. $\uparrow$ or $\downarrow$ indicate whether a larger or smaller value is favored.}}
\vspace*{-5mm}
\label{tab: baseline_overview}
\begin{center}
\resizebox{\linewidth}{!}{%
\begin{tabular}{c|c|ccc}
\toprule[1pt]
\multirow{3}{*}{\textbf{Method}} 
& \textbf{Param. Efficiency} 
& \multicolumn{3}{c}{\textbf{Train-Time Efficiency}} \\
& \begin{tabular}[c]{@{}c@{}}Parameter \\ \# (M) $\downarrow$ \end{tabular} 
& \begin{tabular}[c]{@{}c@{}}Memory Cost \\(G) $\downarrow$\end{tabular} 
& \begin{tabular}[c]{@{}c@{}}Time Cost \\(s/epoch) $\downarrow$\end{tabular} 
& \begin{tabular}[c]{@{}c@{}}Troughput \\(image/s) $\uparrow$\end{tabular} 
\\
\midrule
\multicolumn{5}{c}{\textbf{ResNet-101}} \\
\midrule
{\ff} & 44.5  & 10.32 & 118 & 41.47   \\
{\lp}  & {0.02}  & 6.2 & 39 & 41.33  \\
\midrule
{\nt}  & {0.13}  & 11.7 & 83 & \textbf{41.45}    \\
{\vp}  & {\textbf{0.12}}  & 12.2 & 72 & 40.59  \\
\rowcolor{Gray}
{\ours}  & {\textbf{0.12}}  & \textbf{6.3} & \textbf{41} & 41.36  \\
\midrule
\multicolumn{5}{c}{\textbf{ViT-L/16}}\\
\midrule
{\ff} & 304.33  & 41.5 & 520 & 79.58  \\
{\lp}  & {0.01}  & 9.7 & 121 & 79.64  \\
\midrule
{\nt}  & \textbf{{0.06}}  & \textbf{29.5} & 285 & \textbf{79.51}    \\
{\vp}  & {0.11}  & 35.9 & 280 & 77.14  \\
\rowcolor{Gray}
{\ours}  & {{0.16}}  & 31.6 & \textbf{262} & \textbf{79.48}  \\
\bottomrule[1pt]
\end{tabular}%
}
\vspace*{-2em}
\end{center}
\end{table}

\underline{\textbf{Tab.\,\ref{tab: baseline_overview}}} shows the efficiency   evaluations. Two key insights can be drawn. \textit{First}, in comparison to {\vp}, {\ours} demonstrates superior efficiency regarding reduced memory overhead, decreased training duration, and increased throughput for both models. This superiority is maintained at a comparable parameter efficiency, marked by a negligible ratio difference of $\leq 0.05\%$. 
This trend is amplified for ResNet-101 by the significant reductions in memory usage ($6.3$G for {\ours} vs. $12.2$G for {\vp}) and training duration ($41$ s/epoch for {\ours} vs. $72$ s/epoch for {\vp}). This arises from the AP's preference towards deeper layers over shallower ones in ResNet-101 and a reduced propagation overhead. \textit{Second}, when compared to {\nt}, although {\ours} consumes slightly higher memory cost for ViT-L/16, it achieves higher training efficiency for both models, because while the parameters of {\nt} are dispersed throughout the network, making back-propagation more expensive. Though no obvious difference is observed in throughput, we will show later in Tab.\,\ref{tab: more_peft} that {\ours} enjoys high throughput compared to other PEFT methods.

\textbf{How does the downstream dataset scale  affect {\ours}?} To study the effect brought by the downstream data scales, we follow the setting of \citet{jia2022visual} and examine the performance of different methods under the few-shot setting on VTAB-$1K$. In particular, for each of the 19 datasets in the VTAB benchmark, only 1000 data samples are available for training. \textbf{Tab.\,\ref{tab: vtab_1k}} shows that {\ours}  makes a distinguishable improvement over the baselines {\vp} and {\nt} in the few-shot setting. As we can see, {\ours} achieves a performance boost of over $1\%$ than {\vp} using ViT-L/16 and this advantage increases to $4.3\%$ in the case of ResNet-101. This demonstrates that directly steering the intermediate features can be more effective when facing data scarcity.

\textbf{{Comparing {\ours} with VPT and more PEFT baselines.}} 
{As VP is introduced as a generalization of the conventional (input-level) AP, we do not anticipate it to outperform all model-based PEFT methods. Yet, to demonstrate its potential,  \textbf{Tab.\,\ref{tab: more_peft}} compares the performance of AP with that of PEFT baselines, in particular with VPT \citep{jia2022visual}.} As we can see, even when compared to the stateful PEFT methods, {\ours} still yields competitive performance in terms of both accuracy and efficiency. For example, {\ours} ranks roughly $2$$\sim$$4$ in terms of accuracy among the 8 PEFT methods considered in this work. In addition, {\ours} ranks the first from the efficiency perspective. In contrast, the best accuracy performance of {\adapterformer} comes at a cost of three times lower throughput efficiency. This is due to that extra modules introduce significantly more computations during the inference.

\setlength{\tabcolsep}{4pt}
\begin{table}[t]
\vspace*{-0.5em}
\caption{\footnotesize{Performance of {\ours} and more SOTA PEFT methods on ViT-L/16. Settings follow Tab.\,\ref{tab: main}.
}}
\vspace*{-0.5em}
\label{tab: more_peft}
\begin{center}
\resizebox{\linewidth}{!}{
\begin{tabular}{
l
ccc  !{\color{lightgray}\vrule}
c ccc cc 
}
\toprule[1pt]

& \multicolumn{3}{c}{\textbf{Accuracy}}
& \multicolumn{4}{c}{\textbf{Efficiency}}

\\
\cmidrule{2-8}
&\multicolumn{3}{c!{\color{lightgray}\vrule}}{\textbf{Full-Data}}
& 
&\multicolumn{3}{c}{\textbf{Train-Time Efficiency}}
\\
&{\small{FGVC}} &{\small{VTAB}} &{\small{Others}}
&{Param. \#} &{\small{Memory}} &{\small{Time}} &{\small{Throughput}} 
\\
\midrule
Number of tasks &5 &9 &5 & -& - & - & -
\\
\midrule 
{\ff} 
& 91.43 & 91.97 & 93.91
& 304.33 &  41.5  &  520 &  79.58
\\
{\lp} & 82.23  & 78.90 & 87.81
& {0.01} &  9.7 & 121  & 79.64
\\

\midrule
\bias 
& 85.32 & 89.84 & 90.41 & {0.29} & 32.9  & 297 & \textbf{79.43} 
\\ 
\lora
& 86.87  & 89.81 & 91.45 &  {1.00}  & 33.1  & 363 & \textbf{79.43}
\\
\vpt
& {{86.34}}  & {89.24} & {90.14} &  {0.25} & {33.7}  & {334}  & {76.35}
\\
\textsc{GateVPT}
& 86.31 & 89.14 & 91.11 & 3.14 & 34.9 & 395 & 61.34
\\
\textsc{E2VPT}
& \textbf{89.93} & 90.12 & 91.45 & 1.21 & 33.4 & 369 & 52.32
\\
\adapter
& 87.06  & 89.44 & 91.21 & {2.17} & 32.4  & 357  & 63.39
\\
\adapterformer
& {89.18}  & \textbf{90.69} & \textbf{92.08} &  {0.65} & 32.3  & 289  & 23.69
\\
{\textsc{SSF}}
& {87.32}  & {89.43} & {92.21} &  {0.48} 
& {34.7}  & {299}  & {\textbf{79.49}} 
\\
\midrule
\rowcolor{Gray}
\ours (Ours)
& 85.30  & 90.25 & 91.09 &  \textbf{{0.16}} & \textbf{31.6}  & \textbf{262}  & \textbf{79.43} 
\\
\bottomrule[1pt]
\end{tabular}
}
\end{center}
\vspace*{-7mm}
\end{table}
\setlength{\tabcolsep}{1.4pt}

\begin{table*}[h]
    \vspace*{-2mm}
\caption{\footnotesize{ 
Performance comparison in the few-shot setting on the VTAB-$1K$ benchmark. Other settings follow Tab.\,\ref{tab: main}.}}
\fontsize{6pt}{6pt}\selectfont
\newcolumntype{C}{>{\centering\arraybackslash}X}
\setlength{\tabcolsep}{2.2pt}
\setlength{\extrarowheight}{7pt}
\renewcommand{\arraystretch}{0.6}
\begin{tabularx}{\linewidth}{p{5pt}p{1.8cm}!{\color{lightgray}\vline}CCCCCCC!{\color{lightgray}\vline}CCCC!{\color{lightgray}\vline}CCCCCCCC!{\color{lightgray}\vline}C}
\toprule[1pt]
& \multicolumn{1}{c}{\underline{\textbf{Benchmark}}}
& \multicolumn{7}{c}{\underline{\textbf{VTAB-Natural}}} 
& \multicolumn{4}{c}{\underline{\textbf{VTAB-Specialized}}} 
& \multicolumn{8}{c}{\underline{\textbf{VTAB-Structured}}} \\
\rotatebox{90}{\raisebox{1.0pt}{\textbf{Architecture}}}
& 
& \rotatebox{90}{\raisebox{1.0pt} {Caltech101}} 
& \rotatebox{90}{\raisebox{1.0pt} {CIFAR-100}}
& \rotatebox{90}{\raisebox{1.0pt} {DTD}}
& \rotatebox{90}{\raisebox{1.0pt} {Flowers102}}
& \rotatebox{90}{\raisebox{1.0pt} {OxfordPets}}
& \rotatebox{90}{\raisebox{1.0pt} {Sun397}}
& \rotatebox{90}{\raisebox{1.0pt} {SVHN}}
& \rotatebox{90}{\raisebox{1.0pt} {Camelyon}}
& \rotatebox{90}{\raisebox{1.0pt} {EuroSAT}}
& \rotatebox{90}{\raisebox{1.0pt} {Resisc45}}
& \rotatebox{90}{\raisebox{1.0pt} {Retinopathy}}
& \rotatebox{90}{\raisebox{1.0pt} {Clevr-Count}}
& \rotatebox{90}{\raisebox{1.0pt} {Clevr-Dist}}
& \rotatebox{90}{\raisebox{1.0pt} {DMLab}}
& \rotatebox{90}{\raisebox{1.0pt} {dSpr-Loc}} 
& \rotatebox{90}{\raisebox{1.0pt} {dSpr-Ori}}
& \rotatebox{90}{\raisebox{1.0pt} {KITTI-Dist}}
& \rotatebox{90}{\raisebox{1.0pt} {sNORB-Azim}} 
& \rotatebox{90}{\raisebox{1.0pt} {sNORB-Elev}}
& \rotatebox{90}{\raisebox{1.0pt} Average} \\
\midrule
\multirow{6}{*}{\rotatebox{90}{\textbf{ResNet-101}}}
& {\fullcirc} {\ff} 
& 89.99 &  45.17 &  63.78 &  84.29 &  89.82 &  41.09 &  67.79 &  84.92 &  74.57 &  91.37 &  74.14 &  58.11 &  60.99 &  43.61 &  67.05 &  40.45 &  78.34 &  33.64 &  36.38 & 64.50 \\
\cmidrule{2-22}
& {\partcirc} {\lp}  
& 83.87 &  39.13 &  53.09 &  70.89 &  85.15 &  28.14 &  43.44 &  78.65 &  69.43 &  90.78 &  69.31 &  35.91 &  36.48 &  35.75 &  34.76 &  19.51 &  65.68 &  16.91 &  23.39 &  51.12\\
& {\partcirc} {\nt} 
& 85.61 &  35.78 &  47.71 &  56.64 &  78.10 &  10.10 &  \textbf{68.67} &  \textbf{83.16} &  61.10 &  90.50 &  \textbf{72.44} &  37.54 &  55.24 &  \textbf{40.04} &  \textbf{60.89} &  20.33 &  65.54 &  24.86 &  25.96 & 53.70 \\
& {\partcirc} {\vp} 
& 84.73 &  \textbf{43.01} &  57.55 &  76.91 &  87.03 &  28.75 &  55.47 &  75.15 &  70.27 &  89.26 &  69.08 &  36.70 &  54.24 &  34.48 &  42.41 &  20.32 &  63.71 &  17.93 &  26.93 & 54.42 \\
& {\partcirc} {\ours} 
& \textbf{87.49} &  39.80 &  \textbf{63.62} &  \textbf{81.44} &  \textbf{88.74} &  \textbf{34.83} &  65.92 &  78.91 &  \textbf{74.19} &  \textbf{91.44} &  71.18 &  \textbf{40.20} &  \textbf{55.26} &  38.95 &  54.68 &  \textbf{21.98} &  \textbf{72.86} &  \textbf{26.24} &  \textbf{28.77} &  \textbf{58.76}\\

\midrule
\multirow{5}{*}{\rotatebox{90}{\textbf{ViT-L/16}}}
& {\fullcirc} {\ff} 
& 93.34 &  76.03 &  75.74 &  99.88 &  93.72 &  59.06 &  68.70 & 86.70 &  82.84 &  93.54 &  82.22 & 55.42 &  60.33 &  48.23 &  83.62 &  52.77 &  78.06 &  30.40 &  29.95 & 71.08 \\
\cmidrule{2-22}
& {\partcirc} {\lp}   
& 89.37 &  62.98 &  70.02 &  93.42 &  91.22 &  53.68 &  45.28 &  	80.52 &  80.34 &  91.64 &  70.43 &  		38.15 &  35.26 &  40.74 &  21.84 &  29.42 &  62.54 &  14.59 &  23.09 &  57.60\\
& {\partcirc} {\nt} 
& 91.10 &  \textbf{65.20} &  72.36 &  98.64 &  91.38 &  55.14 &  47.21 &  	\textbf{82.50} &  82.34 &  \textbf{93.94} &  71.74 &  		42.83 &  44.59 &  \textbf{41.21} &  35.64 &  \textbf{32.08} &  63.43 &  16.52 &  24.12 &  60.68 \\
& {\partcirc} {\vp} 
& 90.06 &  63.16 &  71.59 &  95.35 &  91.20 &  54.45 &  46.26 &  81.82 &  81.45 &  92.25 &  71.03 &  		41.03 &  \textbf{45.49} &  39.94 &  32.52 &  30.29 &  62.68 &  15.59 &  23.13 &  59.96 \\
& {\partcirc} {\ours}
& \textbf{91.40} &  64.40 &  \textbf{72.61} &  \textbf{99.50} &  \textbf{91.46} &  \textbf{56.67} &  \textbf{49.43} &  	81.41 &  \textbf{82.76} &  93.14 &  \textbf{71.99} &  		\textbf{43.26} &  38.09 &  40.57 &  \textbf{42.44} &  31.83 &  \textbf{65.40} &  \textbf{18.29} &  \textbf{25.96} & 	\textbf{61.06} \\
\bottomrule[1pt]
\end{tabularx}
\label{tab: vtab_1k}
\end{table*}

\textbf{{Applying {\ours} to various model architectures.}} To ensure that our conclusions generalize well, we shift our focus from the vision source model to the vision-language model, specific to CLIP \citep{radford2021learning}, and the multi-scale transformer structure, \textit{i.e.,} Swin-Transformer \citep{liu2021swin}, which have both received increasing attention in the area of VP \citep{bahng2022visual}. Our experiments demonstrate that the proposed idea of {\ours} works well even on steering a pretrained CLIP model and Swin-Transformer without changing its parameters. In Fig.\,\ref{fig: layer_effect_clip} and Tab.\,\ref{tab: clip_result}, we demonstrate that our main conclusions about {\ours} still holds for these two architectures well on various datasets. Specifically, in \underline{Fig.\,\ref{fig: layer_effect_clip}}, we show that the layer effect of {\ours} still exists. As both CLIP and Swin-Transformer uses a ViT as its backbone, the observed layer effect mimics that of a ViT-Large/16 as observed before. Specifically, {\ours} prefers to be installed on shallow layers to deep ones in order to obtain the best performance. In \underline{Tab.\,\ref{tab: clip_result}}, we demonstrate that in various datasets, {\ours} can significantly outperform {\vp} by $1\% \sim 6\%$. These experiments demonstrate the applicability of {\ours} on various model types.

\textbf{Additional experiments.} We provide abundant additional experiment results in Appx.\,\ref{app: additional_results} for a comprehensive discussion on the AP design and comparison with baselines. We justified the layer effects more (dataset, model architecture) combinations in Fig.\,\ref{fig: layer_effect_complete} similar to Fig.\,\ref{fig: layer_effect}. We also studied various variants of AP, including AP with different prompt types in Tab.\,\ref{tab: more_prompt}, and AP installed in multiple layers in Tab.\,\ref{tab: ap_multiple_layers}. A detailed comparison between {\ours} and other PEFT methods in various experimental settings is also provided, including VPT \citep{jia2022visual} (Tab.\,\ref{tab: vpt_setting}, Tab.\,\ref{tab: vpt_ablation}, and Fig.\,\ref{fig: vpt_deep}), LoRA\,\citep{hu2021lora} (Tab.\,\ref{tab: lora_ablation}), and SST\,\citep{lian2022scaling} (Tab.\,\ref{tab: ssf}).

\textbf{Limitations and discussions.} 
We acknowledge a potential limitation of {\ours} in its implicit reliance on the size of the pretrained model as a factor for achieving superior accuracy.
For compact models like ResNet-18 and ViT-Tiny, while {\ours} enhances the performance of {\vp}, it does not outperform {\nt}. This observation suggests that {\ours} may primarily utilize downstream data to guide or ``direct'' the existing learned knowledge obtained during pretraining, rather than actively acquiring new knowledge. However, we believe that this limitation does not prevent {\ours} from future applications to larger foundational vision models.

\section{Conclusion}
In this paper, we delve into AP (activation prompt) as a means to enhance the conventional input-level VP. We unveil that extending VP to {\ours} yields improved empirical performance and establishes a connection with normalization tuning. Additionally, we investigate the layer preference of {\ours} on CNNs and ViTs both empirically and theoretically. Our experiments demonstrate the superiority of {\ours} over VP, highlighting its efficiency advantages, and showcasing comparable performance to the state-of-the-art PEFT methods.

\section{Acknowledgement}
This work was supported by National Science Foundation (NSF) \#2430223, Army Research Office (ARO) W911NF-25-1-0020, and the Rensselaer-IBM Future of Computing Research Collaboration (http://airc.rpi.edu). The work of Shuai Zhang was supported by National Science Foundation (NSF) \#2349879. We also thank all anonymous reviewers for their constructive comments.




\clearpage

\section*{Checklist}
\begin{enumerate}

  \item For all models and algorithms presented, check if you include:
  \begin{enumerate}
    \item A clear description of the mathematical setting, assumptions, algorithm, and/or model. [Yes] All the settings have been clearly described in Sec.\,\ref{sec: experiments} and Appx.\,\ref{app: exp_detail}.
    \item An analysis of the properties and complexity (time, space, sample size) of any algorithm. [Yes] The efficiency analysis can be found in Sec.\,\ref{sec: experiments}.
    \item (Optional) Anonymized source code, with specification of all dependencies, including external libraries. [Yes] Anonymized source code with dependencies is submitted to the supplementary material.
  \end{enumerate}

  \item For any theoretical claim, check if you include:
  \begin{enumerate}
    \item Statements of the full set of assumptions of all theoretical results. [Yes] All the assumptions are stated in Sec.\,\ref{sec: theory}.
    \item Complete proofs of all theoretical results. [Yes] The proofs are shown in Sec.\,\ref{app: theory}.
    \item Clear explanations of any assumptions. [Yes] All the assumptions are stated in Sec.\,\ref{sec: theory}.
  \end{enumerate}

  \item For all figures and tables that present empirical results, check if you include:
  \begin{enumerate}
    \item The code, data, and instructions needed to reproduce the main experimental results (either in the supplemental material or as a URL). [Yes]
    \item All the training details (e.g., data splits, hyperparameters, how they were chosen). [Yes]
    \item A clear definition of the specific measure or statistics and error bars (e.g., with respect to the random seed after running experiments multiple times). [Yes]
    \item A description of the computing infrastructure used. (e.g., type of GPUs, internal cluster, or cloud provider). [Yes]
  \end{enumerate}

  \item If you are using existing assets (e.g., code, data, models) or curating/releasing new assets, check if you include:
  \begin{enumerate}
    \item Citations of the creator If your work uses existing assets. [Yes]
    \item The license information of the assets, if applicable. [Yes]
    \item New assets either in the supplemental material or as a URL, if applicable. [Not Applicable]
    \item Information about consent from data providers/curators. [Not Applicable]
    \item Discussion of sensible content if applicable, e.g., personally identifiable information or offensive content. [Not Applicable]
  \end{enumerate}

  \item If you used crowdsourcing or conducted research with human subjects, check if you include:
  \begin{enumerate}
    \item The full text of instructions given to participants and screenshots. [Not Applicable]
    \item Descriptions of potential participant risks, with links to Institutional Review Board (IRB) approvals if applicable. [Not Applicable]
    \item The estimated hourly wage paid to participants and the total amount spent on participant compensation. [Not Applicable]
  \end{enumerate}

\end{enumerate}

\clearpage

\newpage
\onecolumn
\section*{\Large{Appendix}}
\setcounter{section}{0}
\setcounter{figure}{0}
\setcounter{table}{0}
\makeatletter 
\renewcommand{\thesection}{\Alph{section}}
\renewcommand{\theHsection}{\Alph{section}}
\renewcommand{\thefigure}{A\arabic{figure}} 
\renewcommand{\theHfigure}{A\arabic{figure}} 
\renewcommand{\thetable}{A\arabic{table}}
\renewcommand{\theHtable}{A\arabic{table}}
\makeatother

\renewcommand{\thetable}{A\arabic{table}}
\setcounter{mylemma}{0}
\renewcommand{\themylemma}{A\arabic{mylemma}}
\setcounter{equation}{0}
\renewcommand{\theequation}{A\arabic{equation}}

\appendix
\label{sec: appendix}

\section{Experiment Setting Details}
\label{app: exp_detail}

\paragraph{Datasets.}
We consider 29 downstream image classification tasks in the target domain across various domains. We show each dataset's attributes in Tab.\,\ref{tab: dataset}.

\begin{table}[htb]
\centering

\resizebox{0.7\linewidth}{!}{%
\begin{tabular}{l|c|c|c|c|c}
\toprule[1pt]
\midrule
\textbf{Dataset} & \textbf{Train Size} & \textbf{Test Size} & \textbf{Class Number} & \textbf{Batch Size} & \textbf{Reference }\\
\midrule
\multicolumn{6}{c}{\textbf{Full-Data Setting}} \\
\midrule
Flowers102
& 4093
& 2463
& 102
& 128
& \citep{nilsback2008automated}
\\
DTD
& 2820
& 1692
& 47
& 128
& \citep{cimpoi2014describing}
\\
UCF101
& 7639
& 3783
& 101
& 128
& \citep{soomro2012ucf101}
\\
Food101
& 50500
& 30300
& 101
& 128
& \citep{bossard2014food}
\\
SVHN
& 73257
& 26032
& 10
& 128
& \citep{netzer2011reading}
\\
GTSRB
& 39209
& 12630
& 43
& 128
& \citep{Houben-IJCNN-2013}
\\
EuroSAT
& 13500
& 8100
& 10
& 128
& \citep{helber2019eurosat}
\\
OxfordPets
& 2944
& 3669
& 37
& 128
& \citep{parkhi2012cats}
\\
StanfordCars
& 6509
& 8041
& 196
& 128
& \citep{KrauseStarkDengFei-Fei_3DRR2013}
\\
SUN397
& 15888
& 19850
& 397
& 128
& \citep{xiao2010sun}
\\
CIFAR10
& 50000
& 10000
& 10
& 128
& \citep{krizhevsky2009learning}
\\
CIFAR100
& 50000
& 10000
& 100
& 128
& \citep{krizhevsky2009learning}
\\
CUB-200-2011
& 5394
& 5794
& 200
& 128
& \citep{WahCUB_200_2011}
\\
NA-Birds
& 21536
& 24633
& 55
& 128
& \citep{van2015nabirds}
\\
StanfordDog
& 10800
& 8580
& 120
& 128
& \citep{Khosla_FGVC2011dogs}
\\
OxfordFlowers
& 1020
& 6149
& 102
& 128
& \citep{nilsback2008automated}
\\
Waterbirds
& 4795
& 5794
& 2
& 128
& \citep{sagawa2019distributionally} \\
Caltech101
& 4128
& 2465
& 102
& 128
& \citep{li2006one}
\\
Camelyon
& 262144
& 32768
& 2
& 128
& \citep{veeling2018rotation}
\\
\midrule
\multicolumn{6}{c}{\textbf{Few-Shot Setting (VTab-1k)}}\\
\midrule
CIFAR-100
& 1000
& 10000
& 100
& 128
& \citep{krizhevsky2009learning}
\\
Caltech101
& 1000
& 6084
& 102
& 128
& \citep{li2006one}
\\
DTD
& 1000
& 47
& 1880
& 128
& \citep{cimpoi2014describing}
\\
Flowers102
& 1000
& 6149
& 102
& 128
& \citep{nilsback2008automated}
\\
OxfordPets
& 1000
& 3669
& 37
& 128
& \citep{parkhi2012cats}
\\
SVHN
& 1000
& 26032
& 10
& 128
& \citep{netzer2011reading}
\\
Sun397
& 1000
& 21750
& 397
& 128
& \citep{xiao2010sun}
\\
Patch Camelyon
& 1000
& 32768
& 2
& 128
& \citep{veeling2018rotation}
\\
EuroSAT
& 1000
& 5400
& 10
& 128
& \citep{helber2019eurosat}
\\
Resisc45
& 1000
& 6300
& 45
& 128
& \citep{cheng2017remote}
\\
Retinopathy
& 1000
& 42670
& 5
& 128
& \citep{kaggle-diabetic-retinopathy}
\\
Clevr/count
& 1000
& 15000
& 8
& 128
& \citep{johnson2017clevr}
\\
Clevr/distance
& 1000
& 15000
& 6
& 128
& \citep{johnson2017clevr}
\\
DMLab
& 1000
& 22735
& 6
& 128
& \citep{beattie2016deepmind}
\\
KITTI/distance 
& 1000
& 711
& 4
& 128
& \citep{Geiger2013IJRR}
\\
dSprites/location
& 1000
& 73728
& 16
& 128
& \citep{dsprites17}
\\
dSprites/orientation
& 1000
& 73728
& 16
& 128
& \citep{dsprites17}
\\
SmallNORB/azimuth
& 1000
& 12150
& 18
& 128
& \citep{lecun2004learning}
\\
SmallNORB/elevation
& 1000
& 12150
& 9
& 128
& \citep{lecun2004learning}
\\
\midrule
\bottomrule[1pt]
\end{tabular}%
}
\caption{\footnotesize{Dataset attributes and training configs through 29 target image-classification datasets.}}
\label{tab: dataset}
\end{table}

\paragraph{Implementation details.} As we stated in the main manuscript, we, by default, install {\ours} to the input of the thrid-to-last ResNet block and the third Transformer block in ViT-Large/16. For LoRA\,\citep{hu2021lora}, we use the rank $r=10$ by default. {For VPT\,\citep{jia2022visual}, we use a prompt length of $10$.} We train all the methods for 1000 epochs using an Adam optimizer. For {\ours}, we adopt a learning rate of $0.001$ for ResNet family and $0.01$ for ViT family without weight decay. For baselines, we adopt the learning rate suggested in the papers or official code repositories. In order to align with the settings of the most parameter efficient fine-tuning methods, for all the prompting-based methods we also tune the classification head as {\lp} throughout this work.

\newpage
\section{Additional Experiment Results}
\label{app: additional_results}

\begin{figure}[htb]
    \centering
    \begin{tabular}{cc}
       \includegraphics[width=.3\linewidth]{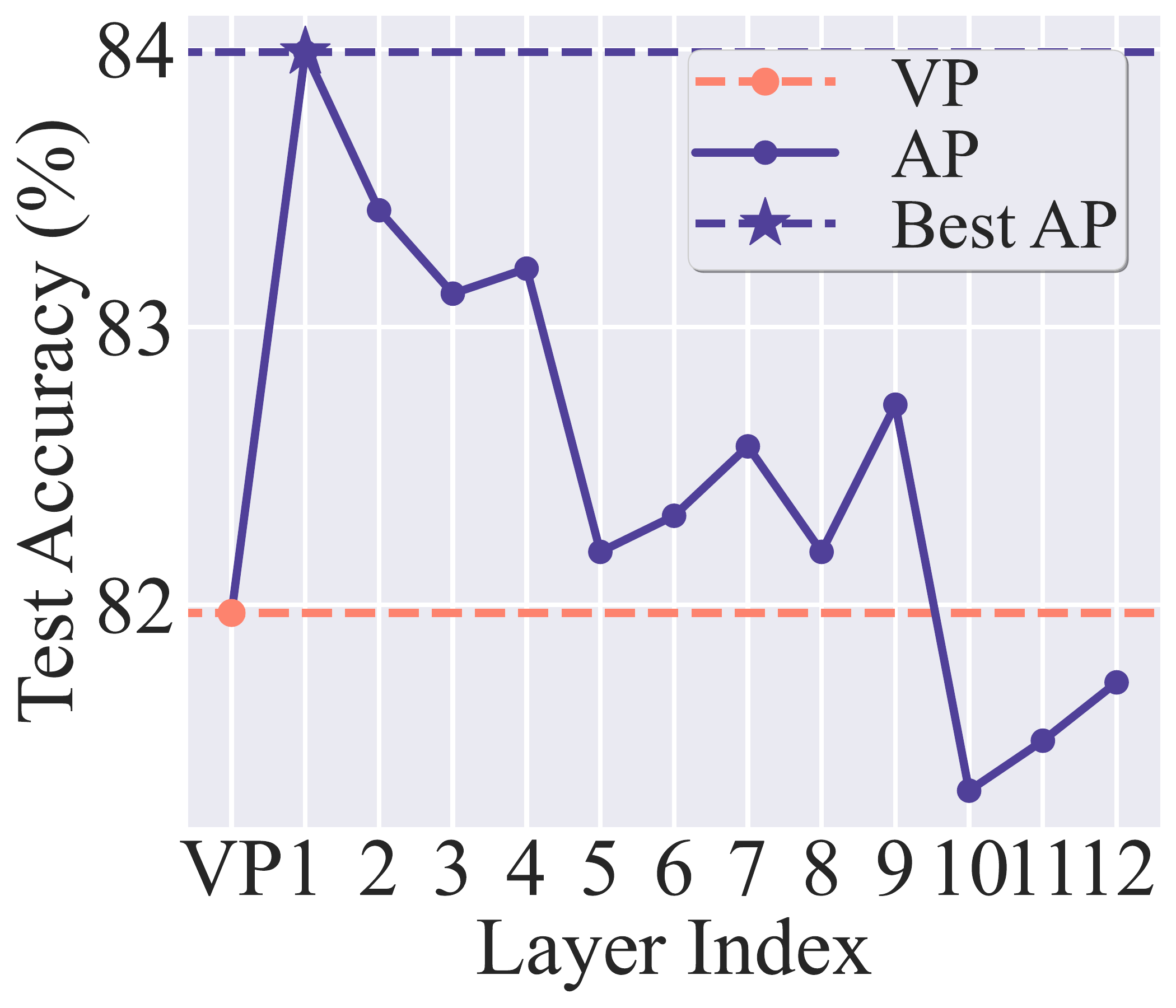}  
       & \includegraphics[width=.3\linewidth]{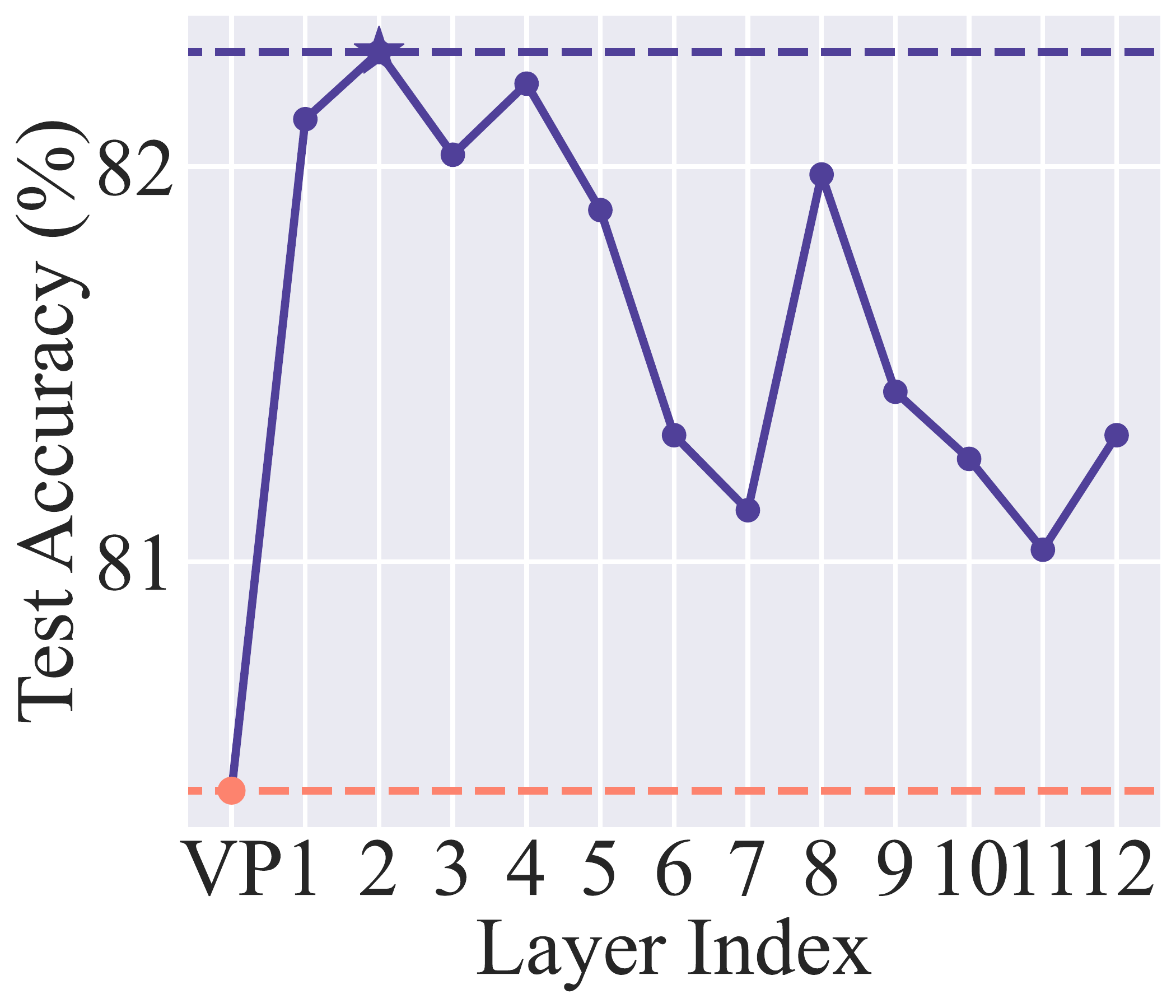}  
    \end{tabular}
    \caption{\footnotesize The layer effect of {\ours} applied to a (Left) CLIP model and {(Right) Swin-Transformer} on the OxfordPets dataset. }
    \label{fig: layer_effect_clip}
\end{figure}

\textbf{{Applying {\ours} to various model architectures.}} To ensure that our conclusions generalize well, we shift our focus from the vision source model to the vision-language model, specific to CLIP \citep{radford2021learning}, and the multi-scale transformer structure, \textit{i.e.,} Swin-Transformer \citep{liu2021swin}, which have both received increasing attention in the area of VP \citep{bahng2022visual}. Our experiments demonstrate that the proposed idea of {\ours} works well even on steering a pretrained CLIP model and Swin-Transformer without changing its parameters. In Fig.\,\ref{fig: layer_effect_clip} and Tab.\,\ref{tab: clip_result}, we demonstrate that our main conclusions about {\ours} still holds for these two architectures well on various datasets. Specifically, in \underline{Fig.\,\ref{fig: layer_effect_clip}}, we show that the layer effect of {\ours} still exists. As both CLIP and Swin-Transformer uses a ViT as its backbone, the observed layer effect mimics that of a ViT-Large/16 as observed before. Specifically, {\ours} prefers to be installed on shallow layers to deep ones in order to obtain the best performance. In \underline{Tab.\,\ref{tab: clip_result}}, we demonstrate that in various datasets, {\ours} can significantly outperform {\vp} by $1\% \sim 6\%$. These experiments demonstrate the applicability of {\ours} on various model types.

\begin{table}[htb]
\centering
\caption{\footnotesize{Performance comparison of {\vp} and the proposed {\ours} on CLIP and {Swin-Transformer} model with different datasets. CLIP with ViT-B/32 and Swin-B with 12 Swin-Transformer blocks pretrained on ImageNet are tested. Other settings follows Tab.\,\ref{tab: main}.}}
\label{tab: clip_result}
\resizebox{0.6\linewidth}{!}{%
\begin{tabular}{c|ccccccc}
\toprule[1pt]
\midrule
Dataset & OxfordPets & DTD & EuroSAT & Flowers102 & UCF101 & Food101 & Waterbirds \\
\midrule
\multicolumn{8}{c}{CLIP}
\\
\midrule
{\vp} & 81.97 & 64.43 & 95.54 & 83.74 & 70.42 & 79.61 & 72.42 \\
{\ours} (Ours) & 83.82 & 69.42 & 96.43 & 85.52 & 76.42 & 82.43 & 79.32 \\
\midrule
\multicolumn{8}{c}{{Swin-Transformer}} \\
\midrule
{\vp} 
& 80.42 & 65.39 & 97.23 & 84.48 & 74.41 & 75.72 & 75.22 \\
{\ours} (Ours) 
& 82.29 & 69.13 & 96.45 & 84.98 & 75.92 & 81.38 & 78.99 \\
\midrule
\bottomrule[1pt]
\end{tabular}}
\end{table}

\textbf{Layer effect study on more datasets.} In Fig.\,\ref{fig: layer_effect_complete}, we demonstrate that the layer effects of {\ours} demonstrated in Sec.\,\ref{sec: ap_layer} is general and apply to multiple datasets.

\begin{figure}[htb]
\centering
\centerline{
\begin{tabular}{cccc}
\hspace*{0mm}\includegraphics[width=.25\textwidth,height=!]{figures/layer_effect/ap_layer_oxfordpets_rn50.pdf} 
&\hspace*{-2mm}\includegraphics[width=.25\textwidth,height=!]{figures/layer_effect/ap_layer_oxfordpets_rn101.pdf}
&\hspace*{-2mm}\includegraphics[width=.25\textwidth,height=!]{figures/layer_effect/ap_layer_oxfordpets_vit_b.pdf}
&\hspace*{-2mm}\includegraphics[width=.25\textwidth,height=!]{figures/layer_effect/ap_layer_oxfordpets_vit_l.pdf}\\
\includegraphics[width=.25\textwidth,height=!]{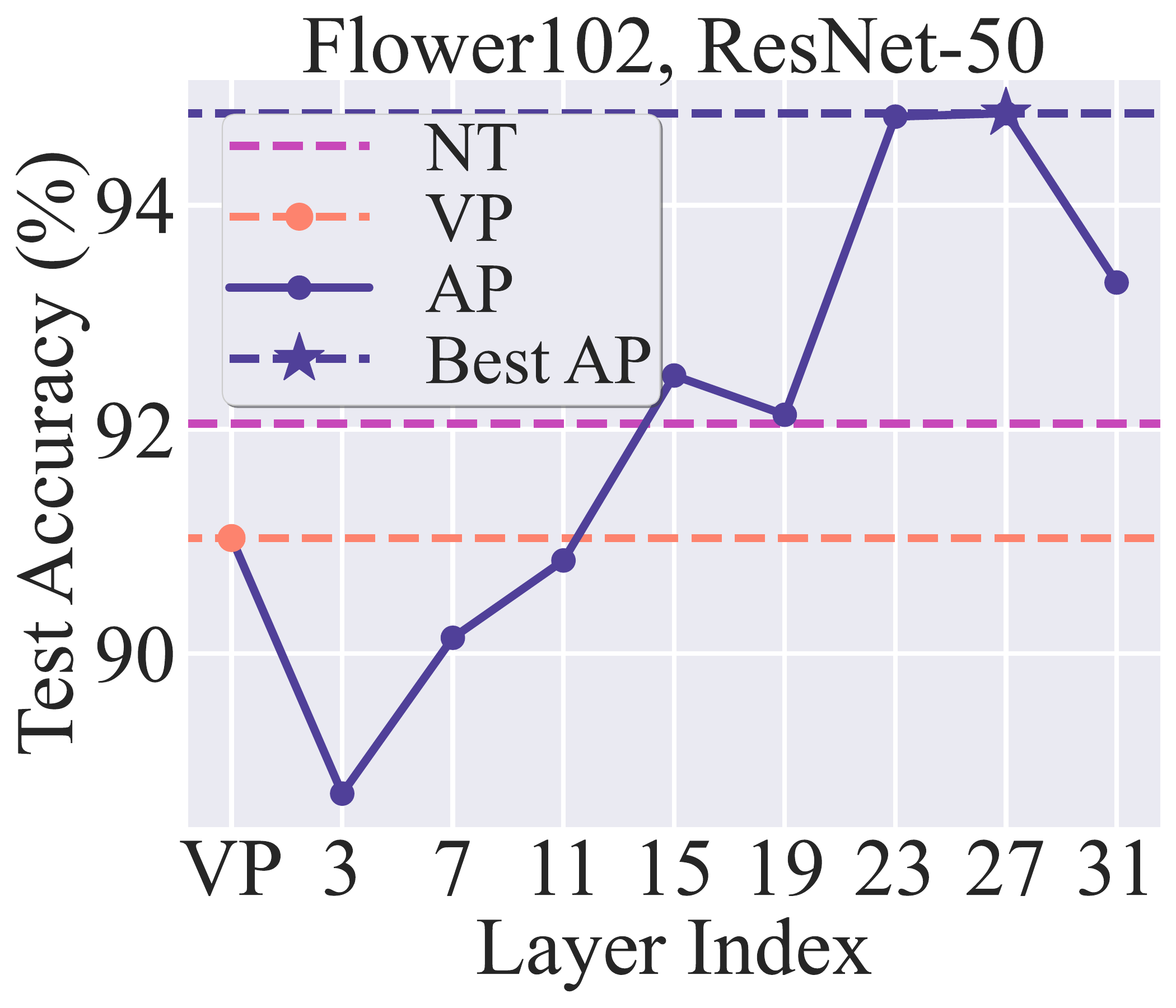} 
&\hspace*{-2mm}\includegraphics[width=.25\textwidth,height=!]{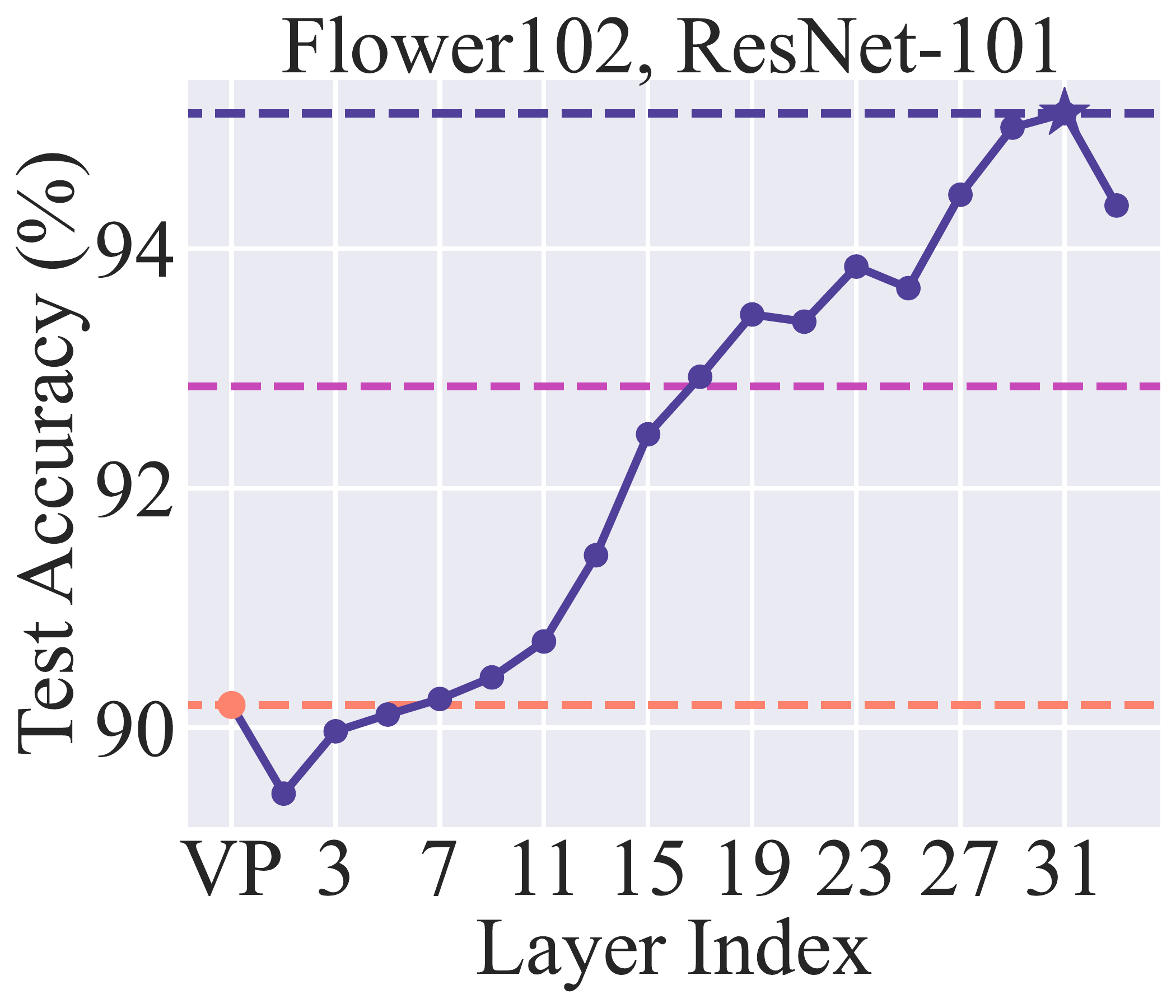}
&\hspace*{-2mm}\includegraphics[width=.25\textwidth,height=!]{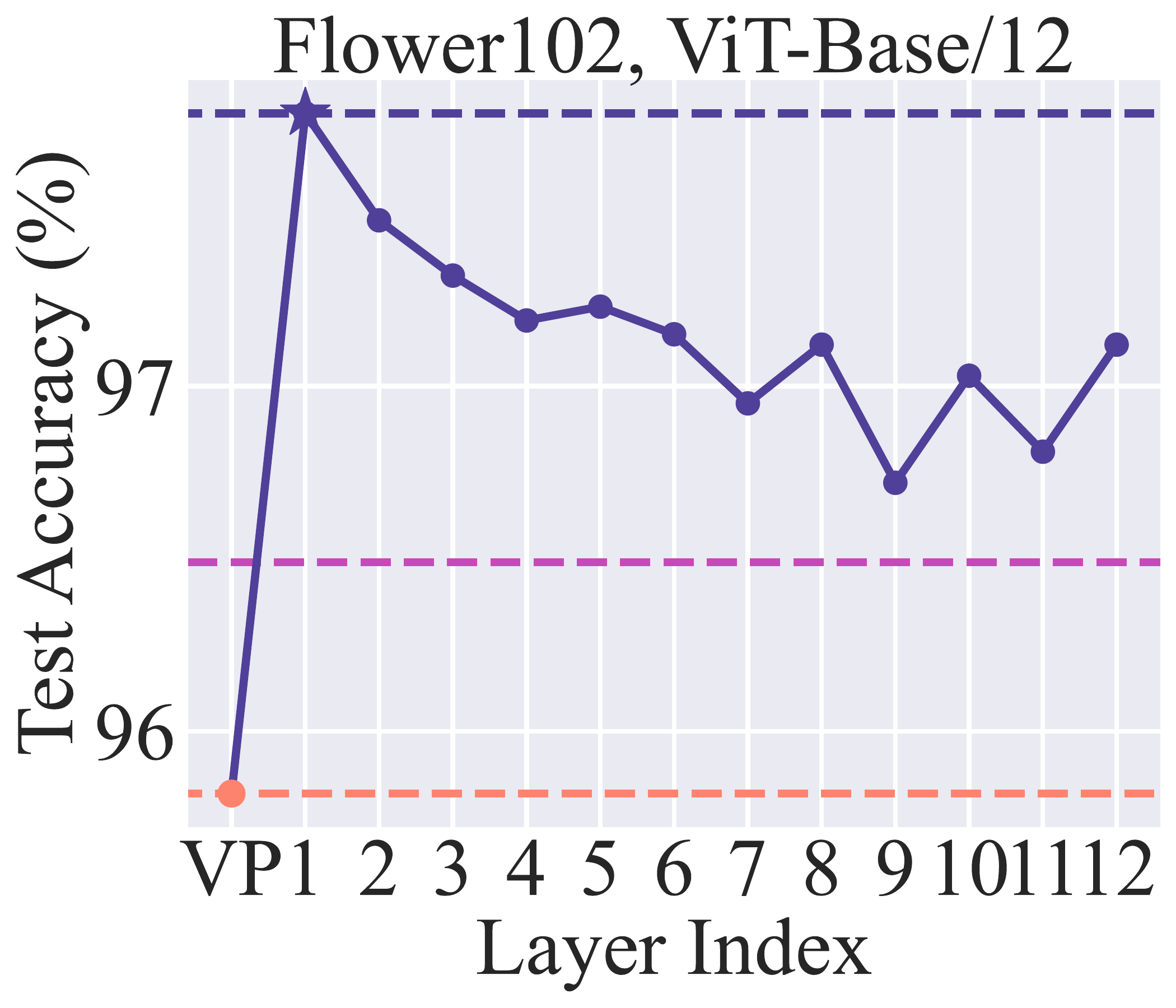}
&\hspace*{-2mm}\includegraphics[width=.25\textwidth,height=!]{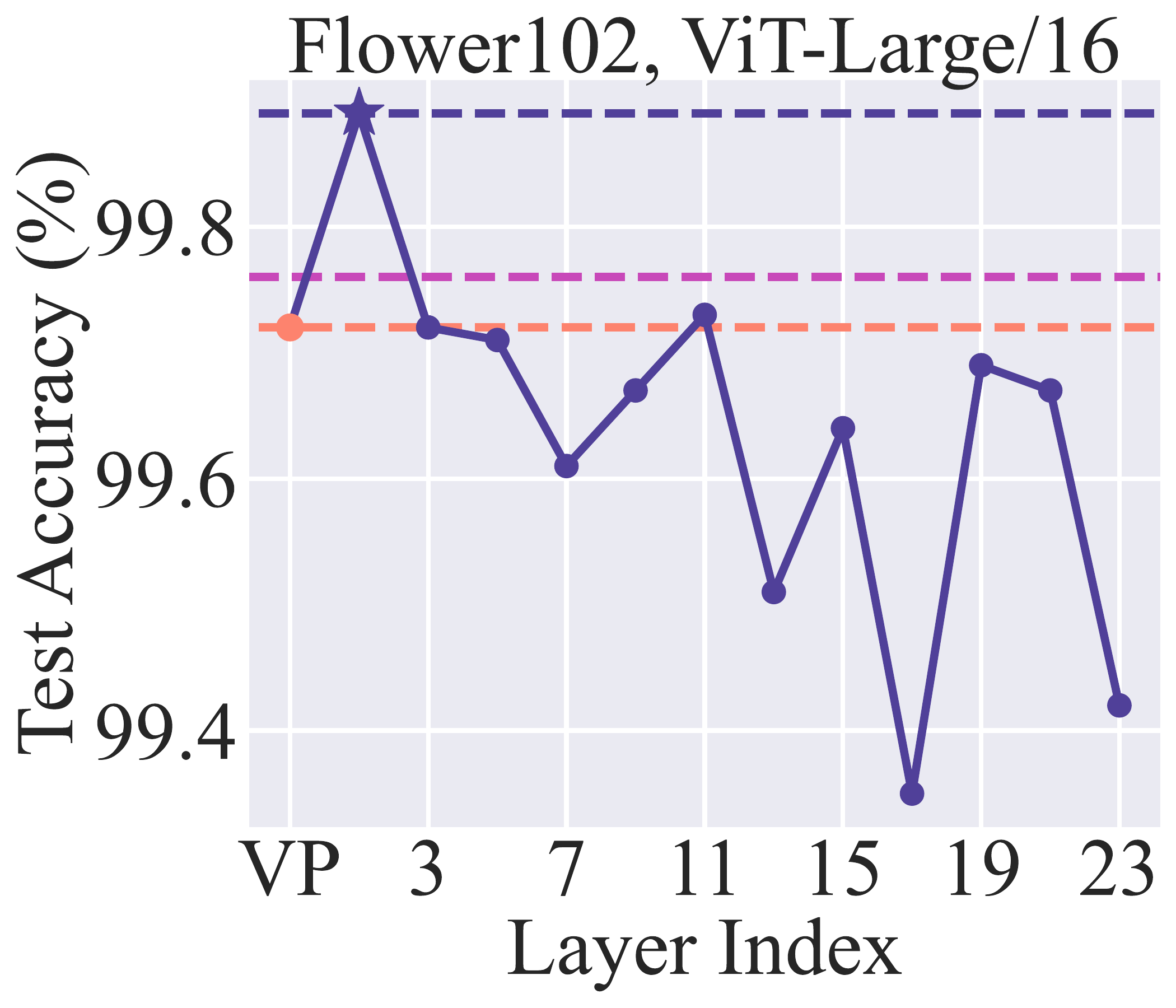}\\
\includegraphics[width=.25\textwidth,height=!]{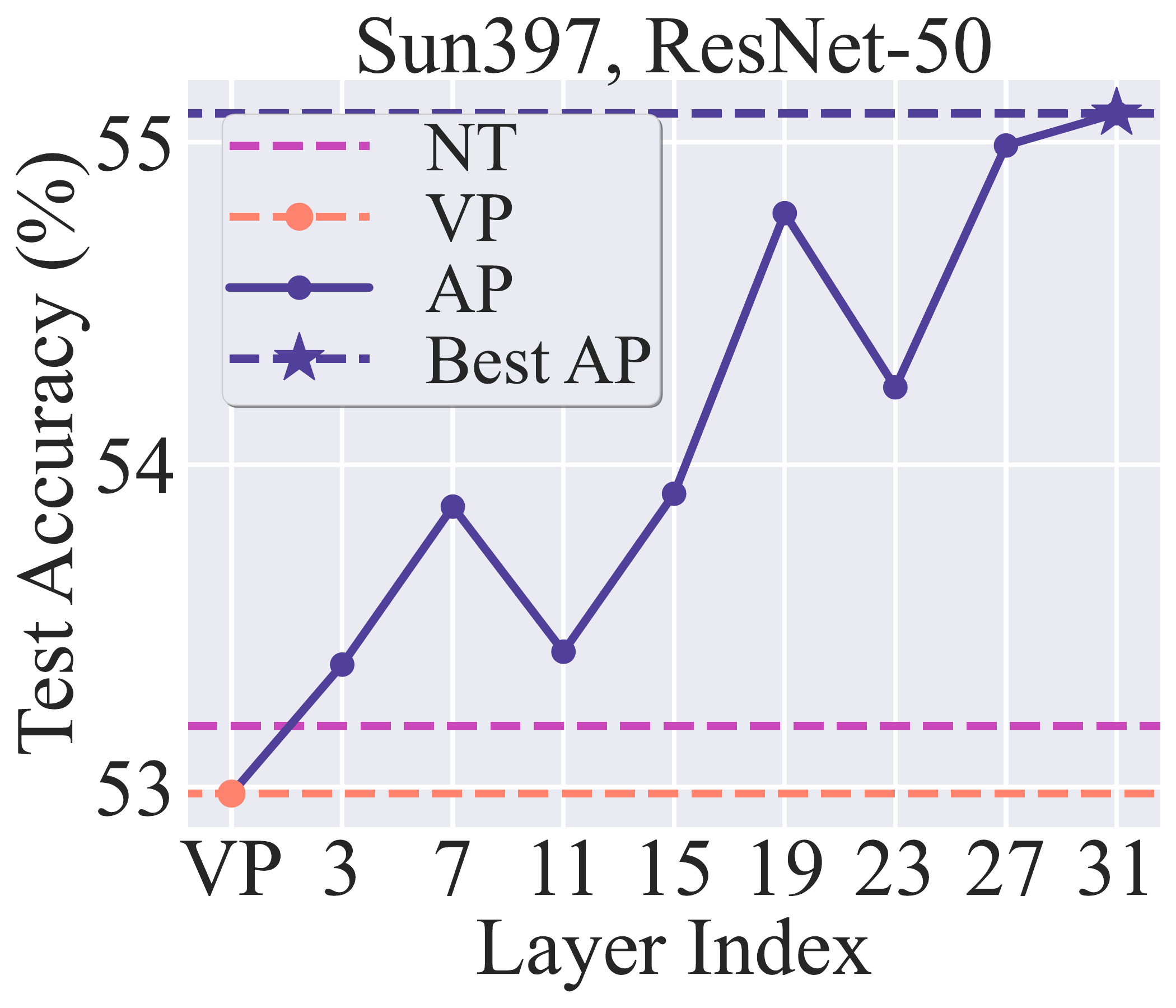} 
&\hspace*{-2mm}\includegraphics[width=.25\textwidth,height=!]{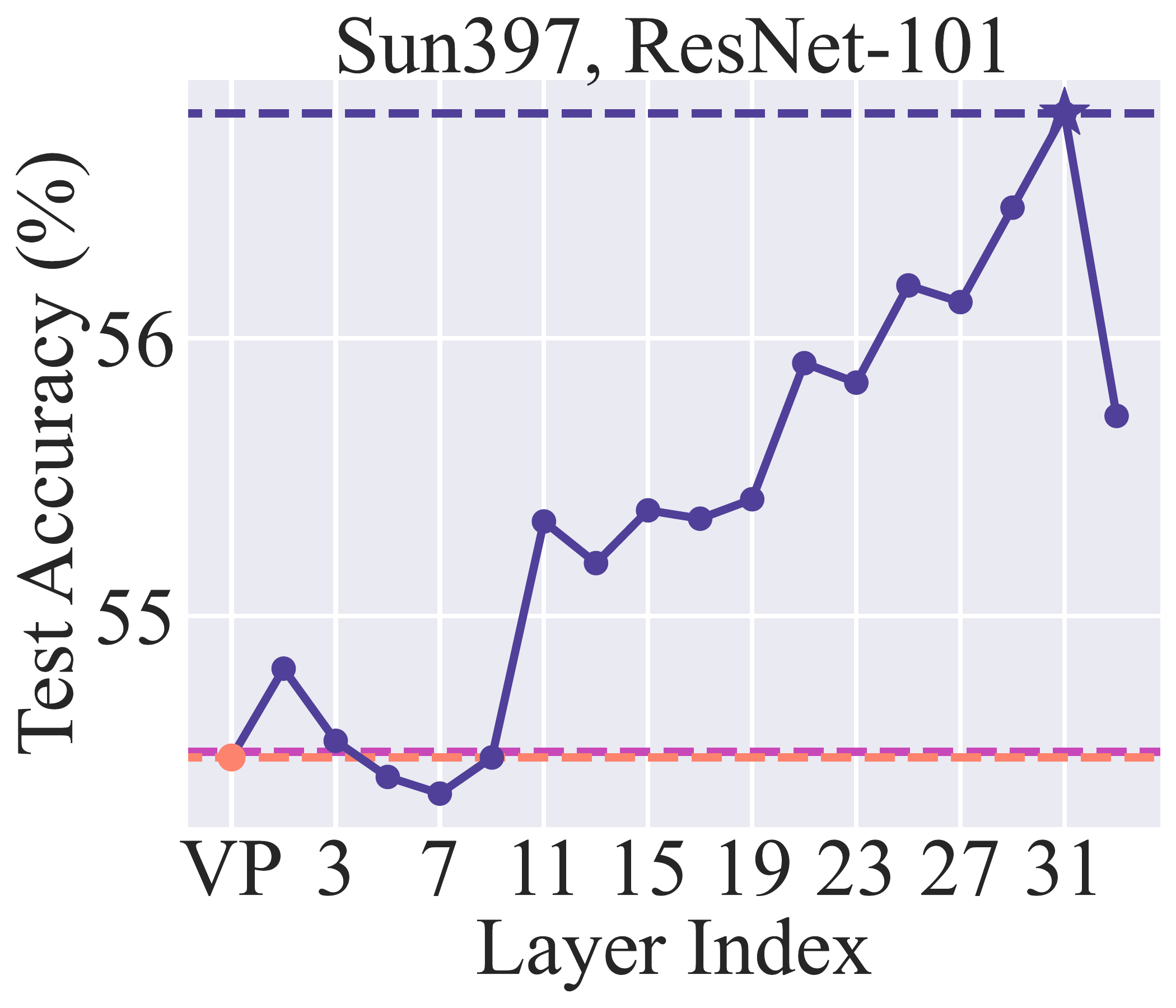}
&\hspace*{-2mm}\includegraphics[width=.25\textwidth,height=!]{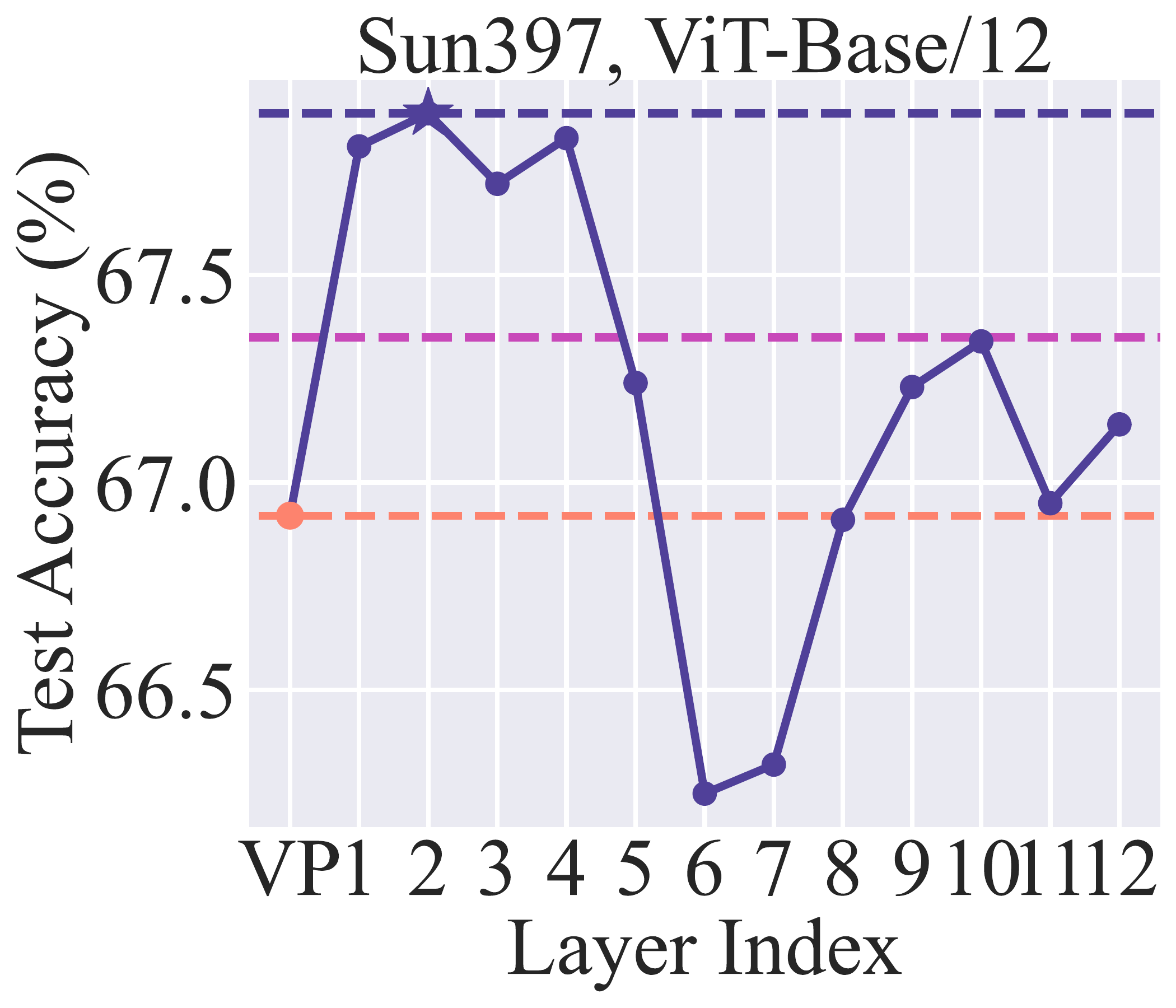}
&\hspace*{-2mm}\includegraphics[width=.25\textwidth,height=!]{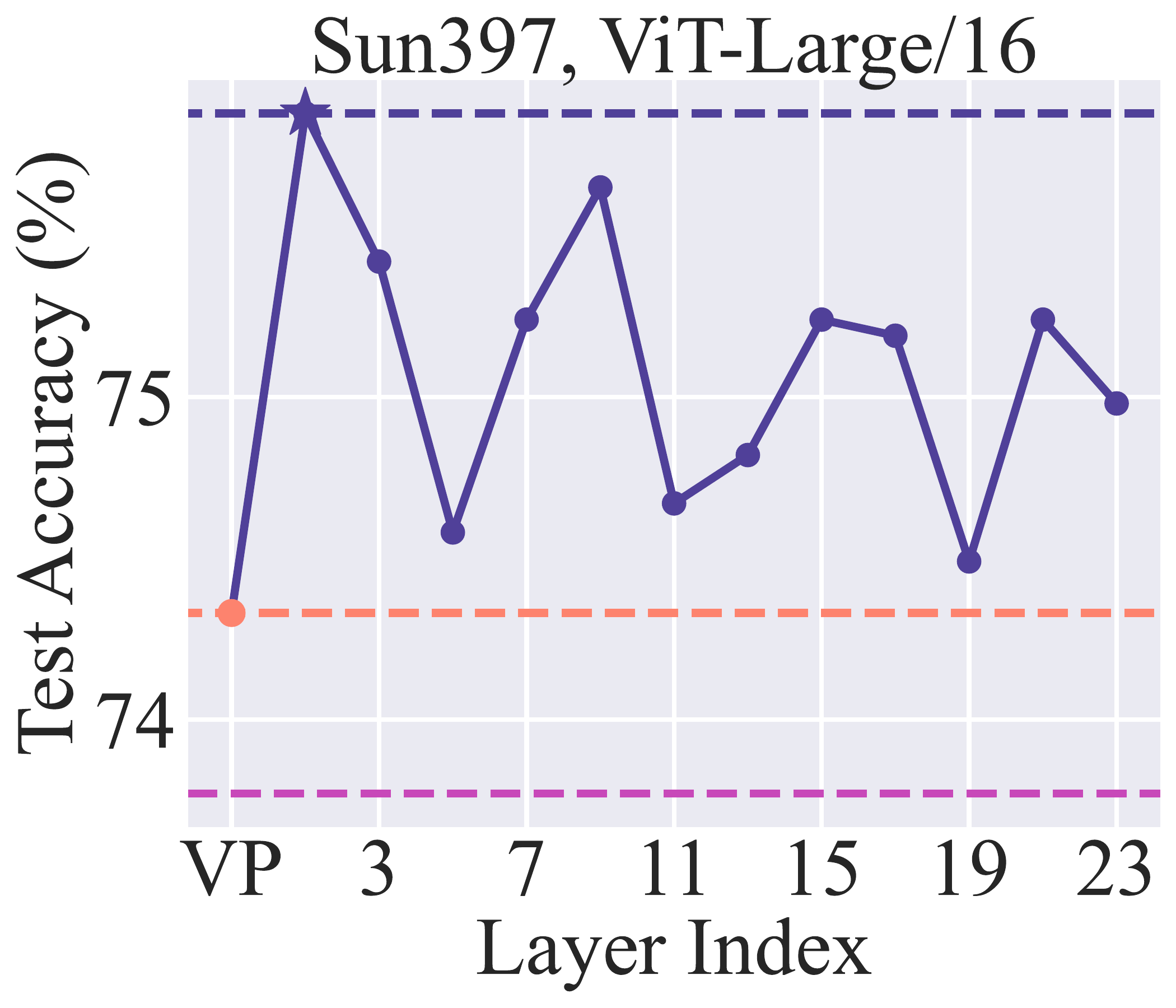}\\
\includegraphics[width=.25\textwidth,height=!]{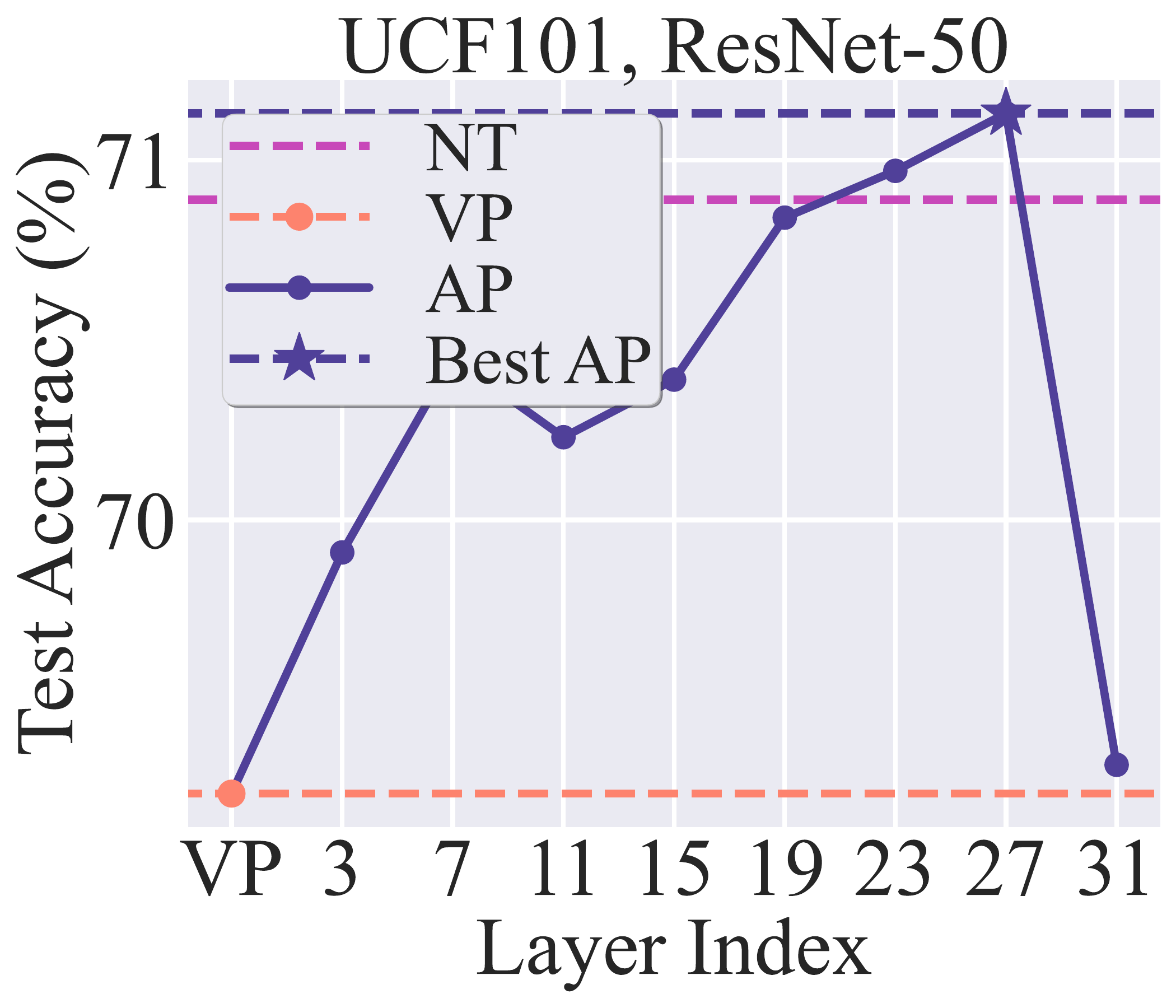} 
&\hspace*{-2mm}\includegraphics[width=.25\textwidth,height=!]{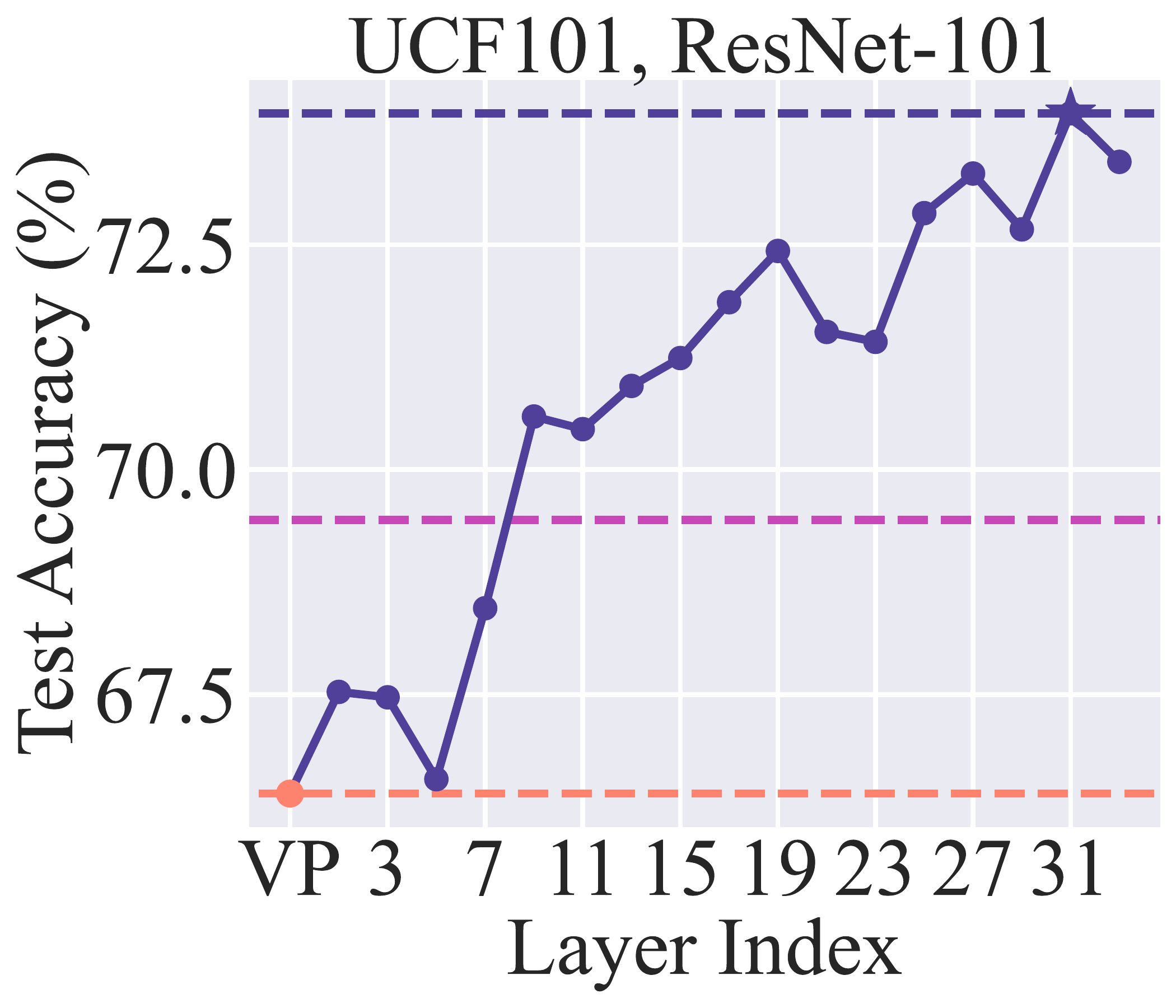}
&\hspace*{-2mm}\includegraphics[width=.25\textwidth,height=!]{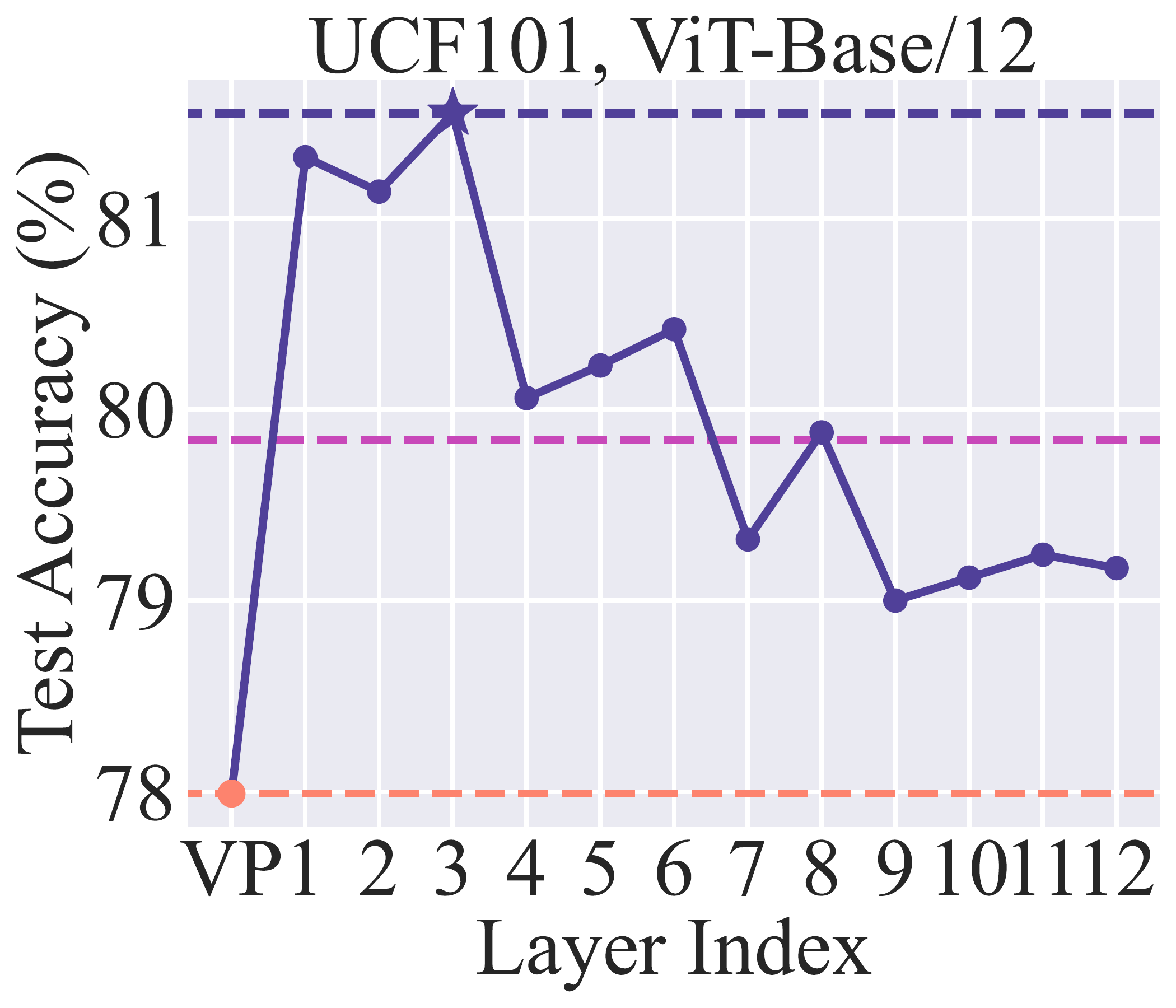}
&\hspace*{-2mm}\includegraphics[width=.25\textwidth,height=!]{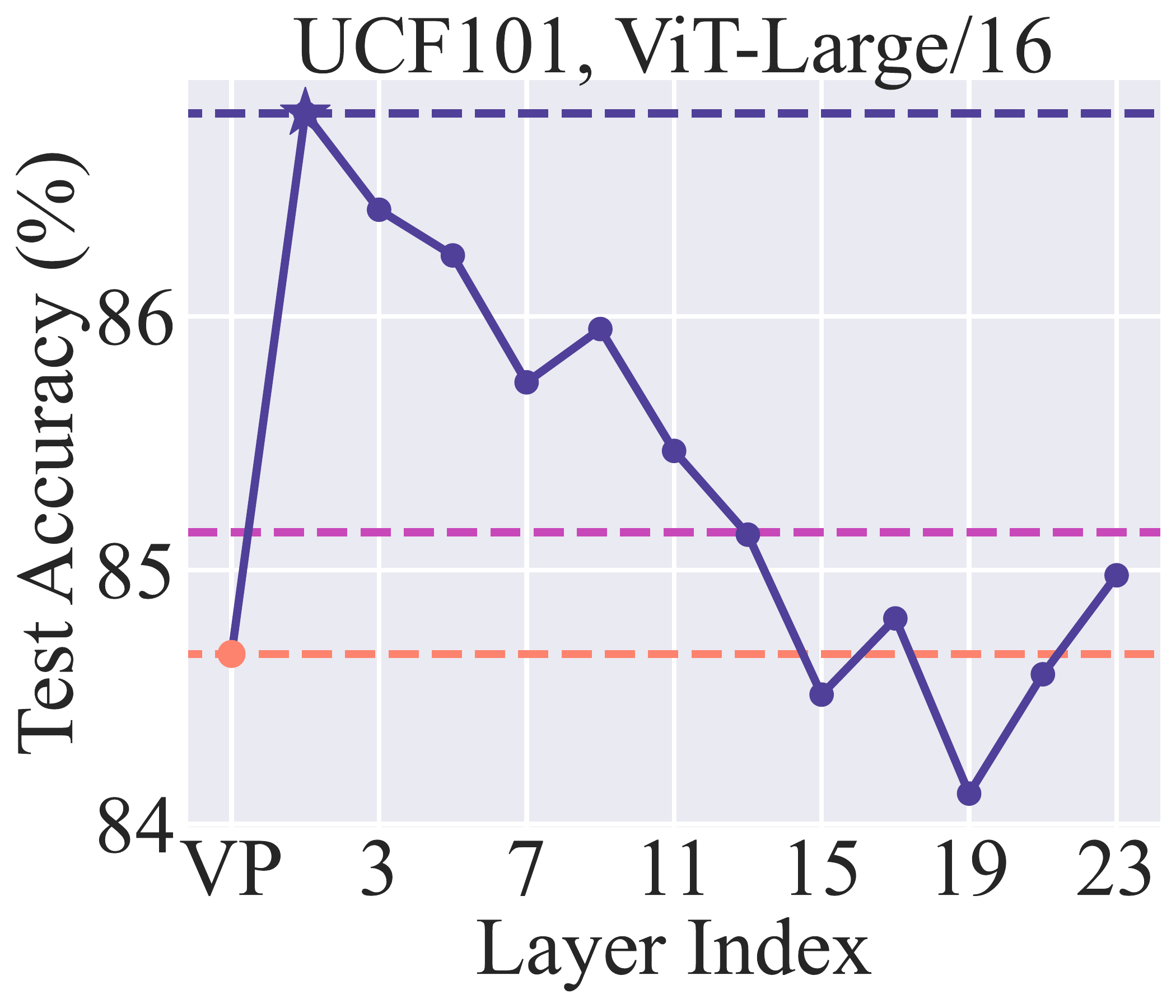}\\
\includegraphics[width=.25\textwidth,height=!]{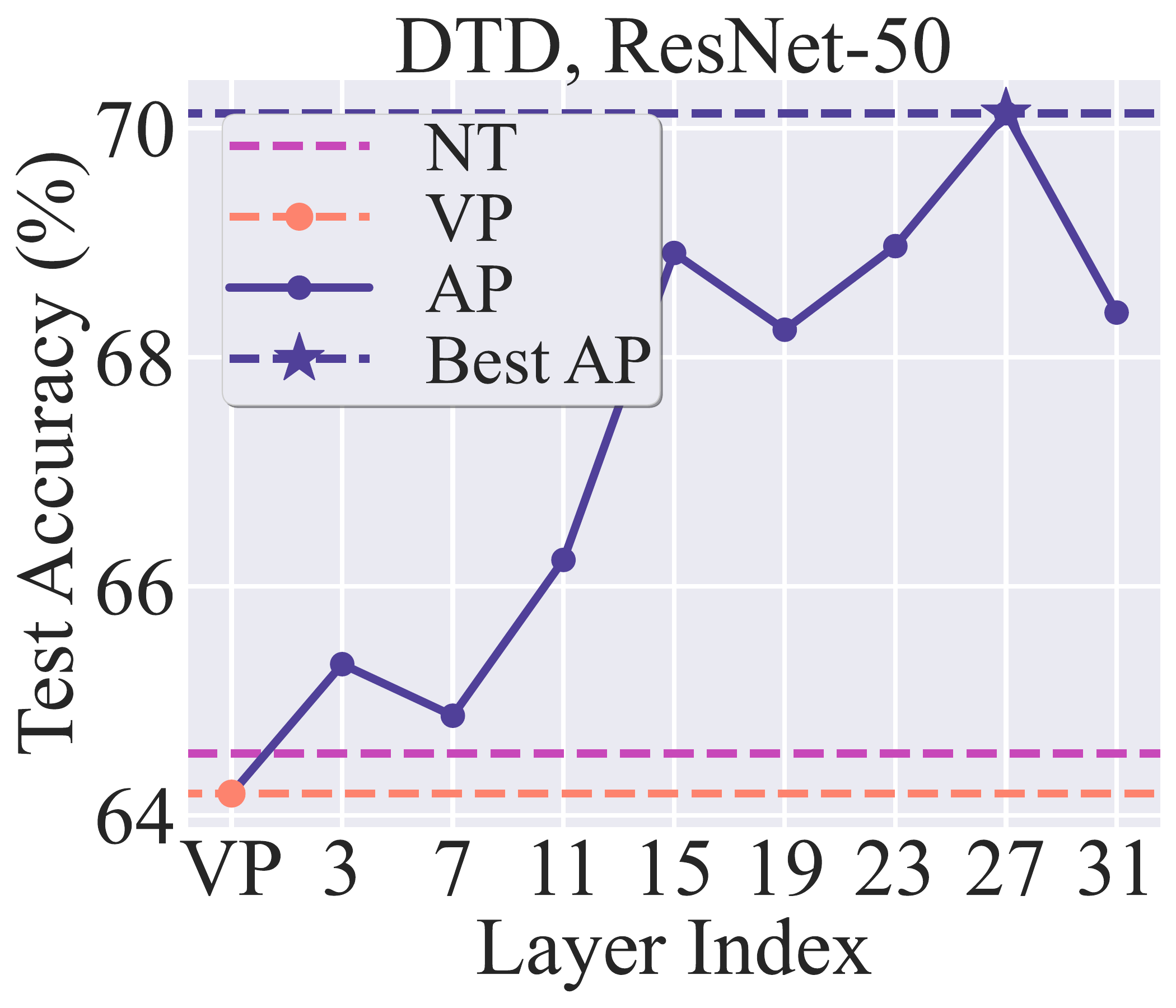} 
&\hspace*{-2mm}\includegraphics[width=.25\textwidth,height=!]{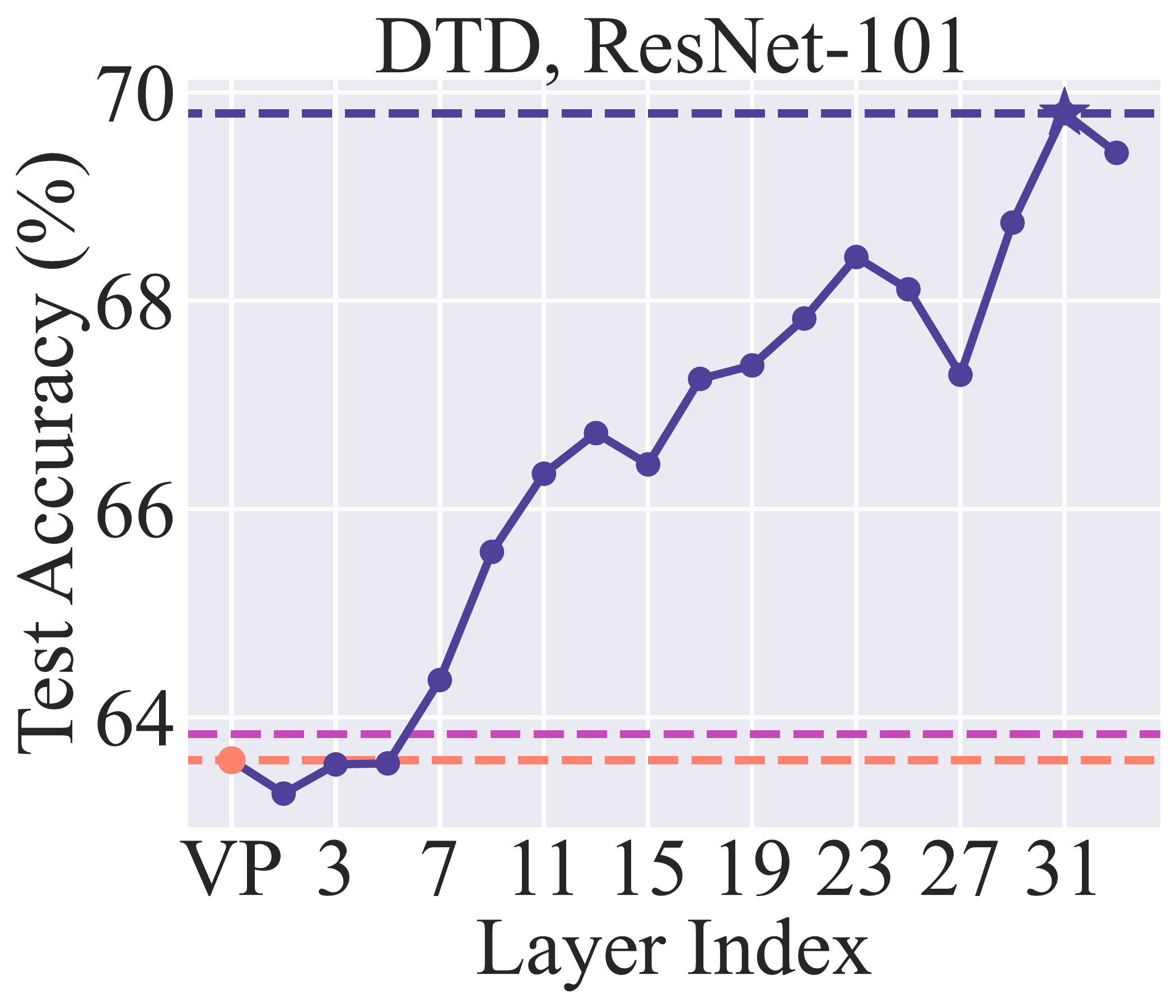}
&\hspace*{-2mm}\includegraphics[width=.25\textwidth,height=!]{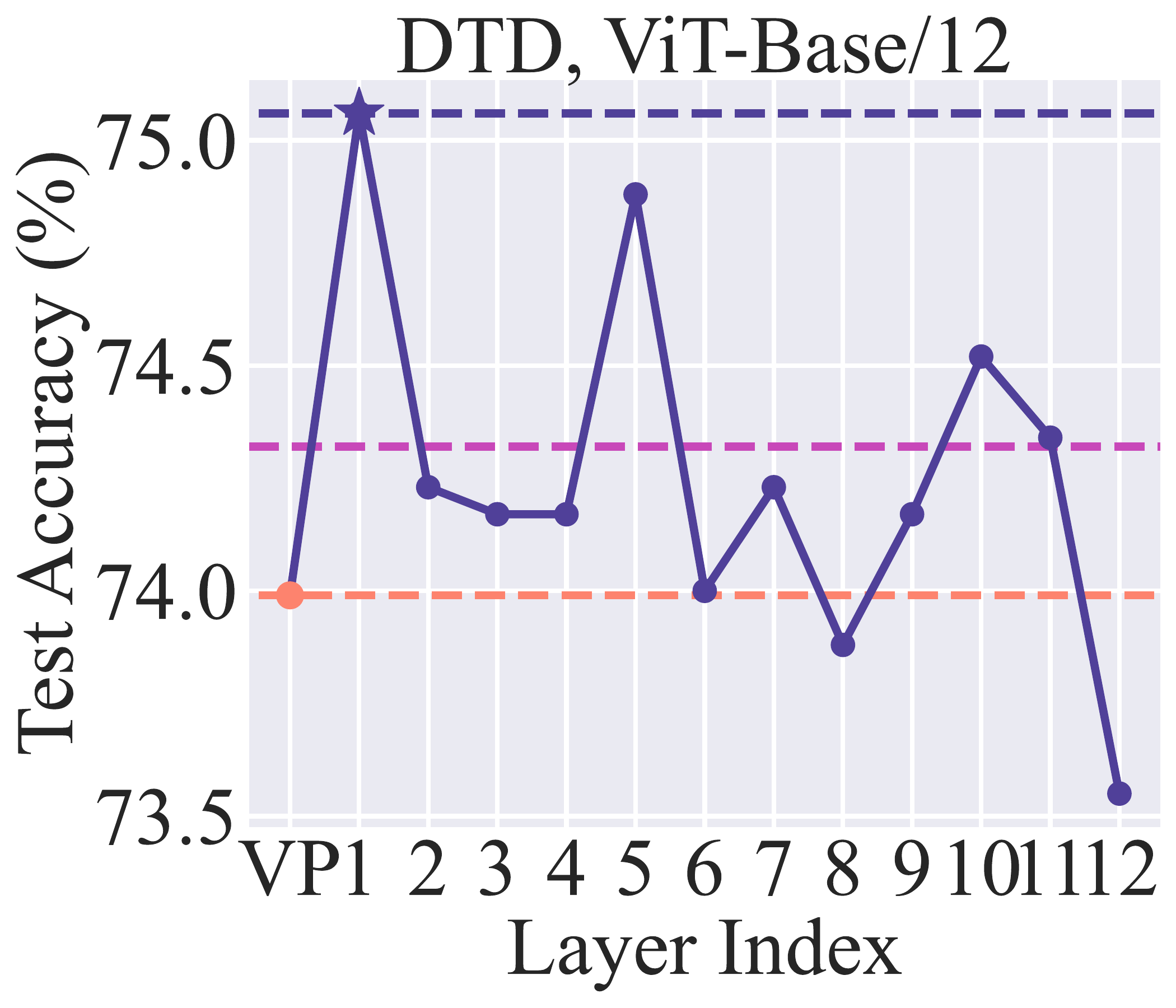}
&\hspace*{-2mm}\includegraphics[width=.25\textwidth,height=!]{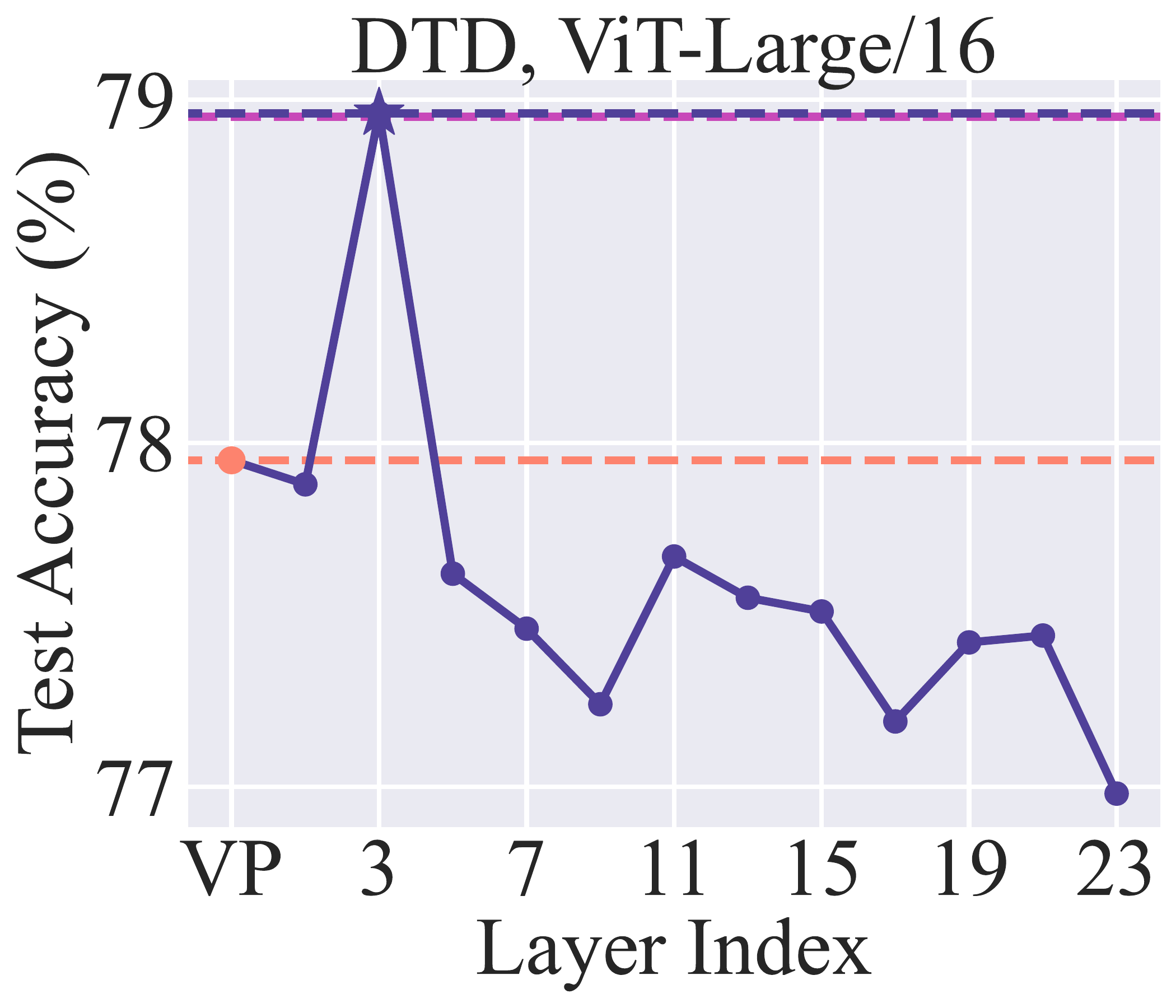}\\
\end{tabular}}
\caption{\footnotesize{Layer preference of {\ours} with different model architectures on different datasets. CNNs and ViTs exhibit opposite layer preferences.
}}
\label{fig: layer_effect_complete}
\end{figure}

\paragraph{{Performance of AP in the original experiment setting of VPT.}}
{We conduct an ablation study to strictly follow the experiment settings of VPT, with these results included in Tab. \ref{tab: vpt_setting}. The performance of VPT is directly sourced from Tab.\,\ref{tab: main} of \citep{jia2022visual}. As we can see, the performance as well as efficiency of {\ours} positions itself between VPT-Shallow and VPT-Deep, with an average of $3\%$ performance gain over VPT-Shallow and an average of $3.5\%$ drop compared to VPT-Deep. Regarding these results, we would like to mention that the results of VPT reported in Table 1 of \citep{jia2022visual} are selected based on its best prompt length per dataset, while {\ours} sticks to the same hyper-parameters across all the datasets.}

\begin{table}[ht]
\caption{Performance comparison of AP with other methods in the setting of VPT \citep{jia2022visual}. Specifically, ViT-B/16 pretrained on supervised ImageNet-21k is adopted as the pretrained model. The numbers except AP are directly sourced from VPT \citep{jia2022visual}.}
\label{tab: vpt_setting}
\centering
\resizebox{0.6\linewidth}{!}{%
\begin{tabular}{c|c|cccc}
\toprule[1pt]
\midrule
\multirow{2}{*}{\begin{tabular}[c]{@{}c@{}}ViT-B/16 \\ (85.8M)\end{tabular}} & \multirow{2}{*}{\begin{tabular}[c]{@{}c@{}}Total\\ Params\end{tabular}} & \multirow{2}{*}{FGCV} & \multicolumn{3}{c}{VTAB-1k} \\
 &  &  & Natural & Specialized & Structured \\
 \midrule
\textsc{Full-Finetune} & 24.02$\times$ & 88.54 & 75.88 & 83.36 & 47.64 \\
\textsc{Linear-Probe} & 1.02$\times$ & 79.32 & 68.93 & 77.16 & 26.84 \\
\midrule
\textsc{VPT-Shallow} & 1.04$\times$ & 84.62 & 76.81 & 74.66 & 46.98 \\
\textsc{VPT-Deep} & 1.18$\times$ & 89.11 & 78.48 & 82.43 & 54.98 \\
\midrule \rowcolor{Gray}
{\ours} (Ours) & 1.11$\times$ & 87.33 & 76.59 & 79.32 & 49.98 \\
\midrule
\bottomrule[1pt]
\end{tabular}%
*}
\end{table}

\paragraph{{Ablation study on additional prompt types in AP.}}
{We conduct additional experiments, with the findings presented in Tab. \ref{tab: more_prompt}. We observed that the originally proposed {\ours} outperforms its new prompt variants studied in Tab. \ref{tab: more_prompt} (AP-Product and AP-Concate). We speculate that the advantage of the originally proposed {\ours} may stem from its intrinsic connection to {\nt}, as discussed in the concluding part of Sec. \ref{sec: method}. }

\setlength{\tabcolsep}{4pt}
\begin{table}[ht]
\caption{\footnotesize{Ablation study on {\ours} with more prompt types. Specifically, instead of using additive prompt in the intermediate layer, \textsc{AP-Product} uses feature-wise product and \textsc{AP-Concate} adopts concatenating prompt.
}}
\label{tab: more_prompt}
\begin{center}
\resizebox{.55\linewidth}{!}{
\begin{tabular}{
l
ccc  !{\color{lightgray}\vrule}
c ccc cc 
}
\toprule[1pt]

& \multicolumn{3}{c}{\textbf{Accuracy}}
& \multicolumn{4}{c}{\textbf{Efficiency}}

\\
\cmidrule{2-8}
&\multicolumn{3}{c!{\color{lightgray}\vrule}}{\textbf{Full-Data}}
& 
&\multicolumn{3}{c}{\textbf{Train-Time Efficiency}}
\\
&{\small{FGVC}} &{\small{VTAB}} &{\small{Others}}
&{Param. \#} &{\small{Memory}} &{\small{Time}} &{\small{Throughput}} 
\\
\midrule
Number of tasks
&5 &9 &5
& -& - & - & -
\\
\midrule 
{\ff} 
& 91.43 & 91.97 & 93.91
& 304.33 
&  41.5  &  520 &  79.58
\\
{\lp} 
& 82.23  & 78.90 & 87.81
& 0.01 
&  9.7 & 121  & 79.64
\\

\midrule
\bias 
& 85.32 & 89.84 & 90.41 &  0.29 
& 32.9  & 297 & \textbf{79.48} 
\\ 
\lora
& 86.87  & 89.81 & 91.45 &  1.00
& 33.1  & 363 & \textbf{79.43}
\\
\vpt
& 86.05  & 89.97 & 90.64 &  1.24 
& 38.6  & 397  & 72.84
\\ 
\adapter
& 87.06  & 89.44 & 91.21 &  2.07 
& 32.4  & 357  & 63.39
\\
\adapterformer
& \textbf{89.18}  & \textbf{90.69} & \textbf{92.08} &  0.65 
& 32.3  & 289  & 23.69
\\
\midrule
{\textsc{AP-Product}}
& {84.20}  & {85.36} & {90.15} &  {\textbf{0.16}} 
& {\textbf{31.6}}  & {\textbf{262}}  & {\textbf{79.43}} 
\\
{\textsc{AP-Concate}}
& {83.29}  & {82.42} & {89.13} &  {\textbf{0.12}} 
& {\textbf{31.4}}  & {\textbf{261}}  & {\textbf{79.47}} 
\\
\rowcolor{Gray}
\ours
& 85.30  & 90.25 & 91.09 &  \textbf{0.16} 
& \textbf{31.6}  & \textbf{262}  & \textbf{79.43} 
\\
\bottomrule[1pt]
\end{tabular}
}
\end{center}
\end{table}

\paragraph{{Application of AP to multiple layers.}} 
{We implement {\ours} with multiple layers, and we show the results in Tab. \ref{tab: ap_multiple_layers}. Our findings indicate that the layer addition of {\ours} does not yield significant improvements in performance. This observation is significant as it suggests that applying {\ours} to a single, carefully selected layer can achieve comparable performance to more extensive applications. This underscores the efficiency of {\ours}, affirming its value in settings where computational resources are a concern.}

\setlength{\tabcolsep}{4pt}
\begin{table}[ht]
\caption{Ablation study on the number of  layers installed with AP. In particular, for AP-3 and AP-5, AP are installed on the input of the first 3 and 5 blocks of the pretrained ViT-L. Other experiment settings follow Tab.\,\ref{tab: main}, and Tab.
\,\ref{tab: baseline_overview}.
}
\label{tab: ap_multiple_layers}
\begin{center}
\resizebox{.55\linewidth}{!}{
\begin{tabular}{
l
ccc  !{\color{lightgray}\vrule}
c ccc cc 
}
\toprule[1pt]

& \multicolumn{3}{c}{\textbf{Accuracy}}
& \multicolumn{4}{c}{\textbf{Efficiency}}

\\
\cmidrule{2-8}
&\multicolumn{3}{c!{\color{lightgray}\vrule}}{\textbf{Full-Data}}
& 
&\multicolumn{3}{c}{\textbf{Train-Time Efficiency}}
\\
&{\small{FGVC}} &{\small{VTAB}} &{\small{Others}}
&{Param. \#} &{\small{Memory}} &{\small{Time}} &{\small{Throughput}} 
\\
\midrule
Number of tasks
&5 &9 &5
& -& - & - & -
\\
\midrule 
{\ff} 
& 91.43 & 91.97 & 93.91
& 304.33 
&  41.5  &  520 &  79.58
\\
{\lp} 
& 82.23  & 78.90 & 87.81
& 0.01
&  9.7 & 121  & 79.64
\\

\midrule
\bias 
& 85.32 & 89.84 & 90.41 &  0.29 
& 32.9  & 297 & {79.48} 
\\ 
\lora
& 86.87  & 89.81 & 91.45 &  1.00 
& 33.1  & 363 & {79.43}
\\
\vpt
& 86.05  & 89.97 & 90.64 &  1.24 
& 38.6  & 397  & 72.84
\\ 
\adapter
& 87.06  & 89.44 & 91.21 &  2.17 
& 32.4  & 357  & 63.39
\\
\adapterformer
& \textbf{89.18}  & {90.69} & \textbf{92.08} &  0.65 
& 32.3  & 289  & 23.69
\\
\midrule
{\ours-3}
& {85.41}  & {90.38} & {91.21} &  {0.46}
& {47.8}  & {297}  & {79.43} 
\\
{\ours-5}
& {85.49}  & {90.49} & {91.31} &  {0.76}
& {69.7}  & {348}  & {79.43} 
\\
\midrule
\rowcolor{Gray}
\ours
& 85.30  & 90.25 & 91.09 &  0.16 
& 31.6  & 262  & 79.43 
\\
\bottomrule[1pt]
\end{tabular}
}
\end{center}
\end{table}

\paragraph{{Performance comparison with re-initialized classification head.}}
{We carried out an ablation experiment using re-initialized classification head. This will influence the tunable parameter counts of {\lp} and other methods involved. As we can see, the results in Tab. \ref{tab: vit_ablation} are nearly identical to our previous findings in Tab. \ref{tab: more_peft} that {\ours} shows a competitive performance and efficiency compared with other strong PEFT baselines.}

\setlength{\tabcolsep}{4pt}
\begin{table}[ht]
\caption{\footnotesize{Performance comparison between {\ours} and SOTA PEFT methods on ViT-Large/16 with {re-initialized classification head.} Experiment settings follow Tab.\,\ref{tab: main}, and Tab.
\,\ref{tab: baseline_overview}.
}}
\label{tab: vit_ablation}
\begin{center}
\resizebox{.55\linewidth}{!}{
\begin{tabular}{
l
ccc  !{\color{lightgray}\vrule}
c ccc cc 
}
\toprule[1pt]

& \multicolumn{3}{c}{\textbf{Accuracy}}
& \multicolumn{4}{c}{\textbf{Efficiency}}

\\
\cmidrule{2-8}
&\multicolumn{3}{c!{\color{lightgray}\vrule}}{\textbf{Full-Data}}
& 
&\multicolumn{3}{c}{\textbf{Train-Time Efficiency}}
\\
&{\small{FGVC}} &{\small{VTAB}} &{\small{Others}}
&{Param. \#} &{\small{Memory}} &{\small{Time}} &{\small{Throughput}} 
\\
\midrule
Number of tasks
&5 &9 &5
& -& - & - & -
\\
\midrule 
{\ff} 
& 91.43 & 91.97 & 93.91
& 304.33 
&  41.5  &  520 &  79.58
\\
{\lp} 
& 82.31  & 78.43 & 87.71
& 0.01
&  8.1 & 121  & 79.69
\\

\midrule
\bias 
& 85.49 & 89.47 & 90.85 &  0.29
& \textbf{27.4}  & 297 & \textbf{79.51} 
\\ 
\lora
& 86.49  & 89.74 & 91.49 &  1.00 
& 32.5  & 363 & 71.47
\\
\vpt
& 86.15  & 90.13 & 90.88 &  1.24 
& 37.2  & 397  & 72.91
\\ 
\adapter
& 87.14  & 89.12 & 91.01 &  2.07 
& 31.1  & 357  & 63.78
\\
\adapterformer
& \textbf{89.24}  & \textbf{90.49} & \textbf{92.21} &  0.65 
& 31.1  & 289  & 23.82
\\
\midrule
\rowcolor{Gray}
\ours
& 85.32  & 90.12 & 91.11 &  \textbf{0.16} 
& 30.2  & \textbf{262}  & \textbf{79.54} 
\\
\bottomrule[1pt]
\end{tabular}
}
\end{center}
\end{table}

\paragraph{{Comparison to VPT with other prompt lengths.}}
{We conducted an experiment to implement VPT-Deep using a smaller prompt token length 10 (VPT-10). The results, presented in Tab. \ref{tab: vpt_ablation}, indicate that VPT-10's performance is comparable to VPT-50 in Tab. \ref{tab: more_peft}, albeit with enhanced efficiency.}

\setlength{\tabcolsep}{4pt}
\begin{table}[ht]
\caption{\footnotesize{Performance comparison between {\ours} and {\vpt} with different prompt lengths on ViT-Large/16. Experiment settings follow Tab.\,\ref{tab: main}, and Tab.
\,\ref{tab: more_peft}.
}}
\label{tab: vpt_ablation}
\begin{center}
\resizebox{.55\linewidth}{!}{
\begin{tabular}{
l
ccc  !{\color{lightgray}\vrule}
c ccc cc 
}
\toprule[1pt]

& \multicolumn{3}{c}{\textbf{Accuracy}}
& \multicolumn{4}{c}{\textbf{Efficiency}}

\\
\cmidrule{2-8}
&\multicolumn{3}{c!{\color{lightgray}\vrule}}{\textbf{Full-Data}}
& 
&\multicolumn{3}{c}{\textbf{Train-Time Efficiency}}
\\
&{\small{FGVC}} &{\small{VTAB}} &{\small{Others}}
&{Param. \#} &{\small{Memory}} &{\small{Time}} &{\small{Throughput}} 
\\
\midrule
Number of tasks
&5 &9 &5
& -& - & - & -
\\
\midrule 
{\ff} 
& 91.43 & 91.97 & 93.91
& 304.33 
&  41.5  &  520 &  79.58
\\
{\lp} 
& 82.23  & 78.90 & 87.81
& 0.01 
&  9.7 & 121  & 79.64
\\

\midrule
{\vpt-10}
& {{86.34}}  & {89.24} & {90.14} &  {0.25} 
& {33.7}  & {334}  & {76.35}
\\ 
\vpt-50
& 86.05  & 89.97 & 90.64 &  1.24 
& 38.6  & 397  & 72.84
\\ 
\midrule
\rowcolor{Gray}
\ours
& 85.30  & 90.25 & 91.09 &  \textbf{0.16} 
& \textbf{31.6}  & \textbf{262}  & \textbf{79.43} 
\\
\bottomrule[1pt]
\end{tabular}
}
\end{center}
\end{table}

\paragraph{{Layerwise comparison between AP and VPT-Deep.}}
{We conduct an experiment for a more detailed layer-wise evaluation in Fig. \ref{fig: vpt_deep}. These additional results highlight a consistent layer-architecture influence on VPT-Deep, akin to what we initially observed in our original {\ours} design. This outcome is not unexpected, considering that the implementation of VPT-Deep essentially converges with that of {\ours} when a specific network layer is selected for prompting. The key divergence lies in the prompt design approach: VPT-Deep favors concatenation, whereas {\ours} opts for addition in prompt design. It is worth noting that, in the context of single-layer prompting, the efficacy of concatenation in prompt design is comparatively lower than that of addition.}

\begin{figure}[ht]
    \centering
    \includegraphics[width=0.35\linewidth]{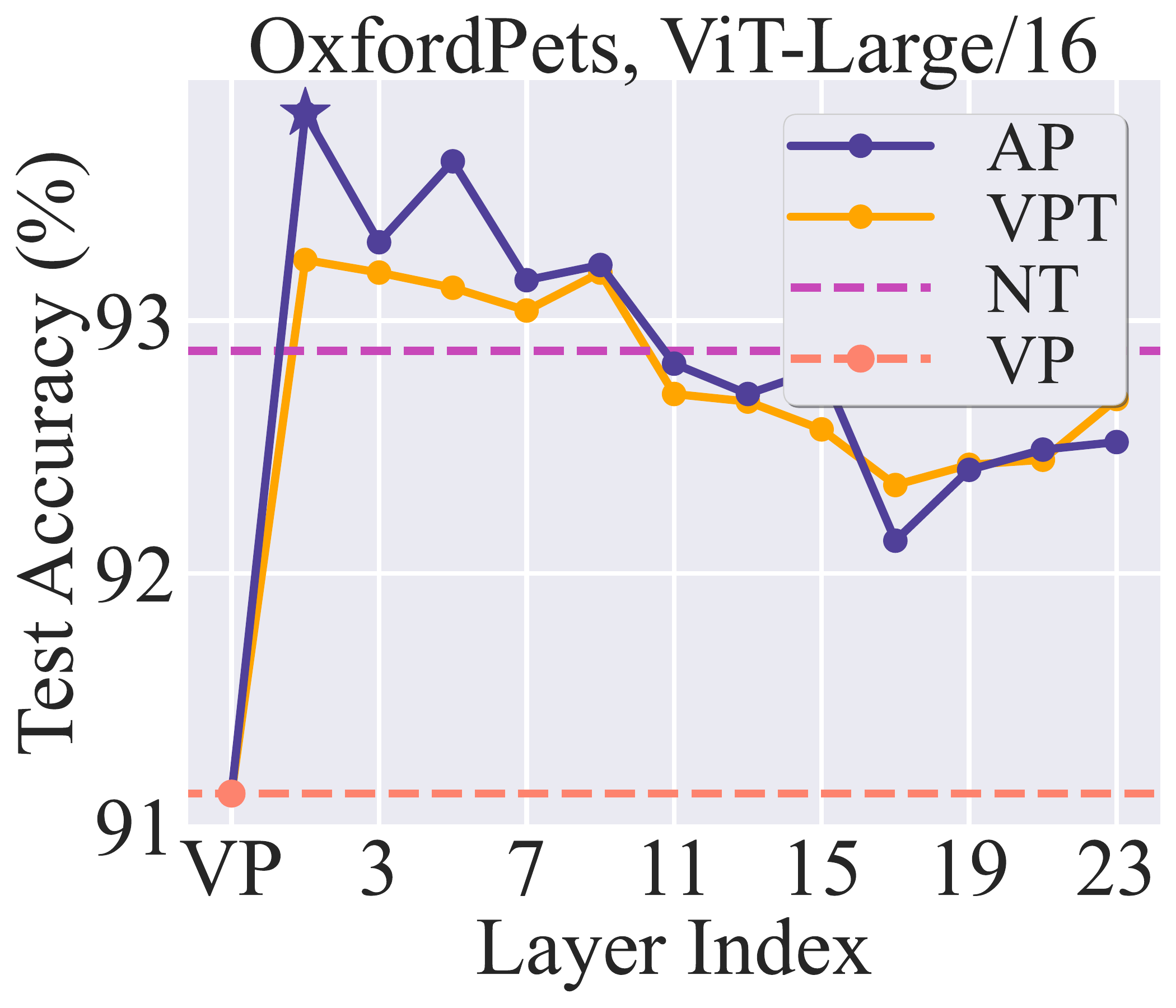}
    \caption{Layer-wise performance comparison between {\ours} and VPT on OxfordPets.}
    \label{fig: vpt_deep}
\end{figure}

\paragraph{{Comparison with additional PEFT methods.}}
{We conduct an experiment and report the results of SSF in Tab. \ref{tab: ssf}. In particular, we can see SSF is also a competitive method among all the baselines but is still under AdapterFormer. Compared to {\ours}, SSF yields better performance for the FGVC benchmark but leads to slightly worse accuracy for the VTAB benchmark. In general, SSF ranks approximately the second or the third place among all the PEFT methods.}

\begin{table}[ht]
\caption{\footnotesize{Performance comparison of {\ours} with more PEFT methods (SSF \citep{lian2022scaling}). Experiment settings follow Tab.\,\ref{tab: main} and Tab.\,\ref{tab: more_peft}.
}}
\label{tab: ssf}
\begin{center}
\resizebox{.55\linewidth}{!}{
\begin{tabular}{
l
ccc  !{\color{lightgray}\vrule}
c ccc cc 
}
\toprule[1pt]

& \multicolumn{3}{c}{\textbf{Accuracy}}
& \multicolumn{4}{c}{\textbf{Efficiency}}

\\
\cmidrule{2-8}
&\multicolumn{3}{c!{\color{lightgray}\vrule}}{\textbf{Full-Data}}
& 
&\multicolumn{3}{c}{\textbf{Train-Time Efficiency}}
\\
&{\small{FGVC}} &{\small{VTAB}} &{\small{Others}}
&{Param. \#} &{\small{Memory}} &{\small{Time}} &{\small{Throughput}} 
\\
\midrule
Number of tasks
&5 &9 &5
& -& - & - & -
\\
\midrule 
{\ff} 
& 91.43 & 91.97 & 93.91
& 304.33 
&  41.5  &  520 &  79.58
\\
{\lp} 
& 82.23  & 78.90 & 87.81
& 0.01
&  9.7 & 121  & 79.64
\\

\midrule
\bias 
& 85.32 & 89.84 & 90.41 &  0.29 
& 32.9  & 297 & \textbf{79.48} 
\\ 
\lora
& 86.87  & 89.81 & 91.45 &  1.00 
& 33.1  & 363 & \textbf{79.43}
\\
\vpt
& 86.05  & 89.97 & 90.64 &  1.24 
& 38.6  & 397  & 72.84
\\ 
\adapter
& 87.06  & 89.44 & 91.21 &  2.17 
& 32.4  & 357  & 63.39
\\
\adapterformer
& \textbf{89.18}  & \textbf{90.69} & \textbf{92.08} &  0.65 
& 32.3  & 289  & 23.69
\\
{\textsc{SSF}}
& {87.32}  & {89.43} & {92.21} &  {0.48} 
& {34.7}  & {299}  & {\textbf{79.49}} 
\\
\midrule
\rowcolor{Gray}
\ours
& 85.30  & 90.25 & 91.09 &  \textbf{0.16} 
& \textbf{31.6}  & \textbf{262}  & \textbf{79.43} 
\\
\bottomrule[1pt]
\end{tabular}
}
\end{center}
\end{table}

\paragraph{{Comparison with LoRA of different rank values.}}
{We conduct additional experiments on the hyper-parameters of LoRA, namely the rank $r$. In Tab.\,\ref{tab: more_peft}, the rank $r$ is adopted to 10 by default. In Tab. \ref{tab: lora_ablation}, we explore more rank values varying from $1$ to $50$. We can see that the performance of LoRA increases with the larger rank values, but the difference between $r=10$ and $r=50$ is insignificant. In contrast, the efficiency of LoRA will drop significantly with a rank larger than $10$. In order to strike a balance between performance and efficiency, we adopt the rank value of 10 as the default value in this work.}

\setlength{\tabcolsep}{4pt}
\begin{table}[ht]
\caption{\footnotesize{Ablation study on performance of {\lora} with different rank values. Experiment settings follow Tab.\,\ref{tab: main} and Tab.\,\ref{tab: more_peft}.
}}
\label{tab: lora_ablation}
\begin{center}
\resizebox{.55\linewidth}{!}{
\begin{tabular}{
l
ccc  !{\color{lightgray}\vrule}
c ccc cc 
}
\toprule[1pt]

& \multicolumn{3}{c}{\textbf{Accuracy}}
& \multicolumn{4}{c}{\textbf{Efficiency}}

\\
\cmidrule{2-8}
&\multicolumn{3}{c!{\color{lightgray}\vrule}}{\textbf{Full-Data}}
& 
&\multicolumn{3}{c}{\textbf{Train-Time Efficiency}}
\\
&{\small{FGVC}} &{\small{VTAB}} &{\small{Others}}
&{Param. \#} &{\small{Memory}} &{\small{Time}} &{\small{Throughput}} 
\\
\midrule
Number of tasks
&5 &9 &5
& -& - & - & -
\\
\midrule 
{\ff}  
& 91.43 & 91.97 & 93.91
& 304.33 
&  41.5  &  520 &  79.58
\\
{\lp} 
& 82.23  & 78.90 & 87.81
& 0.01 
&  9.7 & 121  & 79.64
\\

\midrule
{\lora-1}
& {84.43}  & {88.21} & {90.07} &  {0.04} 
& {10.43}  & {139} & {{79.43}}
\\
\lora-10
& 86.87  & 89.81 & 91.45 &  1.00 
& 33.1  & 363 & {79.43}
\\
{\lora-20}
& {86.93}  & {90.23} & {91.35} &  {4.38} 
& {33.1}  & {443} & {{79.43}}
\\
{\lora-50}
& {87.23}  & {90.41} & {91.97} &  {12.22} 
& {57.2}  & {589} & {{79.43}}
\\
\midrule
\rowcolor{Gray}
\ours
& 85.30  & 90.25 & 91.09 &  0.16 
& {31.6}  & {262}  & {79.43} 
\\
\bottomrule[1pt]
\end{tabular}
}
\end{center}
\end{table}

\textbf{Ablation study on the influence of different data sizes.} We recognize that data size significantly influences performance. To ensure that our conclusions generalize well, we conducted an ablation study on \textsc{Full-Finetune}, VP, and AP, varying the training data ratio from $10\%$ to $100\%$ on datasets with large training sizes (Camelyon, FOOD101, CIFAR10). The results are shown in Figure \ref{fig: data_size}. Results show that \textsc{Full-Finetune} benefits the most from larger datasets. However, AP consistently outperforms VP, regardless of data size, reinforcing that AP is a better design than VP for both few- and many-shot settings.

\begin{figure}[htb]
\vspace*{-1em}
    \centering
    \begin{tabular}{c}
         \includegraphics[width=1\linewidth]{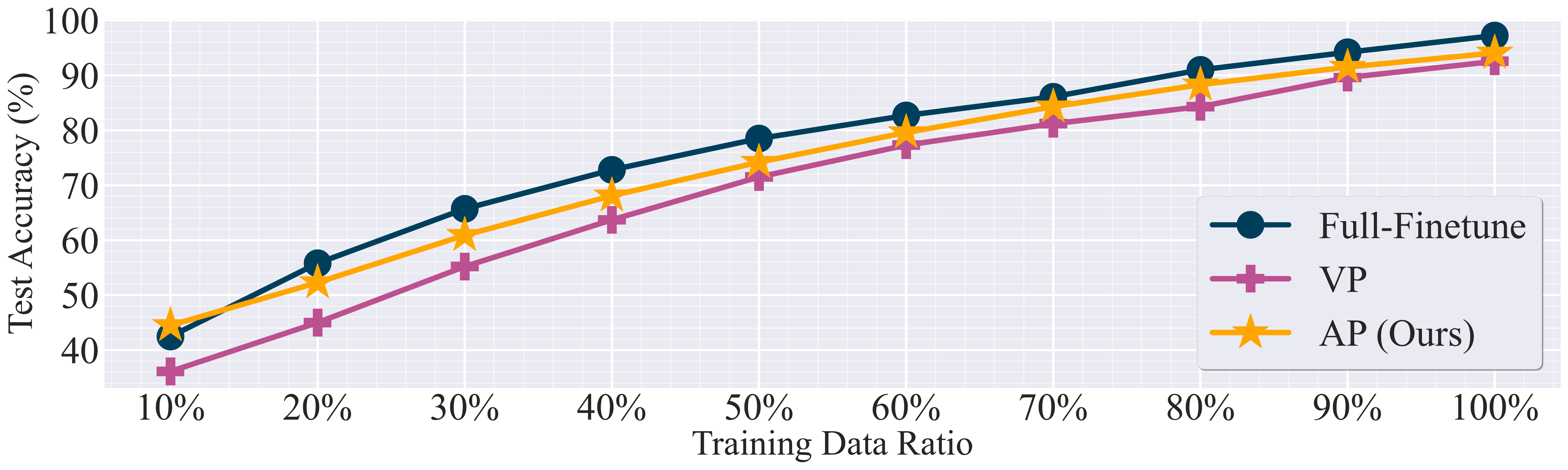} \\
         (a) CIFAR10 \\ 
         \includegraphics[width=1\linewidth]{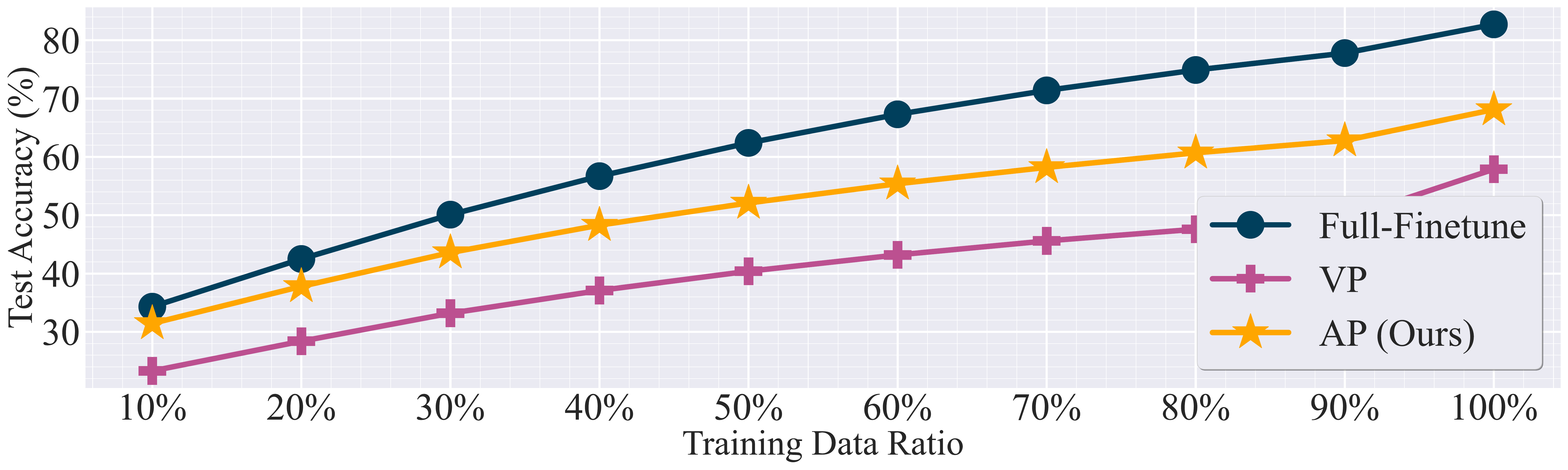} \\
         (b) Food101 \\
         \includegraphics[width=1\linewidth]{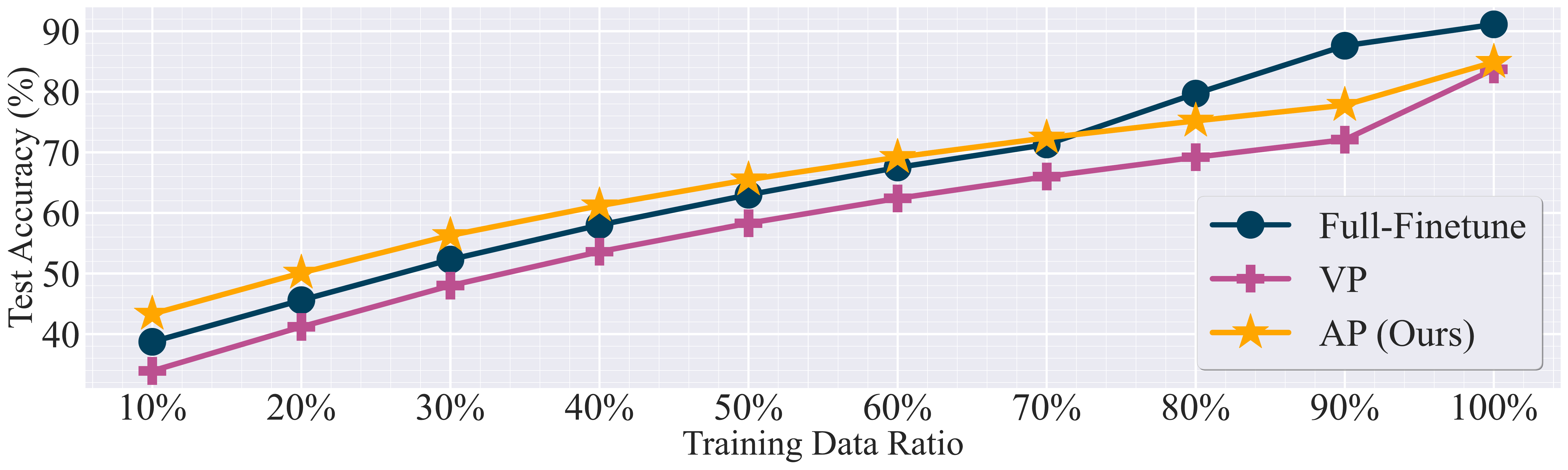} \\
         (c) Camelyon
    \end{tabular}
    \vspace*{-.5em}
    \caption{Performance of ResNet101 trained with varying sizes of available training data on (a) CIFAR10, (b) Food101, and (c) Camelyon. All other experimental settings strictly follow those in Tab.\,\ref{tab: main}.}
    \label{fig: data_size}
\end{figure}

\clearpage

\section{Theoretical details}
\label{app: theory}

\subsection{Model architecture}

We define the general definition of the model architecture CNN, ViT in this section.

\textbf{CNN}: We follow the architecture of ResNet \citep{}, which stacks multiple residual blocks plus an input and an output layer. Each residual block includes several convolutional layers and a skip connection. For the input $\bfz_{\text{in}}^{(l)}$ to the $l$-th convolutional layer, where $l\in[L]$, the output $\bfz_{\text{out}}^{(l)}$ can be computed as
\begin{equation}
\begin{aligned}
    &\bfz^{(l)}=\text{Conv}(\bfz_{\text{in}}^{(l)}; \bfW^{(l)}_1),\  \bfz_{\text{out}}^{(l)}=\text{relu}(\text{BN}(\bfz^{(l)}))
\label{eqn: CNN}
\end{aligned}
\end{equation}
where $\bfz_{\text{in}}^{(0)}=\bfx$. $\text{Conv}(\cdot)$ and $\text{BN}$ denote the Convolution operation and the Batch Normalization, respectively. The output $\hat{y}=\text{FC}(\text{Pooling}(\bfz_{\text{out}}^{(L)}))$, where $\text{FC}(\cdot)$ denotes fully-connected layer. 

\textbf{ViT}: The architecture of Vision Transformer is defined in \citep{}. For the input $\bfz_{\text{in}}^{(l)}$ to the $l$-th Transformer layer, we first let $\bfz^{(l)}=\bfz_{\text{in}}^{(l)}$. Then, the output $\bfz_{\text{out}}^{(l)}$ can be computed as
\begin{equation}
\begin{aligned}\label{eqn: ViT}
    \bfz^{(l)}=\text{MSA}(\text{LN}(\bfz^{(l)}))+\bfz^{(l)},\  \bfz_{\text{out}}^{(l)}=\text{MLP}(\text{LN}(\bfz^{(l)}))+\bfz^{(l)},
\end{aligned}
\end{equation}
where $\bfz_{\text{in}}^{(0)}=\bfx$. $\text{MSA}(\cdot)$ and $\text{LN}(\cdot)$ denote the Multi-Head Self-attention and Layer Normalization, respectively. For an $L$-layer ViT, the output $\hat{y}=\text{Out}(\bfH_{\text{out}}^{(L)})$, where $\text{Out}(\cdot)$ denotes the output layer.

\subsection{Proposition \ref{prpst: AP-NT full} and its proof}\label{subsec: proof-AP-NT}
We first provide a full definition of {\nt}. 

{\nt} is a method where only the Batch Normalization layers for CNNs or Layer Normalization for ViTs are trainable. Consider a batch of the $l$-th-layer features $\bfz^{(l)}_1, \bfz^{(l)}_2,\cdots, \bfz^{(l)}_B$ defined in (\ref{eqn: CNN}) and (\ref{eqn: ViT}), where $\bfz^{(l)}_b=[{\bfz^{(l)}_{b,\cdot, 1}}, {\bfz^{(l)}_{b,\cdot,2}},\cdots, {\bfz^{(l)}_{b,\cdot, P'}}]=\in\mathbb{R}^{D'\times P'}$, $\bfz^{(l)}_{b,\cdot, p}\in\mathbb{R}^{D'}$ for $b\in[B]$ and $p\in[P']$. $B$ is the batch size, $D'$ denotes the number of channels or token dimension, and $P'$ denotes the size of the feature map or token length. We can formulate the Normalization on $h^{(l)}_{b,d,p}$, the $d$-th dimension of $\bfh^{(l)}_{b,\cdot, p}$, as follows.
\begin{equation}
\begin{aligned}
    &\textbf{BN}: \mu_d=\sum_{b=1}^B\sum_{p=1}^{P'}\frac{z^{(l)}_{b,d,p}}{BP'},\  \sigma_d^2=\sum_{b=1}^B\sum_{p=1}^{P'}\frac{(z^{(l)}_{b,d,p}-\mu_d)^2}{BP'},\  \text{BN}(z^{(l)}_{b,d,p})=\gamma_d \frac{z^{(l)}_{b,d,p}-\mu_d}{\sigma_d}+\beta_d,\\
    &\textbf{LN}: \mu_{b,p}=\sum_{d=1}^{D'}\frac{z^{(l)}_{b,d,p}}{D'},\  \sigma_{b,p}^2=\sum_{d=1}^{D'}\frac{(z^{(l)}_{b,d,p}-\mu_{b,p})^2}{D'},\  \text{LN}(z^{(l)}_{b,d,p})=\gamma_d \frac{z^{(l)}_{b,d,p}-\mu_{b,p}}{\sigma_{b,p}}+\beta_d,
\end{aligned}\label{eqn: normalization}
\end{equation}
where $\gamma_d$, $\beta_d$ are trainable parameters for $d\in[D']$. Then, we present a full statement of Proposition \ref{prpst: AP-NT full}.

\begin{prpst}\label{prpst: AP-NT full}
\textbf{Without} the assumption that the input to the batch (or layer) normalization layer has zero mean and unit variance for each dimension (or token), we have the following conclusion:

{\ours} on the $l$-th layer is the same as {\nt} on the $l$-th layer, if 
\begin{itemize}
    \item \textbf{for CNNs}, $\gamma_d/\sigma_d=1$, and all $\bfdt_p$'s added to $\bfz_b^{(l)}$ are the same as $\bfdt$, $\beta_d = \bfw_{d}^{(l)} \bfdt_*+\bfmu_d$ for all $d\in[D']$, where $\boldsymbol{\delta}_*=\boldsymbol{\delta}_i^{(l)}$ for $i\in[P']$; 
    \item \textbf{for ViTs}, $\gamma_d/\sigma_{b,p}=1$, and $\mu_{b,p}$'s are the same as $\mu_p,\ p\in[P']$ among all $b\in[B]$ for ViTs, $\beta_d = \delta_{p,d}^{(l)}+\mu_p$ for all $d\in[D']$, $p\in[P']$. 
\end{itemize} 
\end{prpst}

\textbf{Proof:}

For BN, note that
\begin{equation}
    \text{BN}(z^{(l)}_{b,d,p})=\gamma_d \frac{z^{(l)}_{b,d,p}-\mu_d}{\sigma_d}+\beta_d=\frac{\gamma_d}{\sigma_d}z^{(l)}_{b,d,p}+\beta_d-\frac{\mu_d\gamma_d}{\sigma_d}
\end{equation}
where
\begin{equation}
    z^{(l)}_{b,d,p}=\bfw_d^{(l)}\bfz^{(l-1)}_{b,\cdot,p},\  \bfz^{(l-1)}_{b,\cdot,p}=\bfx_{b,\cdot,p}
\end{equation}
When adding the prompt $\bfdt^{(l)}_p$, we have the output
\begin{equation}
    \bfw_d^{(l)}(\bfz^{(l-1)}_{b,\cdot,p}+\bfdt^{(l)}_p)
\end{equation}
We then need the equation
\begin{equation}
    \frac{\gamma_d}{\sigma_d}z^{(l)}_{b,d,p}+\beta_d-\frac{\mu_d\gamma_d}{\sigma_d}=\bfw_d^{(l)}(\bfz^{(l-1)}_{b,\cdot,p}+\bfdt^{(l)}_p)
\end{equation}
Given $\gamma_d/\sigma_d=1$, we have
\begin{equation}
    \beta_d=\bfw_d^{(l)}\bfdt_p^{(l)}+\mu_d
\end{equation}
Suppose that $\mu_d=0$ for $d\in[D']$ and $\bfdt_p^{(l)}=\bfdt_*$ for $p\in[P']$, we can obtain
\begin{equation}
    \beta_d=\bfw_d^{(l)}\bfdt_*
\end{equation}
For LN, we need 
\begin{equation}
    \text{LN}(z^{(l)}_{b,d,p})=\gamma_d \frac{z^{(l)}_{b,d,p}-\mu_{b,p}}{\sigma_{b,p}}+\beta_d=\frac{\gamma_d}{\sigma_{b,p}}z^{(l)}_{b,d,p}+\beta_d-\frac{\gamma_d \mu_{b,p}}{\sigma_{b,p}}=z^{(l)}_{b,d,p}+\delta_{p,d}^{(l)}
\end{equation}
Given $\gamma_d/\sigma_{b,p}=1$ and $\bfmu_{b,p}=\bfmu_p$ for $b\in[B]$, we have
\begin{equation}
    \beta_d=\delta_{p,d}^{(l)}+\mu_p
\end{equation}
Suppose that $\mu_p=0$, $p\in[P']$ and let $\bfdt_p^{(l)}=\bfdt_*$, $p\in[P']$, we can obtain
\begin{equation}
    \boldsymbol{\beta}=\boldsymbol{\delta}_*
\end{equation}

\subsection{Proof of Lemma \ref{lemma: delta-0}}\label{subsec: prop3}

Before we provide the proof, we state the formulation of a single-head and two-layer ViT, the full assumption on the data model, and the pretrained model in detail. 

Let $\bfx_{n(\cdot,j)}$ be the $j$-th patch/token of $\bfx_n$, $j\in[P]$. The corresponding $1$-st-layer output is $\bfz_{n(\cdot, j)}$. Denote the $j$-th patch/token of $\bfx_n$ or $\bfz_n$ after introducing the {\ours}, $\bdelta^{(h)}$, as $\bfx_n[\bdelta_j^{(h)}]$ and $\bfz_n[\bdelta_j^{(h)}]=(\bfz_n[\bfdt_1^{(h)}],\cdots, \bfz_n[\bfdt_P^{(h)}])$, respectively.

Following \citep{dosovitskiy2020image}, we consider a single-head self-attention parameterized by $\bfW_{Q}^{(l)}$, $\bfW_{K}^{(l)}$, and $\bfW_{V}^{(l)}$ in the $l$-th layer. The shapes of these matrices are $m$ by $d$ if $l=1$ and $m$ by $m$ if $l=2$. Denote $\bfW^{(l)}={\bfW_K^{(l)}}^\top\bfW_Q^{(l)},\ l=1,2$. The MLP layer is a two-layer perceptron with $m\times m$-dimensional parameters $\bfW_{O}^{(l)}$, $\bfW_{U}^{(l)}$,  and Relu activation. The output layer is a fully-connected layer with $\bfa_1,\cdots, \bfa_P$ where $\bfa_l\in\mathbb{R}^m$. Then, a two-layer ViT can be written as
\begin{equation}
    \begin{aligned}
    f_{\btheta}(\bfx_n,\bfdt^{(h)})=\sum_{k=1}^P\bfa_k^\top\bfW_{U}^{(2)}\text{Relu}(\bfW_{O}^{(2)}\bfW_{V}^{(2)}\bfz_n[\bfdt^{(h)}]\text{softmax}(\bfz_n[\bfdt^{(h)}]^\top{\bfW^{(2)}}\bfz_n[\bfdt_k^{(h)}])),\\
    \bfz_n[\bfdt_k^{(h)}]=\bfW_{U}^{(1)}\text{Relu}(\sum_{s=1}^P\bfW_{O}^{(1)}\bfW_{V}^{(1)}\bfx_n[\bfdt_s^{(h)}]\text{softmax}(\bfx_n[\bfdt_s^{(h)}]^\top{\bfW^{(1)}}\bfx_n[\bfdt_k^{(h)}])),
\end{aligned}\label{eqn: 2l vit}
\end{equation}
The {\ours} is restated as
\begin{equation}
    \begin{cases}
    \bfx_n[\bfdt_j^{(h)}]=\bfx_{n(\cdot,j)}+\bfdt_j^{(h)}, \bfz_n[\bfdt_j^{(h)}]\text{ as defined in (\ref{eqn: 2l vit})}, &\text{ if }h=1,\\
    \bfx_n[\bfdt_j^{(h)}]=\bfx_{n(\cdot,j)}, \bfz_n[\bfdt_j^{(h)}]=\bfz_{n(\cdot,j)}+\bfdt_j^{(h)},  &\text{ if }h=2,
    \end{cases}
\end{equation}
 We use Hinge loss $\ell(\bfx_n.y_n)=\max\{0,1/P-y_n f_{\btheta}(\bfx_n, \bdelta^{(h)})\}$ as the loss function.

\textbf{Data model} 
The patch/token $\bfx_{n(\cdot, j)}$ is a noisy version of patterns, i.e., $\bfx_{n(\cdot, j)}=\bfv_l+\epsilon_j^n$, where $\bfv_l,\ l=1,2,3,4$ is a pattern and $\epsilon_j^n\sim\mathcal{N}(0,\sigma^2)$ is a Gaussian noise, $\sigma\leq O(1/P)$. 
$\bfv_1$, $\bfv_2$, $\bfv_3$, $\bfv_4$ are all unit norm and orthogonal to each other except the pairs of $\bfv_3$ and $\bfv_4$. $\bfv_3^\top\bfv_4=\zeta\in(-1,0)$. In each sample $\bfx_n$, only one patch/token $\bfx_{n(\cdot,j)}$ corresponds to either $\bfv_1$ or $\bfv_2$, while other $P-1$ patches/tokens correspond to either $\bfv_3$ or $\bfv_4$. $\bfv_1, \bfv_2$ are called discriminative patterns that decide the label. $\bfv_3, \bfv_4$ are non-discriminative patterns that work as the image background. For instance, if one patch is the noisy version of $\bfv_1$ ($\bfv_2$), then $y^n=1$ ($y^n=-1$). 

\textbf{Pretrained model}
The pretraining stage is assumed to learn a task where all patterns $\{\bfv_1,\bfv_2,\bfv_3,\bfv_4\}$ are key features, where each data contains two types of patterns. The label is determined by the number of $\bfv_1$ or $\bfv_3$ compared with the number of $\bfv_2$ or $\bfv_4$. Inspired by the finding that some trained ``lucky'' hidden neurons represent discriminative features from existing theoretical works \citep{li2023theoretical} on VITs, we accordingly set the neurons of feed-forward-networks 
$\bfW_{O}^{(i)}$ in (\ref{eqn: 2l vit}), $i=1,2$ as pattern representations of that layer and ignore ``unlucky'' neurons, which has a trivial effect on the output. To be more specific, for the 1st layer, we set a $1/4$ fraction of neurons to be $\bfv_i,\ i=1,2,3,4$, and for the 2nd layer, we set a $1/4$ fraction of neurons to be $\bfe_i,\ i=1,2,3,4$, i.e., the 2nd-layer pattern representations. $\bfW_U^{(1)}=\bfW_U^{(2)}=\bfI$. $a_{l(i)}$ equal $1/(mP)$ for neurons of $\bfe_1$ and $\bfe_3$, and they equal $-1/(mP)$ for neurons of $\bfe_2$ and $\bfe_4$.  For ViTs, we follow the orthogonal embedding assumption in \citep{oymak2023role, li2023theoretical, LWML23, HCL23, LWLW23, li2024training, li2024nonlinear, li2024learning, litask,chen2024unveiling, zhang2025merging, shandirasegaran2026a} and set $\bfW_{Q}^{(1)}=\beta_1\bfI$, $\bfW_{K}^{(1)}=\beta_1\bfP_x^{(1)}$, $\bfW_{Q}^{(2)}=\beta_2\bfI$, $\bfW_{K}^{(2)}=\beta_2\bfP_x^{(2)}$, $\bfW_{V}^{(1)}=\bfP_x^{(1)}$, $\bfW_V^{(2)}=\bfP_x^{(2)}$ for simplicity, where $\beta_1=\Theta(1)$, $\beta_2=\Theta(1)$, $\bfI$ is the identity matrix, and $\bfP_x^{(1)}$ and $\bfP_x^{(2)}$ are permutation matrices. 

Then, we present the proof of Lemma \ref{lemma: delta-0}.

\textbf{Proof:}

Without loss of generality, we focus on studying the data where $\bfv_1$ is the discriminative pattern, and $\bfv_4$ is the non-discriminative pattern.

For ViTs, note that the permutation matrix $\bfP_x^{(1)}$ changes the location of the pattern $\bfv_1$ to another place with a distance of at least $d_A$. By computing the feature correlation for the pattern $\bfv_1$, we have
\begin{equation}
    \beta_1^2>0,
\end{equation}
which means the the pattern $\bfv_1$ has the largest correlation with $\bfv_1$. Hence, the pattern of $\bfv_1$ is a global feature. For the feature correlation of the pattern $\bfv_4$, we have
\begin{equation}
    \beta_1^2>0,
\end{equation}
which means the the pattern $\bfv_4$ has the largest correlation with $\bfv_4$. Hence, the pattern of $\bfv_4$ is a global feature because the distance between two $\bfv_4$ patterns is at most $1$. Since that there will be one $\bfv_4$ token corresponding to a $\bfv_1$ token after the permutation, there will be a contribution of distance $1$ to the average distance. The average attention distance of the first layer is
\begin{equation}
    \frac{1}{P}\sum_{i=1}^P |i-\arg\max_{j\in[P]}\left\langle \bfk_j, \bfq_i\right\rangle|=\frac{1+d_A}{P}
\end{equation}

After the first layer, the feature of the $\bfv_1$ token becomes
\begin{equation}
    \frac{e^{\beta_1^2}}{e^{\beta_1^2}+P-1}\bfv_1+\frac{P-1}{e^{\beta_1^2}+P-1}\bfv_4:=\lambda_1\bfv_1+(1-\lambda_1)\bfv_4,
\end{equation}
while the feature of the $\bfv_4$ token becomes
\begin{equation}
    \frac{1}{(P-1)e^{\beta_1^2}+1}\bfv_1+\frac{(P-1)\bfe^{\beta_1^2}}{(P-1)e^{\beta_1^2}+1}\bfv_4:=\lambda_2\bfv_1+(1-\lambda_2)\bfv_4,
\end{equation}
Here $1/2>\lambda_1>\lambda_2>0$. Therefore, we have
\begin{equation}
\begin{aligned}
    &(\lambda_1\bfv_1+(1-\lambda_1)\bfv_4)^\top(\lambda_1\bfv_1+(1-\lambda_1)\bfv_4-\lambda_2\bfv_1-(1-\lambda_2)\bfv_4)\\
    =& (2\lambda_1-1)(\lambda_1-\lambda_2)<0
\end{aligned}
\end{equation}
\begin{equation}
\begin{aligned}
    &(\lambda_2\bfv_1+(1-\lambda_2)\bfv_4)^\top(\lambda_2\bfv_1+(1-\lambda_2)\bfv_4-\lambda_1\bfv_1-(1-\lambda_1)\bfv_4)\\
    =& (2\lambda_2-1)(\lambda_2-\lambda_1)>0
\end{aligned}
\end{equation}
Therefore, the feature from the token of $\bfv_4$ has the largest correlation with the token of both $\bfv_1$ and $\bfv_4$. Since there exists a $\bfv_4$ token close to $\bfv_1$ token with a distance of at most $1$, we have that both $\bfv_1$ and $\bfv_4$ tokens become local features. Then, the average attention distance of the second layer is
\begin{equation}
    \frac{1}{P}\sum_{i=1}^P |i-\arg\max_{j\in[P]}\left\langle \bfk_j, \bfq_i\right\rangle|=\frac{1}{P}
\end{equation}

\subsection{Proof of Theorem \ref{thm: ViT}}\label{subsec: proof thm ViT}

We first present two lemmas. One can observe that Theorem \ref{thm: ViT} is a combination of these two lemmas. Therefore, the proof of Theorem \ref{thm: ViT} is exactly the same as the proof of these two lemmas. 
\begin{lemma}\label{lemma: ViT-deep}
    For a two-layer single-head Transformer
    \begin{equation}
    \begin{aligned}
    f_{\btheta}(\bfx_n, \bfdt)=&\sum_{l=1}^P\sum_{i=1}^m a_{l(i)}^\top\text{Relu}(\sum_{j=1}^P\bfW_{O_{2(i,\cdot)}}\bfW_{V_2}(\bfz_{n(\cdot,j)}+\bfdt_j^{(h)})\\
    &\cdot\text{softmax}(({\bfz_{n(\cdot,j)}}+\bfdt_j^{(h)})^\top\bfW_{K_2}^\top\bfW_{Q_2}(\bfz_{n(\cdot,l)}+\bfdt_l^{(h)})))
\end{aligned}
\end{equation}
where 
\begin{equation}
    \bfz_{n(\cdot,j)}=\text{Relu}(\sum_{s=1}^P\bfW_{O_{1}}\bfW_{V_1}\bfx_{n(\cdot,s)}\text{softmax}({\bfx_{n(\cdot,s)}}^\top\bfW_{K_1}^\top\bfW_{Q_1}\bfx_{n(\cdot,j)}))
\end{equation}
as long as the batch size and the required number of iterations satisfy
    \begin{equation}
        B\geq \Omega(1),\ \ \ \ T=\frac{\eta^{-1}P^2\log P}{(1-\sigma)^{-1}},
    \end{equation}
    where $\sigma\leq \Theta(P^{-1})$, training $\bfdt^{(h)},\ h=2$ with SGD returns a model with zero generalization error.
\end{lemma}

\begin{lemma}\label{lemma: ViT-shallow}
    For a two-layer single-head Transformer
    \begin{equation}
    \begin{aligned}
    f_{\btheta}(\bfx_n, \bfdt)=&\sum_{l=1}^P\sum_{i=1}^m a_{l(i)}^\top\text{Relu}(\sum_{j=1}^P\bfW_{O_{2(i,\cdot)}}\bfW_{V_2}\bfz_{n(\cdot,j)}\text{softmax}({\bfz_{n(\cdot,j)}}^\top\bfW_{K_2}^\top\bfW_{Q_2}\bfz_{n(\cdot,l)}))
\end{aligned}
\end{equation}
where 
\begin{equation}
    \bfz_{n(\cdot,j)}=\text{Relu}(\sum_{s=1}^P\bfW_{O_{1}}\bfW_{V_1}(\bfx_{n(\cdot,s)}+\bfdt_s^{(h)})\text{softmax}(({\bfx_{n(\cdot,s)}}+\bfdt_s^{(h)})^\top\bfW_{K_1}^\top\bfW_{Q_1}(\bfx_{n(\cdot,j)}+\bfdt_j^{(h)})))
\end{equation}
    as long as the batch size and the required number of iterations satisfy
    \begin{equation}
        B\geq \Omega(1),\ \ \ \ T=\frac{\eta^{-1}P}{(1-P\sigma)^{-1}(1+\gamma)},
    \end{equation}
    where $\sigma\leq O(P^{-1})$, training $\bfdt^{(h)},\ h=1$ with SGD returns a model with zero generalization error, where $\gamma:=\bfv_3^\top\bfv_4\in(-1,0)$.
\end{lemma}

\subsubsection{Proof of Lemma \ref{lemma: ViT-deep}}
\noindent \textbf{Proof:}

\noindent For $h=2$,
\begin{equation}
\begin{aligned}
    f_\theta(\bfx_n, \bfdt^{(h)})=&\sum_{l=1}^P\sum_{i=1}^m a_{l(i)}^\top\text{Relu}(\sum_{s=1}^P\bfW_{O_{(i,\cdot)}}\bfW_V(\bfz_{n(\cdot,s)}+\bfdt_s^{(h)})\\
    &\cdot\text{softmax}({(\bfz_{n(\cdot,s)}+\bfdt_s^{(h)})}^\top\bfW_K^\top\bfW_Q(\bfz_{n(\cdot,s)}+\bfdt_l^{(h)}))),
\end{aligned}
\end{equation}
we have $\bfW_K=\beta_2\cdot\bfP_x$, $\bfW_Q=\beta_2\cdot\bfI$, and $\bfW_V=\bfP_x$ where $\beta_2=\Theta(1)$. To avoid multiple superscripts, we use $\boldsymbol{\delta}$ to denote $\boldsymbol{\delta}^{(h)}$ since that $h$ is fixed in this proof. We use $\boldsymbol{\delta}^{(t)}$ to denote the update of $\boldsymbol{\delta}$ at $t$-th iteration. Then,
\begin{equation}
    \begin{aligned}
        &\frac{\partial f_\theta(\bfx_n, \bfdt )}{\partial \bfdt_j}\\
        =& \sum_{l=1}^P\sum_{i=1}^m a_{l(i)}\mathbbm{1}[\sum_{s=1}^P\bfW_{O_{(i,\cdot)}}(\bfz_{n(\cdot,P_{s,2})}+\bfdt_{P_{s,2}} )\text{softmax}({(\bfz_{n(\cdot,P_{s,2})}+\bfdt_{P_{s,2}} )}^\top(\bfz_{n(\cdot,s)}\\
        &+\bfdt_l ))\geq 0]\cdot\Big(\text{softmax}({(\bfz_{n(\cdot,P_{s,2})}+\bfdt_{P_{s,2}} )}^\top(\bfz_{n(\cdot,s)}+\bfdt_l ))\bfW_{O_{(i,\cdot)}}\\
        &+\mathbbm{1}[j\neq l] \bfW_{O_{(i,\cdot)}}(\bfz_{n(\cdot,j)}+\bfdt_j )\cdot (\bfz_{n(\cdot,j)}+\bfdt_l )\cdot(-\text{softmax}(\beta_2^2{(\bfz_{n(\cdot,j)}+\bfdt_j )}^\top\\
        &\cdot(\bfz_{n(\cdot,l)}+\bfdt_l )))\text{softmax}(\beta_2^2{(\bfz_{n(\cdot,l)}+\bfdt_l )}^\top(\bfz_{n(\cdot,l)}+\bfdt_l ))\\
        &+ \mathbbm{1}[j= l]\bfW_{O_{(i,\cdot)}}(\bfz_{n(\cdot,l)}+\bfdt_l )\text{softmax}(\beta_2^2{(\bfz_{n(\cdot,l)}+\bfdt_l )}^\top(\bfz_{n(\cdot,l)}+\bfdt_l ))\\
        &\cdot(1-\text{softmax}(\beta_2^2{(\bfz_{n(\cdot,l)}+\bfdt_l )}^\top(\bfz_{n(\cdot,j)}+\bfdt_l )))(\bfz_{n(\cdot,l)}+\bfdt_l )
    \end{aligned}
\end{equation}
Let $t=0$. For $y^n=+1$, Note that if $\bfz_n=[\bfe_3,\bfe_3,\cdots,\bfe_3,\bfe_1,\bfe_3,\cdots, \bfe_3]$ without noise, the loss is $0$. Hence, we compute the loss from $\bfz_n=[\bfe_4,\bfe_4,\cdots,\bfe_4,\bfe_1,\bfe_4,\cdots, \bfe_4]$.
\begin{equation}
\begin{aligned}
    &\mathbb{E}[\mathbbm{1}[\sum_{s=1}^P\bfW_{O_{(i,\cdot)}}(\bfx_{n(\cdot,s)}+\bfdt_s^{(t)})\text{softmax}(\beta_2^2{( \bfz_{n(\cdot,P_{s,2})}+\bfdt_{P_{s,2}}^{(t)})}^\top( \bfz_{n(\cdot,l)}+\bfdt_{l}^{(t)}))\geq 0]\\
    =&\Pr(\sum_{s=1}^L\bfW_{O_{(i,\cdot)}}( \bfz_{n(\cdot,P_{s,2})}+\bfdt_{P_{s,2}}^{(t)})\text{softmax}(\beta_2^2{( \bfz_{n(\cdot,P_{s,2})}+\bfdt_{P_{s,2}}^{(t)})}^\top( \bfz_{n(\cdot,l)}+\bfdt_{l}^{(t)}))\geq 0)
\end{aligned}
\end{equation}
for $\bfW_{O_{(i,\cdot)}}=\bfe_1$ or $\bfe_4$. We can finally show that with a high probability, the above indicator is close to $1$. Meanwhile, for $\bfW_{O_{(i,\cdot)}}=\bfe_2$ or $\bfe_3$, the indicator equals $0$ or $1$ with half probability when $t=0$. Consider that $\bfz_{n(\cdot,j)}$ comes from $\bfv_4$, which means $\bfz_{n(\cdot,j)}$ is close to $\bfv_4$ by a noisy term. In this case, if $\bfz_{n(\cdot,l)}$ comes from $\bfv_1$,

\begin{equation}
    \text{softmax}(\beta_2^2{( \bfz_{n(\cdot,l)}+\bfdt_{l}^{(t)})}^\top( \bfz_{n(\cdot,l)}+\bfdt_{l}^{(t)}))\geq \frac{1}{P}
\end{equation}
\begin{equation}
    \text{softmax}(\beta_2^2{(\bfz_{n(\cdot,j)}+\bfdt_j)}^\top( \bfz_{n(\cdot,l)}+\bfdt_{l}^{(t)}))= \Theta(\frac{1}{P})
\end{equation}
If $\bfz_{n(\cdot,l)}$ comes from $\bfv_4$, then
\begin{equation}
    \text{softmax}(\beta_2^2{( \bfz_{n(\cdot,l)}+\bfdt_{l}^{(t)})}^\top( \bfz_{n(\cdot,l)}+\bfdt_{l}^{(t)}))\geq \frac{1}{P}
\end{equation}
\begin{equation}
    \text{softmax}(\beta_2^2{(\bfz_{n(\cdot,j)}+\bfdt_j^{(t)})}^\top( \bfz_{n(\cdot,l)}+\bfdt_{l}^{(t)}))= \Theta(\frac{1}{P})
\end{equation}

\noindent Then we consider that $\bfz_{n(\cdot,j)}$ comes from $\bfe_1$. In this case, if $\bfz_{n(\cdot,l)}$ comes from $\bfv_1$, then
\begin{equation}
    \text{softmax}(\beta_2^2{(\bfz_{n(\cdot,j)}+\bfdt_j^{(t)})}^\top( \bfz_{n(\cdot,l)}+\bfdt_{l}^{(t)}))\geq \frac{1}{P}
\end{equation}
If $\bfz_{n(\cdot,l)}$ comes from $\bfv_4$,
\begin{equation}
    \text{softmax}(\beta_2^2{(\bfz_{n(\cdot,j)}+\bfdt_j^{(t)})}^\top( \bfz_{n(\cdot,l)}+\bfdt_{l}^{(t)}))\leq \frac{1}{P}
\end{equation}
Therefore, if $\bfz_{n(\cdot,j)}$ comes from $\bfv_1$,
\begin{equation}
        \frac{\partial f_{\btheta}(\bfx_n,\bfdt^{(t)})}{\partial \bfdt_j^{(t)}}= \frac{1}{4P}\lambda\bfe_1+\Theta(\frac{1}{P})(-\bfe_2+\bfe_3-\bfe_4)\label{vit_d_t01},
\end{equation}
and if $\bfz_{n(\cdot,j)}$ comes from $\bfv_4$,
\begin{equation}
        \frac{\partial f_{\btheta}(\bfx_n,\bfdt^{(t)})}{\partial \bfdt_j^{(t)}}= -\frac{1}{4P}\lambda\bfe_4+ \Theta(\frac{1}{P})(-\bfe_2+\bfe_3+\bfe_1)\label{vit_d_t04},
\end{equation}
where $\lambda=\mu=\Theta(1)$. The last terms in (\ref{vit_d_t01}) and (\ref{vit_d_t04}) come from the indicators from other $\bfW_O$ neurons, which may become $1$ because of feature noises. 
Note that when $t\geq 2$, since the data which contains $\bfe_2$ and $\bfe_3$ would similarly contribute to the overall gradient, there will be a close amount of $\bfe_1$ and $\bfe_2$ in $\bfdt_j^{(t)}$ and a close amount of $\bfe_3$ and $\bfe_4$ in $\bfdt_j^{(t)}$. Hence, when $k\mu< \Theta(1)$,
\begin{equation}
\begin{aligned}
    \mathbb{E}[\bfdt_j^{(t)}]&=\mathbb{E}[\bfdt_j^{(0)}]-\mathbb{E}[\eta\sum_{b=1}^t\frac{1}{B}\sum_{n\in\mathcal{B}_b}\frac{\partial}{\partial\bfdt_j}\ell(f_{\btheta}(\bfx_n, \bfdt^{(b)}), y_n)]\\
    &= \eta t \frac{1}{4P}(\lambda\bfe_1+\lambda\bfe_2-\mu\bfe_3-\mu\bfe_4)\\
    &= k(\lambda\bfe_1+\lambda\bfe_2-\mu\bfe_3-\mu\bfe_4),
\end{aligned}
\end{equation}
\begin{equation}
    \bfdt_j^{(t)}=\mathbb{E}[\bfdt_j^{(t)}]+ \frac{\eta t}{L} \sqrt{\frac{\log Bt}{Bt}}(\pm\bfe_1\pm\bfe_2\pm\bfe_3\pm\bfe_4)
\end{equation}
where $\lambda\geq \Theta(1)\cdot (1-\sigma P
)$, $\mu\geq \Theta(1)\cdot (1-\sigma P
)$ for $t\geq 2$. The term $(1-\sigma P)$ comes from that for $\bfW_{O_(i,\cdot)}=\bfe_1$ or $\bfe_4$, 
\begin{equation}
\begin{aligned}
        &\mathbb{E}[\mathbbm{1}[\sum_{s=1}^P\bfW_{O_{(i,\cdot)}}( \bfz_{n(\cdot,P_{s,2})}+\bfdt_{P_{s,2}}^{(t)})\text{softmax}(\beta_2^2{( \bfz_{n(\cdot,P_{s,2})}+\bfdt_{P_{s,2}}^{(t)})}^\top( \bfz_{n(\cdot,l)}+\bfdt_{l}^{(t)}))\geq 0]\\
        \geq &1-e^{\frac{(Bt)^2}{\sigma^2P^2}}\geq 1-\sigma P
\end{aligned}
\end{equation}
given $B\geq  \Theta(1)$ by Hoeffding inequality. When $k\mu\geq \Theta(1)$, for $\bfz_n=[\bfe_4,\bfe_4,\cdots,\bfe_4,\bfe_1,\bfe_4,\cdots,\bfe_4]$,
\begin{equation}
    \bfz_{n(\cdot,j)}+\bfdt_j^{(t)}=k\lambda(\bfe_1+\bfe_2)-k\mu\bfe_3+(1-k\mu)\bfe_4
\end{equation}
for $\bfz_{n(\cdot,j)}$ from $\bfv_4$. Then, 
\begin{equation}
    \mathbb{E}[\mathbbm{1}[\sum_{s=1}^P\bfe_1( \bfz_{n(\cdot,P_{s,2})}+\bfdt_{P_{s,2}}^{(t)})\text{softmax}(\beta_2^2{( \bfz_{n(\cdot,P_{s,2})}+\bfdt_{P_{s,2}}^{(t)})}^\top( \bfz_{n(\cdot,l)}+\bfdt_{l}^{(t)}))]]\geq 1-e^{\frac{(Bt)^2}{\sigma^2}}\geq 1-\sigma
\end{equation}
\begin{equation}
    \Pr(\sum_{s=1}^P\bfe_4( \bfz_{n(\cdot,P_{s,2})}+\bfdt_{P_{s,2}}^{(t)})\text{softmax}(\beta_2^2{( \bfz_{n(\cdot,P_{s,2})}+\bfdt_{P_{s,2}}^{(t)})}^\top( \bfz_{n(\cdot,l)}+\bfdt_{l}^{(t)})))\leq e^{-\frac{1}{\sigma^2}}\leq e^{-P^2}
\end{equation}
Hence, with a probability at least $1-e^{-P^2}$, no patches is activated by $\bfe_4$. For $\bfz_{n(\cdot,k)}$ from $\bfv_1$ and $\bfz_{n(\cdot,j)}$ from $\bfv_4$, we have
\begin{equation}
    \text{softmax}((\bfz_{n(\cdot,k)}+\bfdt_k^{(t)})^\top(\bfz_{n(\cdot,k)}+\bfdt_k^{(t)}))\geq \frac{1}{P}
\end{equation}
\begin{equation}
    \text{softmax}((\bfz_{n(\cdot,j)}+\bfdt_j^{(t)})^\top(\bfz_{n(\cdot,k)}+\bfdt_k^{(t)}))= \Theta(\frac{1}{P})
\end{equation}
\begin{equation}
    \text{softmax}((\bfz_{n(\cdot,j)}+\bfdt_j^{(t)})^\top(\bfz_{n(\cdot,j)}+\bfdt_j^{(t)}))\geq  \frac{1}{P}
\end{equation}
\begin{equation}
    \text{softmax}((\bfz_{n(\cdot,k)}+\bfdt_k^{(t)})^\top(\bfz_{n(\cdot,j)}+\bfdt_j^{(t)}))= \Theta(\frac{1}{P})
\end{equation}
Therefore, when $k\mu>\Theta(1)$, i.e., $t\geq t_0=4P\eta^{-1}(1-\sigma P)^{-1}$ we have
\begin{equation}
\begin{aligned}
    \bfdt_j^{(t)}=&\mathbb{E}[\bfdt_j^{(t)}]+ \frac{\eta t}{P}\sqrt{\frac{\log B(t-t_0)}{B(t-t_0)}}(\pm(\bfe_1+\bfe_2)\pm\frac{1}{P}e^{-P^4}(\bfe_3+\bfe_4))\\
    =&\mathbb{E}[\bfdt_j^{(t_0)}]-\mathbb{E}[\eta\sum_{b=t_0}^t\frac{1}{B}\sum_{n\in\mathcal{B}_b}\frac{\partial}{\partial\bfdt_j}\ell(f_{\btheta}(\bfx_n, \bfdt^{(b)}),y_n)]\pm \frac{\eta t}{P}\sqrt{\frac{\log B(t-t_0)}{B(t-t_0)}}(\bfe_1+\bfe_2)\\
    = &\mathbb{E}[\bfdt_j^{(t_0)}]+\frac{\eta (t-t_0)}{4P}(\lambda\bfe_1+\lambda\bfe_2+\mu\bfe_3+\mu\bfe_4)\pm \frac{\eta t}{P}\sqrt{\frac{\log B(t-t_0)}{B(t-t_0)}}(\bfe_1+\bfe_2),
\end{aligned}
\end{equation}
where $\lambda\gtrsim (1-\sigma)^{-1}$. Then,
\begin{equation}
    \begin{aligned}
        \Big|\bfe_3^\top\mathbb{E}[\eta\sum_{b=t_0}^t\frac{1}{B}\sum_{n\in\mathcal{B}_b}\frac{\partial}{\partial\bfdt}\ell(f_{\btheta}(\bfx_n, \bfdt^{(b)}),y_n)]\Big|\lesssim \eta(t-t_0)\frac{1}{P} \cdot\sqrt{\frac{\log B(t-t_0)}{B(t-t_0)}}
    \end{aligned}
\end{equation}
\begin{equation}
    \begin{aligned}
        \Big|\bfe_4^\top\mathbb{E}[\eta\sum_{b=t_0}^t\frac{1}{B}\sum_{n\in\mathcal{B}_b}\frac{\partial}{\partial\bfdt}\ell(f_{\btheta}(\bfx_n, \bfdt^{(b)}),y_n)]\Big|\lesssim \eta(t-t_0)\frac{1}{P} \cdot\sqrt{\frac{\log B(t-t_0)}{B(t-t_0)}}
    \end{aligned}
\end{equation}
and thus $|\mu|\leq \Theta(1/\sqrt{B(t-t_0)})$. Hence, for $\bfz_{n(\cdot,k)}$ from $\bfv_1$ and $\bfz_{n(\cdot,j)}$ from $\bfv_4$, 
\begin{equation}
\begin{aligned}
&(\bfz_{n(\cdot,k)}+\bfdt_k^{(t)})^\top(\bfz_{n(\cdot,k)}+\bfdt_k^{(t)})-(\bfz_{n(\cdot,k)}+\bfdt_k^{(t)})^\top(\bfz_{n(\cdot,j)}+\bfdt_j^{(t)})\\
=&\Theta(1)\cdot \frac{e^{\beta_2^2}}{e^{\beta_2^2}+P-1}(\frac{e^{\beta_2^2}}{e^{\beta_2^2}+P-1}+\bfe_1^\top\bfdt^{(t)})
\end{aligned}
\end{equation}
\begin{equation}
\begin{aligned}
    &(\bfz_{n(\cdot,j)}+\bfdt_j^{(t)})^\top(\bfz_{n(\cdot,k)}+\bfdt_k^{(t)})-(\bfz_{n(\cdot,j)}+\bfdt_j^{(t)})^\top(\bfz_{n(\cdot,j)}+\bfdt_j^{(t)})\\=&\Theta(1)\cdot \frac{e^{\beta_2^2}}{e^{\beta_2^2}+P-1}\cdot \bfe_1^\top\bfdt^{(t)}
\end{aligned}
\end{equation}
Since that $\beta_2=\Theta(1)$, we have
\begin{equation}
    \text{softmax}((\bfz_{n(\cdot,k)}+\bfdt_k^{(t)})^\top(\bfz_{n(\cdot,k)}+\bfdt_k^{(t)}))=\frac{e^{\Theta(1)\cdot \frac{\bfe_1^\top\bfdt^{(t)}}{P}}}{P-1+e^{\Theta(1)\cdot \frac{\bfe_1^\top\bfdt^{(t)}}{P}}}
\end{equation}
\begin{equation}
    \text{softmax}((\bfz_{n(\cdot,k)}+\bfdt_k^{(t)})^\top(\bfz_{n(\cdot,j)}+\bfdt_j^{(t)}))=\frac{e^{\Theta(1)\cdot \frac{\bfe_1^\top\bfdt^{(t)}}{P
    }}}{P-1+e^{\Theta(1)\cdot \frac{\bfe_1^\top\bfdt^{(t)}}{P}}}
\end{equation}
To make 
\begin{equation}
    f_{\btheta}(\bfx_n, \bfdt^{(t)})\geq 1/P,
\end{equation}
we require that
\begin{equation}
    \frac{e^{\Theta(1)\cdot \frac{\bfe_1^\top\bfdt^{(t)}}{P}}}{P-1+e^{\Theta(1)\cdot \frac{\bfe_1^\top\bfdt^{(t)}}{P}}}\cdot \frac{e^{\beta_2^2}}{e^{\beta_2^2}+P-1}+\frac{P-1}{P-1+e^{\Theta(1)\cdot \frac{\bfe_1^\top\bfdt^{(t)}}{P}}}\cdot \frac{1}{e^{\beta_2^2}(P-1)+1}\geq \frac{1}{P}
\end{equation}
As a result, we finally need
\begin{equation}
    e^{\Theta(1)\cdot \frac{\bfe_1^\top\bfdt^{(t)}}{P}}\gtrsim P
\end{equation}
which holds as long as $t-t_0\gtrsim P^2 \eta^{-1}(1-\sigma)^{-1}\log P$.
Therefore, we have
\begin{equation}
    f_{\btheta}(\bfx_n,\bfdt)\geq 1/P
\end{equation}
for $\bfx_n$ that contains a patch from $\bfv_1$. We similarly have 
\begin{equation}
    f_{\btheta}(\bfx_n,\bfdt)\leq -1/P
\end{equation}
for $\bfx_n$ that contains a patch from $\bfv_2$. To sum up,  we need $t\geq \Theta(\eta^{-1}P^2 (1-\sigma)^{-1}\log P)$ iterations. 

\subsubsection{Proof of Lemma \ref{lemma: ViT-shallow}}
\noindent \textbf{Proof:}\\
\noindent To avoid multiple superscripts, we use $\boldsymbol{\delta}$ to denote $\boldsymbol{\delta}^{(h)}$ since that $h$ is fixed in this proof. We use $\boldsymbol{\delta}^{(t)}$ to denote the update of $\boldsymbol{\delta}$ at $t$-th iteration. For the network
\begin{equation}
\begin{aligned}
    f_{\btheta}(\bfx_n, \bfdt)=&\sum_{l=1}^{P}\sum_{i=1}^m a_{l(i)}^\top\text{Relu}(\sum_{j=1}^{P}\bfW_{O_{2(i,\cdot)}}\bfW_{V_2}\bfz_{n(\cdot,j)}\text{softmax}({\bfz_{n(\cdot,j)}}^\top\bfW_{K_2}^\top\bfW_{Q_2}\bfz_{n(\cdot,l)}))
\end{aligned}
\end{equation}
where 
\begin{equation}
    \bfz_{n(\cdot,j)}=\text{Relu}(\sum_{s=1}^{P}\bfW_{O_{1}}\bfW_{V_1}(\bfx_{n(\cdot, P_{s,1})}+\bfdt_s)\text{softmax}(({\bfx_{n(\cdot, P_{s,1})}}+\bfdt_s)^\top\bfW_{K_1}^\top\bfW_{Q_1}(\bfx_j^n+\bfdt_j))),
\end{equation}
we have
\begin{equation}
    \frac{\partial f_{\btheta}(\bfx_n, \bfdt)}{\partial \bfdt_s}=\sum_{j=1}^{P}\frac{\partial f_{\btheta}(\bfx_n, \bfdt)}{\partial \bfz_{n(\cdot,j)}}\frac{\partial \bfz_{n(\cdot,j)}}{\partial \bfdt_s}
\end{equation}
Note that $\bfW_{Q_2}=\beta_2 \bfI$, $\bfW_{Q_1}=\beta_1 \bfI$, $\bfW_{K_2}=\beta_2\bfP_x$, $\bfW_{K_1}=\beta_1\bfP_x$,, $\bfW_{V_2}=\bfP_x$, $\bfW_{V_1}=\bfP_x$, where $\beta_1=\Theta(1)$ and $\beta_2=\Theta(1)$. Therefore,
\begin{equation}
\begin{aligned}
      &\frac{\partial f_{\btheta}(\bfx_n, \bfdt)}{\partial \bfz_{n(\cdot,j)}}\\
      =& \sum_{l=1}^{P}\sum_{i=1}^m\bfa_{(l)_i}^\top\mathbbm{1}[\sum_{s=1}^{P}\bfW_{O_{2(i,\cdot)}}\bfz_{n(\cdot,P_{s,2})}\text{softmax}(\beta_2^2{\bfz_{n(\cdot,P_{s,2})}}^\top\bfz_{n(\cdot,l)})]\Big(\text{softmax}(\beta_2^2{\bfz_{n(\cdot,j)}}^\top\bfz_{n(\cdot,l)})\\
      &\cdot\bfW_{O_{2(i,\cdot)}}+\mathbbm{1}[j\neq l] \bfW_{O_{2(i,\cdot)}}\bfz_{n(\cdot,j)}\cdot \bfz_{n(\cdot,l)}\cdot(-\text{softmax}(\beta_2^2{\bfz_{n(\cdot,j)}}^\top\bfz_{n(\cdot,l)}))\\
      &\cdot\text{softmax}(\beta_2^2{\bfz_{n(\cdot,l)}}^\top\bfz_{n(\cdot,l)})+ \mathbbm{1}[j= l]\bfW_{O_{2(i,\cdot)}}\bfz_{n(\cdot,l)}\text{softmax}(\beta_2^2{\bfz_{n(\cdot,l)}}^\top\bfz_{n(\cdot,l)})\\
      &\cdot(1-\text{softmax}(\beta_2^2{\bfz_{n(\cdot,l)}}^\top\bfz_{n(\cdot,l)}))\bfz_{n(\cdot,l)}\Big)
\end{aligned}
\end{equation}
\begin{equation}
\begin{aligned}
      &\frac{\partial \bfz_{n(\cdot,j)}}{\partial \bfdt_k}\\
      =& \mathbbm{1}[\sum_{s=1}^{P}\bfW_{O_{1}}(\bfx_{n(\cdot, P_{s,1})}+\bfdt_s)\text{softmax}(({\bfx_{n(\cdot, P_{s,1})}}+\bfdt_s)^\top(\bfx_j^n+\bfdt_j))]\Big(\text{softmax}(({\bfx_j^n}+\bfdt_j)^\top\\
      &\cdot(\bfx_{n(\cdot, l)}+\bfdt_l))\bfW_{O_1}+\mathbbm{1}[k\neq l] \bfW_{O_{1}}(\bfx_{n(\cdot, k)}+\bfdt_k)\cdot {(\bfx_{n(\cdot, l)}+\bfdt_l)}^\top\\
      &\cdot(-\text{softmax}(\beta_1^2{(\bfx_j^n+\bfdt_j)}^\top(\bfx_{n(\cdot, l)}+\bfdt_l)))\text{softmax}(\beta_1^2{(\bfx_{n(\cdot, l)}+\bfdt_l)}^\top(\bfx_{n(\cdot, l)}+\bfdt_l))\\
      &+\mathbbm{1}[k=l] \bfW_{O_{1}}(\bfx_{n(\cdot, l)}+\bfdt_l)(\bfx_{n(\cdot, l)}+\bfdt_l)^\top\\
      &\cdot\text{softmax}(\beta_1^2{(\bfx_{n(\cdot, l)}+\bfdt_l)}^\top(\bfx_{n(\cdot, l)}+\bfdt_l))\\
      &\cdot(1-\text{softmax}(\beta_1^2{(\bfx_{n(\cdot, l)}+\bfdt_l)}^\top(\bfx_{n(\cdot, l)}+\bfdt_l)))\Big)
\end{aligned}
\end{equation}
Let $t=0$. For $y^n=+1$, Note that if $\bfx_n=[\bfe_3,\bfe_3,\cdots,\bfe_3,\bfe_1,\bfe_3,\cdots, \bfe_3]$ without noise, the loss is $0$. Hence, we compute the loss from $\bfx_n=[\bfe_4,\bfe_4,\cdots,\bfe_4,\bfe_1,\bfe_4,\cdots, \bfe_4]$.
\begin{equation}
\begin{aligned}
    &\mathbb{E}[\mathbbm{1}[\sum_{s=1}^{P}\bfW_{O_{(i,\cdot)}}( \bfx_{n(\cdot,P_{s,1})}+\bfdt_{P_{s,1}}^{(t)})\text{softmax}(\beta_1^2{( \bfx_{n(\cdot,P_{s,1})}+\bfdt_{P_{s,1}}^{(t)})}^\top(\bfx_{n(\cdot, l)}+\bfdt_l))\geq 0]\\
    =&\Pr(\sum_{s=1}^{P}\bfW_{O_{(i,\cdot)}}( \bfx_{n(\cdot,P_{s,1})}+\bfdt_{P_{s,1}}^{(t)})\text{softmax}(\beta_1^2{( \bfx_{n(\cdot,P_{s,1})}+\bfdt_{P_{s,1}}^{(t)})}^\top(\bfx_{n(\cdot, l)}+\bfdt_l))\geq 0)
\end{aligned}
\end{equation}
for $\bfW_{O_{(i,\cdot)}}=\bfe_1$ or $\bfe_4$. We can finally show that with a high probability, the above indicator is close to $1$. Meanwhile, for $\bfW_{O_{(i,\cdot)}}=\bfe_2$ or $\bfe_3$, the indicator equals $0$ or $1$ with half probability when $t=0$. Consider that $\bfx_{n(\cdot,j)}$ comes from $\bfv_4$. In this case, if $\bfx_{n(\cdot, l)}$ comes from $\bfv_1$,

\begin{equation}
    \text{softmax}(\beta_1^2{(\bfx_{n(\cdot, l)}+\bfdt_l)}^\top(\bfx_{n(\cdot, l)}+\bfdt_l))\geq \frac{1}{P}
\end{equation}
\begin{equation}
    \text{softmax}(\beta_1^2{(\bfx_j^n+\bfdt_j^{(t)})}^\top(\bfx_{n(\cdot, l)}+\bfdt_l))= \Theta(\frac{1}{P})
\end{equation}
\begin{equation}
    \text{softmax}(\beta_2^2{\bfz_{n(\cdot,l)}}^\top\bfz_{n(\cdot,l)})\geq \frac{1}{P}
\end{equation}
\begin{equation}
    \text{softmax}(\beta_2^2{\bfz_{n(\cdot,j)}}^\top\bfz_{n(\cdot,l)})= \Theta(\frac{1}{P})
\end{equation}
If $\bfx_{n(\cdot, l)}$ comes from $\bfv_4$, then
\begin{equation}
    \text{softmax}(\beta_1^2{(\bfx_{n(\cdot, l)}+\bfdt_l^{(t)})}^\top(\bfx_{n(\cdot, l)}+\bfdt_l^{(t)}))\geq \frac{1}{P}
\end{equation}
\begin{equation}
    \text{softmax}(\beta_1^2{(\bfx_j^n+\bfdt_j^{(t)})}^\top(\bfx_{n(\cdot, l)}+\bfdt_l^{(t)}))= \Theta(\frac{1}{P})
\end{equation}
\begin{equation}
    \text{softmax}(\beta_2^2{\bfz_{n(\cdot,l)}}^\top\bfz_{n(\cdot,l)})\geq \frac{1}{P}
\end{equation}
\begin{equation}
    \text{softmax}(\beta_2^2{\bfz_{n(\cdot,j)}}^\top\bfz_{n(\cdot,l)})= \Theta(\frac{1}{P})
\end{equation}

\noindent Then we consider that $\bfx_{n(\cdot,j)}$ comes from $\bfv_1$. In this case, if $\bfz_{n(\cdot,l)}$ comes from $\bfv_1$, then
\begin{equation}
    \text{softmax}(\beta_1^2{(\bfx_{n(\cdot,j)}+\bfdt_j^{(t)})}^\top(\bfx_{n(\cdot, l)}+\bfdt_l^{(t)}))\geq \Theta(\frac{1}{P})
\end{equation}
\begin{equation}
    \text{softmax}(\beta_2^2{\bfz_{n(\cdot,j)}}^\top\bfz_{n(\cdot,l)})\geq \Theta(\frac{1}{P})
\end{equation}
If $\bfx_{n(\cdot, l)}$ comes from $\bfv_4$,
\begin{equation}
    \text{softmax}(\beta_1^2{(\bfx_{n(\cdot,j)}+\bfdt_j^{(t)})}^\top(\bfx_{n(\cdot, l)}+\bfdt_l^{(t)}))= \Theta(\frac{1}{P})
\end{equation}
\begin{equation}
    \text{softmax}(\beta_2^2{\bfz_{n(\cdot,j)}}^\top\bfz_{n(\cdot,l)})= \Theta(\frac{1}{P})
\end{equation}
Therefore, if $\bfx_{n(\cdot,j)}$ comes from $\bfv_1$,
\begin{equation}
        \frac{\partial f_{\btheta}(\bfx_n,\bfdt)}{\partial \bfdt_j^{(t)}}= P\cdot\frac{1}{4P}\lambda(\bfe_1^\top \cdot\frac{1}{P}\bfW_{O_1})^\top=\frac{1}{4P}\bfv_1+\Theta(\frac{1}{P})(-\bfv_2+\bfv_3-\bfv_4),
\end{equation}
and if $\bfx_{n(\cdot,j)}$ comes from $\bfv_4$,
\begin{equation}
        \frac{\partial f_{\btheta}(\bfx_n,\bfdt)}{\partial \bfdt_j^{(t)}}= -\frac{1}{4P}\mu\bfv_4+\Theta(\frac{1}{P})(-\bfv_2+\bfv_3+\bfv_1),
\end{equation}
where $\lambda=\bfmu=\Theta(1)$. Note that when $t\geq 2$, since the data which contains $\bfv_2$ and $\bfv_3$ would similarly contribute to the overall gradient, there will be a close amount of $\bfv_1$ and $\bfv_2$ in $\bfdt_s^{(t)}$ and a close amount of $\bfv_3$ and $\bfv_4$ in $\bfdt_s^{(t)}$. Hence, when $k\mu< \Theta(1)$,
\begin{equation}
\begin{aligned}
    \mathbb{E}[\bfdt_s^{(t)}]&=\mathbb{E}[\bfdt_s^{(0)}]-\mathbb{E}[\eta\sum_{b=1}^t\frac{1}{B}\sum_{n\in\mathcal{B}_b}\frac{\partial}{\partial\bfdt_s}\ell(f_{\btheta}(\bfx_n, \bfdt_s^{(b)}),y_n)]\\
    &= \eta t \frac{1}{4P}(\lambda\bfv_1+\lambda\bfv_2-\mu\bfv_3-\mu\bfv_4)\\
    &= k(\lambda\bfv_1+\lambda\bfv_2-\mu\bfv_3-\mu\bfv_4),
\end{aligned}
\end{equation}
\begin{equation}
    \bfdt_s^{(t)}=\mathbb{E}[\bfdt_s^{(t)}]+ \frac{\eta t}{P} \sqrt{\frac{\log Bt}{Bt}}(\pm\bfv_1\pm\bfv_2\pm\bfv_3\pm\bfv_4)
\end{equation}
where $\lambda\geq \Theta(1)\cdot (1-\sigma P
)$, $\mu\geq \Theta(1)\cdot (1-\sigma P
)$ for $t\geq 2$. The term $(1-\sigma P)$ comes from that for $\bfW_{O_2(i,\cdot)}=\bfv_1$ or $\bfv_4$, 
\begin{equation}
\begin{aligned}
        &\mathbb{E}[\mathbbm{1}[\sum_{s=1}^{P}\bfW_{O_{1(i,\cdot)}}( \bfx_{n(\cdot,P_{s,1})}+\bfdt_{P_{s,1}}^{(t)})\text{softmax}(\beta_1^2{( \bfx_{n(\cdot,P_{s,1})}+\bfdt_{P_{s,1}}^{(t)})}^\top(\bfx_{n(\cdot, l)}+\bfdt_l^{(t)}))\geq 0]\\
        \geq &1-e^{\frac{(Bt)^2}{\sigma^2P^2}}\geq 1-\sigma P
\end{aligned}
\end{equation}
given $B\geq  \Theta(1)$ by Hoeffding inequality. When $k\mu\geq \frac{\Theta(1)}{1+\gamma}$, we have that for $\bfx_{n(\cdot,j)}$ from $\bfv_4$, 
\begin{equation}
    \begin{aligned}
        &\mathbbm{1}[\sum_{s=1}^{P}\bfW_{O_{1}}(\bfx_{n(\cdot, P_{s,1})}+\bfdt_s)\text{softmax}(\beta_1^2({\bfx_{n(\cdot, P_{s,1})}}+\bfdt_s)^\top(\bfx_{n(\cdot,j)}+\bfdt_j^{(t)}))\geq 0]\\
        \geq & [1,1,-k\mu+(1-k\mu)\gamma+\bfv_3^\top\bfa, -k\mu\gamma+1-k\mu+\bfv_4^\top\bfa]^\top\\
        \geq & [1,1,0,0]^\top
    \end{aligned}
\end{equation}
where $\bfa\sim\mathcal{N}(0,\sigma^2\bfI)$ in the first step, and the last step holds with probability at least
\begin{equation}
    \Pr(\bfv_4^\top\bfa-k\mu\gamma+1-k\mu\leq 0)\leq 1-\Pr(\bfv_4^\top\bfa\geq \Theta(1))\leq 1-e^{\frac{1}{\sigma^2}}\leq 1-e^{-P^2}
\end{equation}
\begin{equation}
    \Pr(\bfv_3^\top\bfa-k\mu+(1-k\mu)\gamma\leq 0)\leq 1-\Pr(\bfv_3^\top\bfa\geq \Theta(1))\leq 1-e^{\frac{1}{\sigma^2}}\leq 1-e^{-P^2}
\end{equation}
Hence, for $\bfx_{n(\cdot, k)}$ from $\bfv_1$ and $\bfx_{n(\cdot,j)}$ from $\bfv_4$, 
\begin{equation}
    (\bfx_{n(\cdot, k)}+\bfdt_k^{(t)})^\top(\bfx_{n(\cdot, k)}+\bfdt_k^{(t)})-(\bfx_{n(\cdot, k)}+\bfdt_k^{(t)})^\top(\bfx_{n(\cdot,j)}+\bfdt_j^{(t)})=\Theta(1)\cdot (1+2 (k\mu)^2)
\end{equation}
\begin{equation}
    (\bfx_{n(\cdot,j)}+\bfdt_j^{(t)})^\top(\bfx_{n(\cdot, k)}+\bfdt_k^{(t)})-(\bfx_{n(\cdot,j)}+\bfdt_j^{(t)})^\top(\bfx_{n(\cdot,j)}+\bfdt_j^{(t)})=\Theta(1)\cdot (2k\mu-1)
\end{equation}
Since that $\beta_1=\Theta(1)$, we have
\begin{equation}
    \text{softmax}(\beta_1^2(\bfx_{n(\cdot, k)}+\bfdt_k^{(t)})^\top(\bfx_{n(\cdot, k)}+\bfdt_k^{(t)}))=\frac{e^{\Theta(1)\cdot (k\mu)^2)}}{P-1+e^{\Theta(1)\cdot (k\mu)^2)}}
\end{equation}
\begin{equation}
    \text{softmax}(\beta_1^2(\bfx_{n(\cdot, k)}+\bfdt_k^{(t)})^\top(\bfx_{n(\cdot,j)}+\bfdt_j^{(t)}))=\frac{e^{\Theta(1)\cdot k\mu}}{P-1+e^{\Theta(1)\cdot k\mu}}
\end{equation}
To make 
\begin{equation}
    f_{\btheta}(\bfx_n, \bfdt^{(t)})\geq 1/P,
\end{equation}
we require that
\begin{equation}
    \frac{e^{\Theta(1)\cdot (k\mu)^2)}}{P-1+e^{\Theta(1)\cdot (k\mu)^2)}}\cdot 1\geq \frac{1}{P}
\end{equation}
or 
\begin{equation}
    \frac{e^{\Theta(1)\cdot k\mu}}{P-1+e^{\Theta(1)\cdot k\mu}}\cdot 1\geq \frac{1}{P}
\end{equation}
As a result, we finally need
\begin{equation}
    e^{\Theta(1)\cdot k\mu}\gtrsim 1
\end{equation}
which holds as long as $t\gtrsim P \eta^{-1}(1-P\sigma)^{-1}(1+\gamma)^{-1})$. 
With the same condition, we also have that for all $y^n=-1$,
\begin{equation}
    f_{\btheta}(\bfx_n,\bfdt)\leq -1/P
\end{equation}
To sum up, we need $t\geq \Theta(P \eta^{-1}(1-P\sigma)^{-1}(1+\gamma)^{-1}))$.

\end{document}